\RequirePackage{scrlfile}
\PreventPackageFromLoading{subfig}
\documentclass{article}

\usepackage[nonatbib,final]{neurips_2020}
\usepackage[square, comma, sort&compress, numbers]{natbib}

\usepackage[utf8]{inputenc} 
\usepackage[T1]{fontenc}    
\usepackage{hyperref}       
\usepackage{url}            
\usepackage{booktabs}       
\usepackage{amsfonts}       
\usepackage{nicefrac}       
\usepackage{microtype}      

\usepackage{preamble}
\usepackage{subfloat} 
\usepackage{subcaption} 
\usepackage{xcolor}

\title{Learning Optimal Representations with the Decodable Information Bottleneck}

\author{%
  Yann Dubois \\
  Facebook AI Research\\
  \texttt{yannd@fb.com} \\
  \And
  Douwe Kiela \\
  Facebook AI Research\\
  \texttt{dkiela@fb.com} \\
  \And
  David J. Schwab \\
  Facebook AI Research\\
  CUNY Graduate Center\\
  \texttt{dschwab@fb.com} \\
  \And
  Ramakrishna Vedantam \\
  Facebook AI Research\\
  \texttt{ramav@fb.com} \\
}

\begin{document}

\maketitle


\begin{abstract}
We address the question of characterizing and finding optimal representations for supervised learning. 
%
%
%
Traditionally, this question has been tackled using the Information Bottleneck, which compresses the inputs while retaining information about the targets, in a decoder-agnostic fashion. In machine learning, however, our goal is not compression but rather generalization, which is intimately linked to the predictive family or decoder of interest (e.g. linear classifier). 
We propose the Decodable Information Bottleneck (DIB) that considers information retention and compression from the perspective of the desired predictive family.
As a result, DIB gives rise to representations that are optimal in terms of expected test performance and can be estimated with guarantees.  
Empirically, we show that the framework can be used to enforce a small generalization gap on downstream classifiers and to predict the generalization ability of neural networks.
\end{abstract}

\ydnote{review the abstract to be more specific on contributions ?}
\section{Introduction}
\ydnote{Review all appendices so that there are references in order}
\label{sec:introduction}
A fundamental choice in supervised machine learning (ML) centers around the data representation from which to perform predictions.
While classical ML uses predefined encodings of the data~\citep{salton1986bow,cortes1995svm,leung2001texton,Lowe2004sift,RahimiR07}
recent progress~\citep{BengioCV13,ZhongWD16} has been driven by learning such representations.
A natural question, then, is what characterizes an ``optimal'' representation --- in terms of generalization --- and how to learn it. 

The standard framework for studying generalization, statistical learning theory~\cite{shalevShwartz2014uml}, usually assumes a fixed dataset/representation, and aims to restrict the predictive functional family $\V$ (e.g. linear classifiers) such that empirical risk minimizers (ERMs) generalize.\footnote{Rather than defining learning in terms of deterministic hypotheses $h \in \mathcal{H}$, we consider the more general \cite{gressmann2018probabilistic} case of probabilistic predictors $\V$, in order to make a link with information theory.} 
Here, we turn the problem on its head: we ask whether it is possible to enforce generalization by
changing the representation of the inputs
such that ERMs in $\V$ perform well, irrespective of the complexity of $\V$.

A common approach to representation learning consists of jointly training the classifier and representation by minimizing the empirical risk (which we call J-ERM).
By only considering empirical risk, J-ERM is optimal in the infinite data limit (consistent; \cite{duchi2018multiclass}), but the resulting representations do not favor classifiers that will generalize from finite samples.
In contrast, the information bottleneck (IB) method \cite{tishby2000information} aims for representations that have \textit{minimal} information about the inputs to avoid over-fitting, while having \textit{sufficient} information about the labels \cite{shamir2010learning}. 
While conceptually appealing and used in a range of applications \citep{slonim2000document,shwartz2017opening,farajiparvar2018information,goyal2019infobot}, IB is based on Shannon's mutual information, which was developed for communication theory \cite{shannon1948mathematical} and does not take into account the predictive family $\V$ of interest. 
As a result, IB's sufficiency requirement does not ensure the existence of a predictor $f \in \V$ that can perform well using the learned representation; \footnote{As an illustration, IB invariance to bijections suggests that a non-linearly entangled representation is as good as a linearly separable one if there is a bijection between them, even when classifying using a logistic regression.} while its minimality term is difficult to estimate, making IB impractical without resorting to approximations \cite{ChalkMT16,alemi2016deep,AchilleS18dropout,kolchinsky2019nonlinear}.

We resolve these issues by introducing the \textit{decodable information bottleneck} (DIB) objective, which recovers minimal sufficient representations relative to a predictive family $\V$. 
Intuitively, it ensures that classifiers in $\V$ can predict labels (\vsufficiency) but cannot distinguish examples with the same label (\vminimality), as illustrated in 
\cref{fig:schema_rep_med}.
\begin{figure}
\centering
\begin{subfigure}{0.28\textwidth}
 \frame{\includegraphics[width=\textwidth]{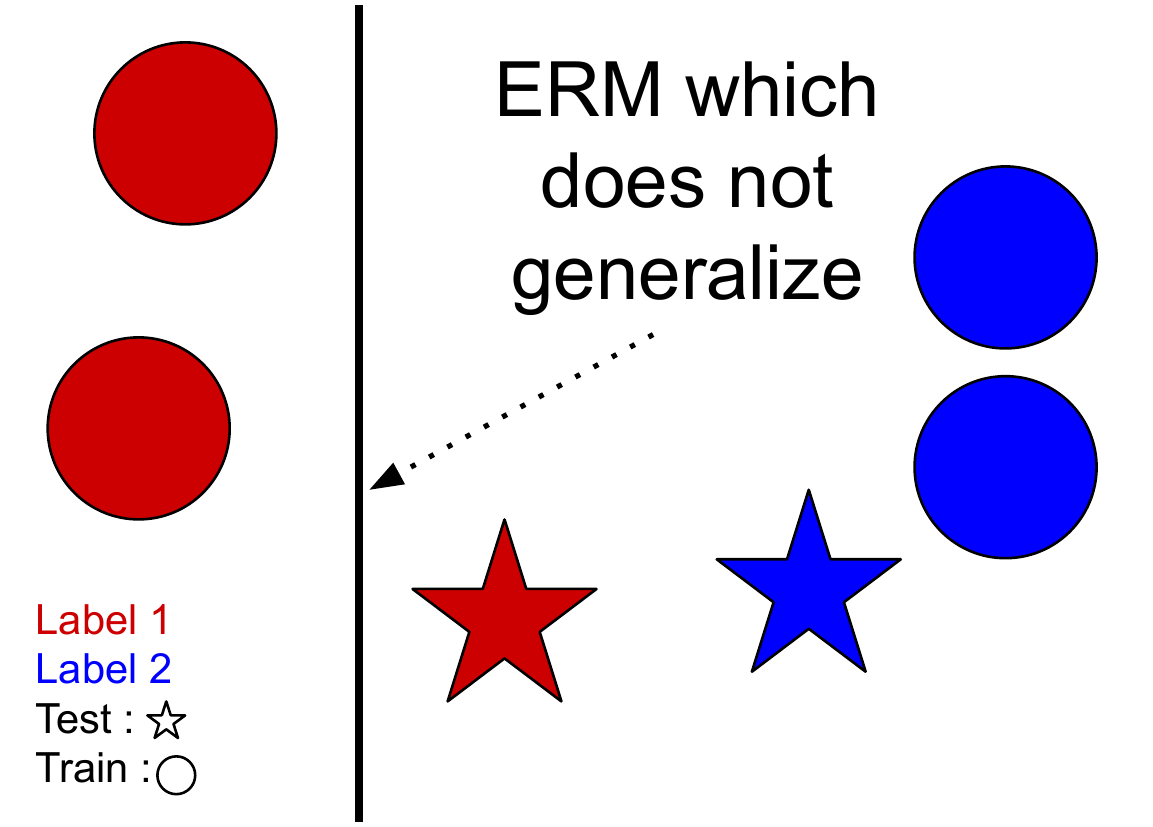}} 
  \caption{Schematic J-ERM}
 \label{fig:joint_erm_med}
 \end{subfigure}
 \begin{subfigure}{0.28\textwidth}
 \frame{\includegraphics[width=\textwidth]{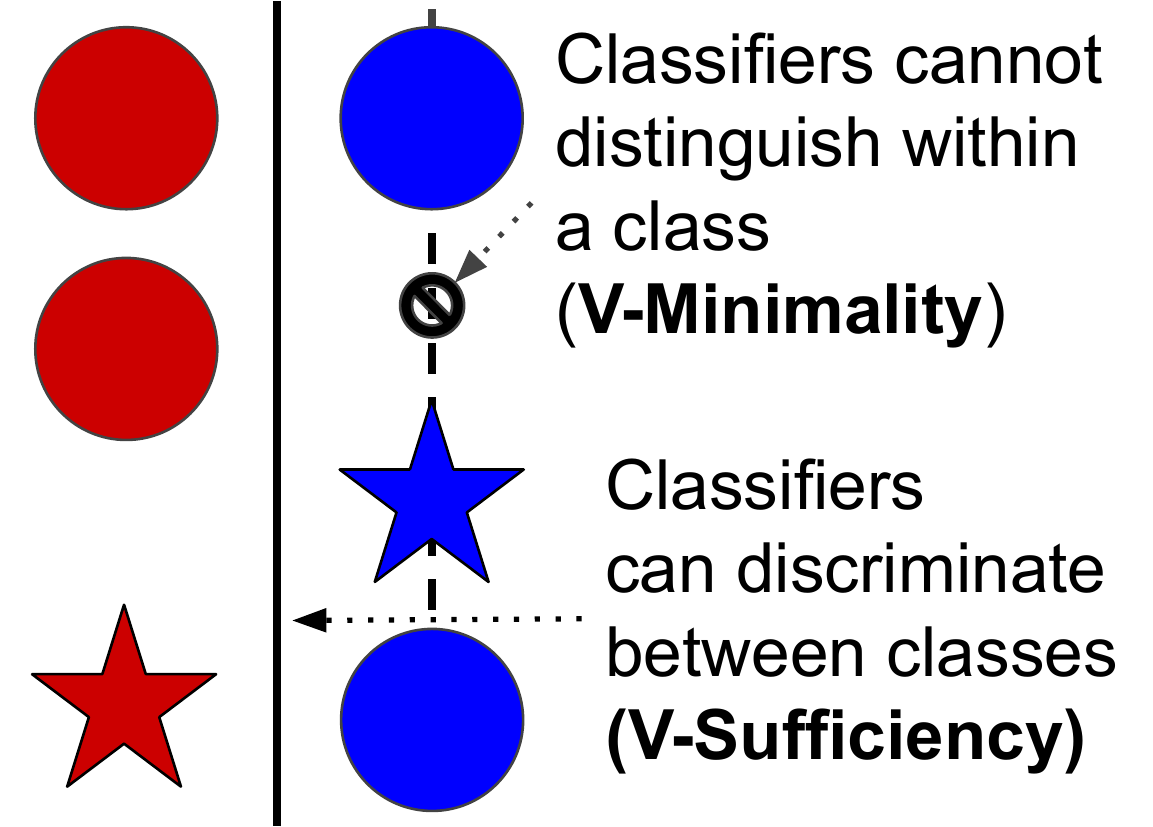}}
  \caption{Schematic DIB}
 \label{fig:Qmin_Qsuff_rep_med}
 \end{subfigure}
 \hfill
  \begin{subfigure}{0.20\textwidth}
\includegraphics[width=\textwidth]{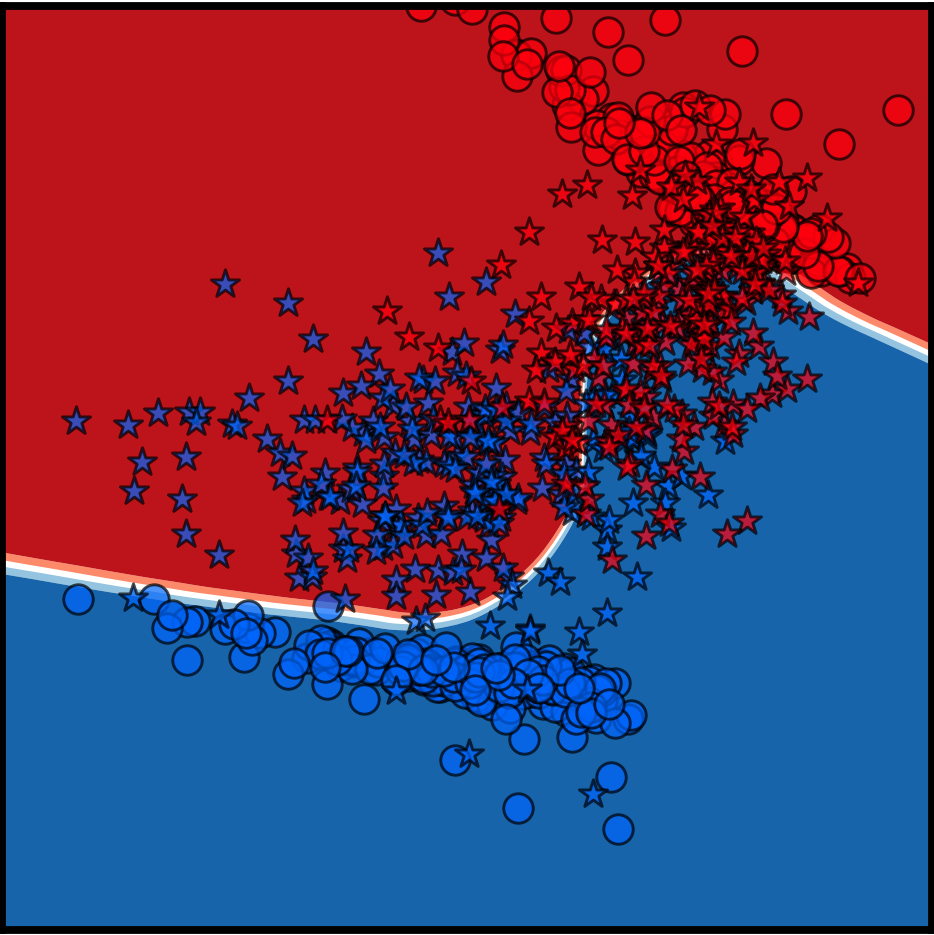} 
  \caption{Empirical J-ERM}
 \label{fig:qmin2d_suff}
 \end{subfigure}
   \begin{subfigure}{0.20\textwidth}
\includegraphics[width=\textwidth]{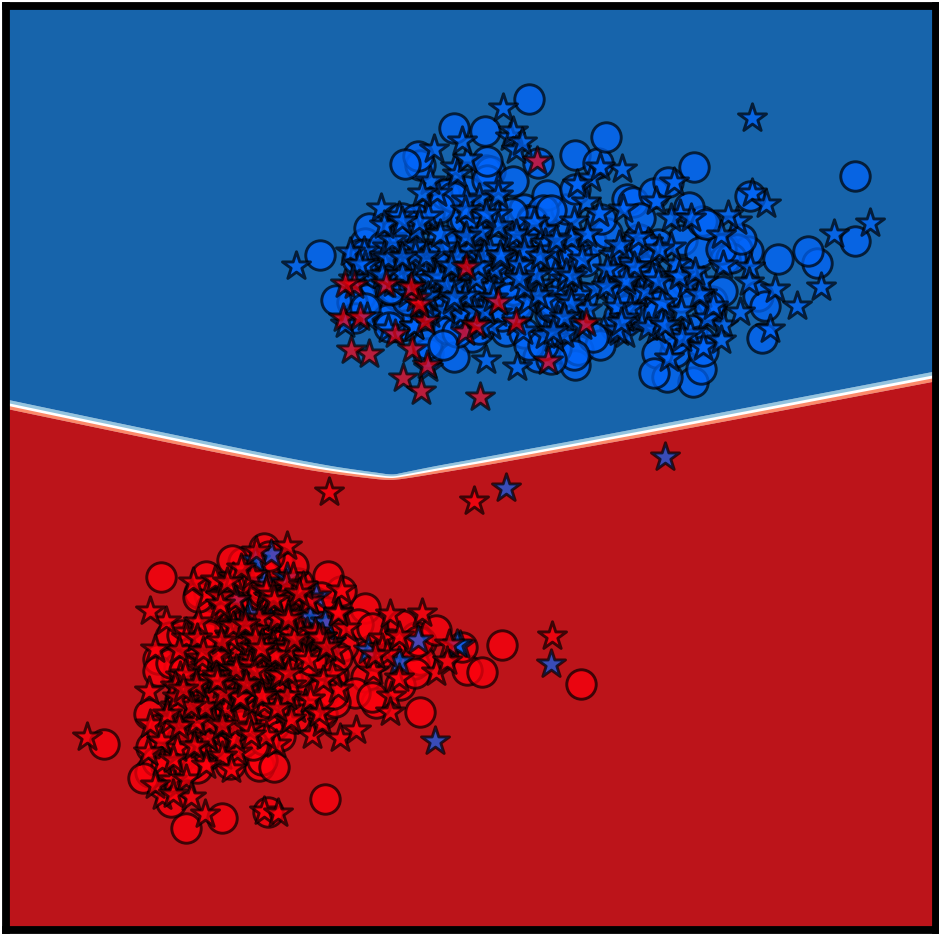} 
  \caption{Empirical DIB}
 \label{fig:qmin2d_min}
 \end{subfigure}
\caption{
\textbf{Left two plots}: illustration of representations learned by joint empirical risk minimization (J-ERM) and our decodable information bottleneck (DIB), for classifiers $\V$ with linear vertical decision boundaries. 
(a) For representations learned by J-ERM, there may exist an ERM that does not generalize;
(b) Representations learned by DIB ensure that any ERM will generalize to the test set (\vminimality)
.
\textbf{Right two plots}: 2D representations encoded by an multi-layer perceptron (MLP) for odd-even classification of 200 MNIST \cite{DBLP:journals/neco/LeCunBDHHHJ89} examples.
The white decision boundary corresponds to a classifer which was trained to perform well on train but bad on test (see \cref{sec:qmin}).
(c) J-ERM allows such classifiers that cannot generalize;
(d) DIB ensures that there are no such classifiers in $\V$.
%
}
\label{fig:schema_rep_med}
\end{figure}
Our main contributions can be summarized as follows:
\begin{itemize}
\item We generalize notions of minimality and sufficiency to consider predictors $\V$ of interest.
\item We prove that such representations are optimal --- every downstream ERM in $\V$ reaches the best achievable test performance --- and can be learned with guarantees using DIB.
\item We experimentally demonstrate that using our representations can increase the performance and robustness of downstream classifiers in average and worst case scenarios.
\item We show that the generalization ability of a neural network is highly correlated with the degree of $\V$-minimality of its hidden representations in a wide range of settings (562 models).
\end{itemize}
\section{Problem Statement and Background}
\label{sec:background}
Throughout this paper, we provide a more informal presentation in the main body, and refer the reader to the appendices for more precise statements. For details about our notation, see \cref{appx:notation}.

\subsection{Problem Set-Up: Representation Learning as a Two-Player Game}
\label{sec:formal_prob_state}
Consider a game between Alice, who selects a classifier $f \in \V$, and Bob, who provides her with a representation to improve her performance. We are interested in Bob's optimal choice. Specifically: 
\begin{enumerate}
\item Alice decides \textit{a priori} on: a predictive family $\family{}$, a task of interest that consists of classifying labels $\rv Y$ given inputs $\rv X$, and a score/loss function $S$ measuring the quality of her predictions.
\item Given Alice's selections, Bob trains an encoder $P_{\rv Z \cond \rv X}$ to map inputs $\rv X$ to representations $\rv Z$.
\item Using Bob's encoder and a dataset $\sampleD{}$ of $M$ input-output pairs $(x,y)$, Alice selects a classifier $\hpred{}$ from all ERMs $\hat{\V}(\mathcal{D}) \defeq  \arg \min_{f \in \V}  \eRisk{}$, where $\eRisk{} \defeq \frac{1}{M} \sum_{y,x \in \mathcal{D}} \E{z \sim P_{\rv Z \cond x}}{S(y,f[z])}$ is an estimate of the risk $\Risk{}=\E{\mathcal{D}}{\eRisk{}}$.
\end{enumerate}

Our goals are to: 
\begin{inlinelist}
\item characterize \textit{optimal} representations $\rv Z^*$ that minimize Alice's expected loss  $\mathrm{R}(\hpred{},\rv Z)$;
\item derive an objective $\mathcal{L}$ that can be optimized to approximate the optimal encoder $P_{\rv Z^* \cond \rv X }$.
\end{inlinelist}

We assume that: 
\begin{inlinelist}
\item sample spaces $\mathcal{Y},\mathcal{Z},\mathcal{X}$ are finite;
\item $S(y,f[z])$ is the log loss $-\log f[z](y)$, where $f[z](y)$ approximates $P_{\rv  Y | \rv Z}(y|z)$ as in \citep{xu2020theory};
\item the family $\V$ satisfies mild constraints that hold for practical classifiers such as neural networks.
See \cref{sec:appx_assumptions} for all assumptions.
\end{inlinelist}
\subsection{Sufficiency, Minimality, and the Information Bottleneck (IB)}
\label{sec:min_sif_ib}

We review IB, an information theoretic method for supervised representation learning. 
IB is built upon the intuition that a representation $\rv Z$ should be maximally informative about $\rv Y$ (sufficient), but contain no additional information about $\rv X$ (minimal) to avoid possible over-fitting. 
Specifically, the set of \textbf{sufficient} representations $\suff{}$ and \textbf{minimal sufficient} $\smin{}$ representations 
are defined as:\footnote{As shown in \cref{sec:proof_minsuffstat}, this common \cite{shamir2010learning,achille2018emergence} definition is a generalization of minimal sufficient \textit{statistics}~\cite{fisher1922mathematical} to stochastic statistics $\rv Z = T(\rv X, \epsilon)$ where $\epsilon$ is a source of noise independent of $\rv X$.
}
\begin{equation}
\suff{} \defeq \arg \max_{\rv z'} \MI{ Y}{ z'} 
\quad\text{and}\quad 
\smin{} \coloneqq \arg\min_{\rv Z' \in \suff{}} \MI{ X}{z'} \label{eqn:minsuff}
\end{equation}
%
The IB criterion (to minimize) can then be interpreted \cite{shamir2010learning} as the Lagrangian relaxation of \cref{eqn:minsuff}: 
\begin{equation}\label{eqn:IB}
\ib{}  \defeq  - \MI{Y}{Z} + \beta * \MI{X}{Z}
\end{equation}
Despite its intuitive appeal, IB suffers from the following theoretical and practical issues: 
\begin{inlinelist}
\item \textbf{Lack of optimality} guarantees for $\rv Z \in \smin{}$. Generalization bounds based on $\MI{Z}{X}$ are a step towards such guarantees \cite{shamir2010learning,vera2018role} but current bounds are still vacuous \cite{rodriguez2019information}.
The strong performance of invertible neural networks \cite{ChangMHRBH18,jacobsen2018revnet} also shows that a small $\MI{X}{Z}$ is not required for generalization;
\item $\ib{}$ is \textbf{hard to estimate} with finite samples
\cite{McAllester,shamir2010learning}.
One has to either restrict the considered setting  \cite{tishby2000information,chechik2005information,rey2012meta}, or optimize variational \cite{ChalkMT16,alemi2016deep,AchilleS18dropout} or non-parametric  \cite{kolchinsky2019nonlinear} bounds;
\item $\ib{}$ is \textbf{invariant to bijections} and thus does not favor simple decision boundaries \cite{amjad2019learning} that can be achieved by a $f \in \V$.
\end{inlinelist}

We stipulate that these known issues stem from a common cause: IB uses mutual information, which is agnostic to the predictive family $\V$ of interest. 
To remedy this, we leverage the recently proposed $\V$-information \cite{xu2020theory} to formalize the notion of $\V$-minimal $\V$-sufficient representations.

\subsection{$\V$-information}
\label{sec:v_information}

From a predictive perspective, mutual information $\MI{Y}{Z}$ corresponds to the difference in expected log loss when predicting $\rv Y$ with or without $\rv Z$ using the best possible probabilistic classifier. 
\begin{align}
\MI{Y}{Z} \defeq \op{H}{\rv Y} - \op{H}{\rv Y | \rv Z} &= \op{H}{\rv Y} - \E{z,y \sim P_{\rv Z, \rv Y}}{- \log P_{\rv Y  \cond \rv Z}} \\
&= \op{H}{\rv Y} - \inf_{f \in {\color{burntorange} \mathcal{U}}} \E{z,y \sim P_{\rv Z, \rv Y}}{- \log f[z](y) } & \text{Strict Properness \cite{gneiting2007strictly}}  \label{eqn:properness}
\end{align}
where $\mathcal{U}$ is the collection of all predictors from $\mathcal{Z}$ to distributions over $\mathcal{Y}$, which we call \textit{universal}.
As the optimization in \cref{eqn:properness} is over $\mathcal{U}$, $\MI{Y}{Z}$ measures information that might not be ``usable'' by $f \in \V \subset \universal{}$.
\citepos{xu2020theory} resolve this issue by introducing $\V$-information $\DIF{Y}{Z}$ to only consider the information that can be decoded by a predictors of interest $f \in {\color{blue} \V}$ instead of $f \in {\color{burntorange} \universal{}}$.\footnote{\cite{xu2020theory} also replace $\op{H}{\rv Y}$ with $\CHF{Y}{\varnothing}$.
This requires that any $f \in \V$ can be conditioned on the empty set~$\varnothing$.
We keep $\op{H}{\rv Y}$ for simplicity and show that both are equivalent in our setting (\cref{sec:proof_emptyset}).
}
\begin{equation}\label{eqn:Vinformation}
\DIF{Y}{Z} \defeq \op{H}{\rv Y} - \op{H_{\color{blue} \V}}{\rv Y \cond \rv Z} = \op{H}{\rv Y} - \inf_{f \in {\color{blue} \V}} \E{z,y \sim P_{\rv Z, \rv Y}}{- \log f[z](y) }
\end{equation}
$\DIF{Y}{Z}$ has useful properties, it: recovers $\MI{Y}{Z}$ for $\V=\universal{}$, is non-negative, and is zero when $\rv Z$ is independent of $\rv Y$.
Importantly, $\V$-information is easier to estimate than Shannon's information; 
indeed, it corresponds to estimating the risk ($\CHF{Y}{Z}$) and thus inherits \cite{xu2020theory} probably approximately correct (PAC; \cite{valiant1984theory}) estimation bounds that depend on the (Rademacher \cite{bartlett2002rademacher}) complexity of $\V$.
\section{Methods}
\label{sec:methods}
In this section, we define $\V$-minimal $\V$-sufficient representations, prove that they are optimal in the two-player representation learning game, and discuss how to approximately learn them in practice.
\subsection{$\V$-Sufficiency and Best Achievable Performance}
\label{sec:theory_qsuf}
Let us study Alice's best risk $\min_{f \in \V} \mathrm{R}(f,\rv Z)$ using a representation $\rv Z$. 
This tight lower bound on her performance looks strikingly similar to $\CHF{Y}{Z}$, which is controlled by $\DIF{Y}{z}$ (see \cref{eqn:Vinformation}).
As a result, if Bob maximizes $\DIF{Y}{z}$, he will ensure that Alice can achieve the lowest loss.

\begin{restatable}{definition}{vsuff}\label{def:vsuff} A representation $\rv Z$ is said to be \textbf{$\V$-sufficient} if it
maximizes $\V$-information with the labels.
We denote all such representations as $\Qsuff{} \defeq \arg \max_{\rv z'} \DIF{Y}{z'}$.
\end{restatable}
\begin{restatable}{proposition}{optimalQsuf}\label{proposition:opt_qsuf}
$\rv Z$ is $\V$-sufficient $\iff$  there exists $f^* \in \V$  whose test loss when predicting from $\rv Z$ is the best achievable risk, i.e., $\mathrm{R}(f^*,\rv Z) = \min_{\rv Z} \min_{f \in \V} \Risk{}$. 
\end{restatable}
 Although the previous proposition may seem trivial, it bears important implications, namely that contrary to the sufficiency term of IB one should maximize $\DIF{Y}{Z}$ rather than $\MI{Y}{Z}$ when predictors live in a constrained family $\V$.
Indeed, ensuring sufficient $\MI{Y}{Z}$ does not mean that there is a classifier $f \in \V$ that can ``decode'' that information.
\footnote{Notice that $\DIF{Y}{Z}$ corresponds to the variational lower bound on $\MI{Y}{Z}$ used by \citet{alemi2016deep}.
We view $\DIF{Y}{Z}$ as the correct criterion rather than an estimate of $\MI{Y}{Z}$.}
We prove our claims in \cref{sec:proof_sufficiency}.
%
\subsection{$\V$-minimality and Generalization}
\label{sec:theory_opt_rep}
We have seen that $\V$-sufficiency ensures that Alice \textit{could} achieve the best loss.
In this section, we study what representations Bob should chose to guarantee that Alice's ERMs \textit{will} perform optimally by ensuring that any ERM generalizes beyond the training set.

IB suggests minimizing the information $\MI{Z}{X}$ between the representation $\rv Z$ and inputs $\rv X$ to avoid over-fitting.
We, instead, argue that only the information that can be decoded by $\mathcal{V}$ matters, and would thus like to minimize $\DIF{X}{Z}$.
However, the latter is not defined as $\rv X$ does not generally take value in (t.v.i.) $\mathcal{Y}$, the sample space of $\V$'s co-domain.
For example, in a $32\times 32$ image binary classification, $\mathfam{Y}=\{0,1\}$ but $\mathfam{X}=[0,\dots,256]^{1024}$ so classifiers $f \in \V$ cannot predict $x \in \mathfam{X}$.
To circumvent this, we decompose $\rv X$ into a collection of r.v.s $\rv N$ that t.v.i. $\mathcal{Y}$, so that $\DIF{N}{Z}$ is well defined.
Specifically, let $\rv X_y$,$\rv Z_y$ be ``conditional r.v.'' s.t. $P_{\rv X_y} = P_{\rv X | y}$, $P_{\rv Z_y} = P_{\rv Z | y}$, $P_{\rv X_y,\rv Z_y} = P_{\rv X,\rv Z | y}$. 
We define the $y$ decomposition of $\rv X$ as r.v.s that arise by all possible labelings of $\rv X_y$:\footnote{Such (deterministic) labelings are also called ``random labelings'' \cite{zhang2016understanding} as they are semantically meaningless.}
\begin{equation}\label{eqn:dec}
\decxy{} \defeq \{ \rv N \cond \exists t' : \mathcal{X} \to \mathcal{Y} \ s.t. \ \rv N = t'(\rv X_y)\}
\end{equation}
We can now define the average $\V$-information between $\rv Z$ and the $y$ decompositions of $\rv X$ as:
%
\begin{equation}\label{eqn:DIFxzCy}
  \DIFxzCy{} \defeq \frac{1}{|\mathcal{Y}|} \sum_{y \in \mathcal{Y}} \frac{1}{|\decxy{}|} \sum_{\rv N \in \decxy{}} \DIF{N}{Z_y}
\end{equation}
$\DIFxzCy{}$ essentially measures how well predictors in $\V$ can predict arbitrary labeling
$\rv N \in \decxy{}$
of examples with the same underlying label $y$. 
Replacing the minimality term $\MI{X}{Z}$ by $\DIFxzCy{}$ we get our notion of $\V$-minimal $\V$-sufficient representations.
\begin{restatable}{definition}{vmin}\label{def:vmin}
$\rv Z$ is \textbf{$\V$-minimal} $\V$-sufficient if it is $\V$-sufficient \textit{and} has minimal average $\V$-information with $y$ decompositions of $\rv X$.
We denote all such $\rv Z$ as $\Qsmin{} \defeq \arg \min_{\rv Z \in \Qsuff{}} \DIFxzCy{}$.
\end{restatable}
Intuitively, a representation is $\V$-minimal if no predictor in $\V$ can assign different predictions to examples with the same label.
Consequently, predictors will not be able to distinguish train and test examples and must thus perfectly generalize.
In this case, predictors will perform optimally as there is at least one which does ($\V$-sufficiency; \cref{proposition:opt_qsuf}).
We formalize this intuition in \cref{appx:proof_optimal}:
\begin{restatable}{theorem}{optimalQmin}\label{theo:opt_qmin}
(Informal) Let $\V$ be a predictive family, $\sampleD{}$ a dataset, and assume labels $\rv Y$ are a deterministic function of the inputs $t(\rv x)$.
If $\rv Z \in \Qsmin{}$ be $\V$-minimal $\V$-sufficient, then the expected test loss of any ERM $\hpred{} \in \hat{\V}(\mathcal{D})$ is the best achievable risk, i.e., $\mathrm{R}(\hpred{},\rv Z)= \min_{\rv Z} \min_{f \in \V} \Risk{}$.
\end{restatable}
As all ERMs reach the best risk, so does their expectation, i.e., any $\rv Z \in \Qsmin{}$ is optimal.
We also show in \cref{appx:proof_prop} that $\V$-minimality and $\V$-sufficiency satisfy the following properties:
\begin{restatable}{proposition}{Vminsuff}\label{proposition:Vminsuff}
Let $\mathfam{V} \subseteq \mathfam{V}^+$ be two families and $\universal{}$ the universal one. If labels are deterministic:
\begin{itemize}
\item \textbf{Recoverability} 
The set of $\universal{}$-minimal $\universal{}$-sufficient representations corresponds to the minimal sufficient representations that t.v.i. in the domain of $\universal{}$, i.e., $\Usmin = \smin \cap \mathcal{Z}$.
\item \textbf{Monotonicity} $\mathfam{V}^+$-minimal $\V$-sufficient representations are  $\V$-minimal $\V$-sufficient, i.e., $\arg \max_{\rv Z \in \Qsuff{}} \op{I_{\mathfam{V}^+}}{\rv Z \rightarrow \decxY{} } \subseteq \Qsmin{}$.
\item \textbf{Characterization} $\rv Z \in \Qsmin{} \iff \rv Z \in \Qsuff{} $ and $\DIFxzCy{} = 0$.
\item \textbf{Existence} At least one $\universal{}$-minimal $\universal{}$-sufficient representation always exists, i.e.,  $|\Qsmin{}| > 0$.
\end{itemize}
\end{restatable}
The recoverability shows that our notion of $\V$-minimal $\V$-sufficiency is a generalization of minimal sufficiency.
As a corollary, IB's representations are optimal when Alice is unconstrained in her choice of predictors $\V=\universal{}$. 
The monotonicity implies that minimality with respect to (w.r.t.) a larger $\mathcal{V}^+$ is also optimal.
Finally, the characterizations property gives a simple way of testing for $\V$-minimality.
\subsection{Practical Optimization and Estimation}
\label{sec:practical_optim}
%
In the previous section we characterized optimal representations $\rv Z^* \in \Qsmin{}$.
Unfortunately, Bob cannot learn these $\rv Z^*$ as it requires the underlying distribution $P_{\rv X,\rv Y}$.
We will now show that he
can nevertheless approximate $\rv Z^*$ in a sample- and computationally- efficient manner.

\textbf{Optimization}. Learning $\rv Z \in \Qsmin{}$ requires solving a constrained optimization problem.
Similarly to IB, we minimize the \textit{decodable information bottleneck} (DIB), a Lagrangian relaxation of \cref{def:vmin}:
\begin{equation}
\label{eqn:dib}
\dib{} \defeq  - \DIF{Y}{Z} + \beta *  \DIFxzCy{}
\end{equation}
Notice that each $\DIF{\cdot}{Z}$ has an internal optimization.
In particular $\DIFxzCy{}$ turns the problem into a min (over $\rv Z$) - max (over $f \in \V$) optimization, which can be hard to optimize \cite{goodfellow2014distinguishability,metz2016unrolled}.
We empirically compare methods for optimizing $\dib{}$ in \cref{sec:appx_minimax} and show that joint gradient descent ascent performs well if we ensure that the norm of the learned representation cannot  diverge.
\newcommand{\vspacing}{$\vphantom{\sum_{i}^T}$}
\begin{figure}
\centering
\makebox[0.45\linewidth][l]{\hspace{12pt}\begin{subfigure}{0.45\linewidth}
        \small
        \SetInd{0.3em}{0.6em}
        \begin{algorithm}[H]
            \DontPrintSemicolon
            \SetKwFunction{Hv}{$\mathrm{H}$}
  \SetKwFunction{FMain}{$\hat{\mathcal{L}}_{\mathrm{DIB}}$}
  \SetKwProg{Fn}{Function}{:}{}
  \Fn{\Hv{$\V$, $\mathbf{x},\mathbf{y}$, $P_{\rv Z|\rv X}$}}{
        \vspacing$\mathbf{z} \gets \text{sample once from each } P_{\rv Z| \mathbf{x}}$\\
        \vspacing\KwRet {$\inf_{f \in \V}  \sum_{z,y \in \mathbf{z},\mathbf{y}} \frac{-\log f[z](y)}{|\mathcal{D}|}$}
  }
  \vspace{0.5em}
  \SetKwProg{Pn}{Function}{:}{\KwRet}
  \Pn{\FMain{$\V$, $\mathcal{D}$, $P_{\rv Z|\rv X}$, $K$, $\beta$}}{
        \vspacing$(\mathbf{x},\mathbf{y}),\mathcal{L}_{\V \text{min}},\mathcal{Y} \gets \mathcal{D},0,\mathrm{unique}(\mathbf{y})$ \\
        \vspacing$\mathcal{L}_{\V \text{suff}} \gets $ \Hv{$\V$, $\mathbf{x}$, $\mathbf{y}$,  $P_{\rv Z|\rv X}$} \\
        \For{$y$ in $\mathcal{Y}$}{
        \vspacing$\mathbf{x}_y \gets \mathbf{x}[\mathbf{y}==y]$ \\
        \vspacing\For{$k\gets1$ \KwTo $K$}{
            \vspacing$\mathbf{n} \gets $ random\_choice($\mathcal{Y}$, size=$|\mathbf{x}_y|$)\\
            \vspacing$\mathcal{L}_{\V \text{min}}$ += $ \frac{ \mathrm{H}(\V,\mathbf{x}_y,\mathbf{n},P_{\rv Z | \rv X}) }{|\mathcal{Y}|*K} $
        }
        }
        \vspacing\KwRet {(const) $+ \mathcal{L}_{\V \text{suff}} - \beta* \mathcal{L}_{\V \text{min}}$}\\
  }
        \end{algorithm}
        \vspace{-0.7em}
        \caption{Pseudo-code for $\edib{}$}
        \label{alg:pseudo_dib}
    \end{subfigure}}
\begin{subfigure}{0.47\linewidth}
\includegraphics[width=\linewidth]{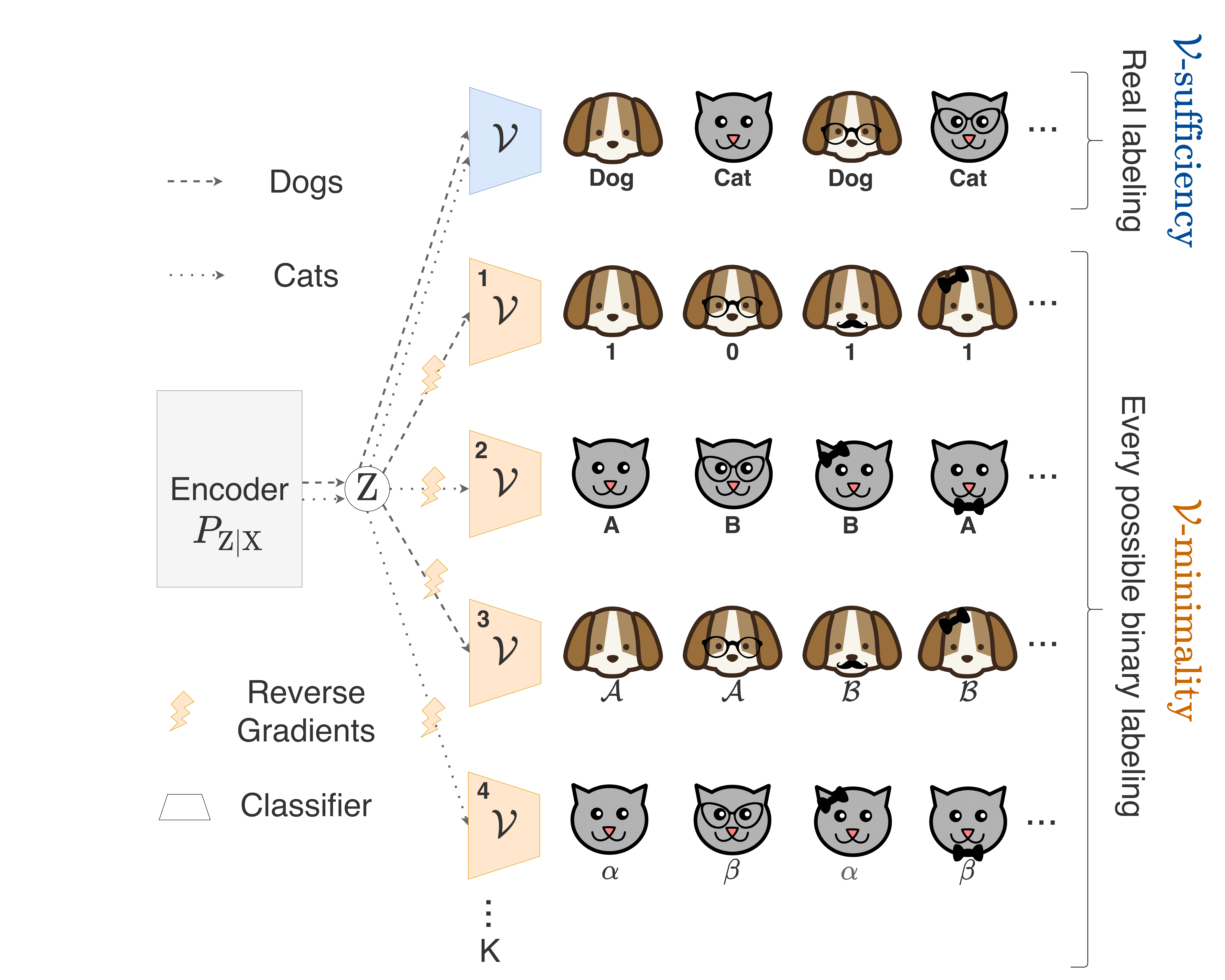}
\caption{DIB with neural networks}
 \label{fig:dib_neural_net}
\end{subfigure}
\caption{
\textbf{Practical DIB}
(a) Pseudo-code to compute the $\edib{}$; (b) Illustration of DIB to train a neural encoder.
$\V$-sufficiency corresponds to the standard log loss. 
$\V$-minimality heads are trained to classify $K$ arbitrary labeling within each class but the gradients w.r.t. the encoder are reversed so that $\rv Z$ cannot be used for that task.
Each head has different parameters but the same architecture $\V$.
}\label{fig:dib_practice}
\end{figure}

\textbf{Estimation}. A major benefit of $\dib{}$ over $\ib{}$ is that it can be estimated with guarantees using finite samples.
Namely, if Bob has access to a training set $\sampleD{}$, he can estimate $\dib{}$ reasonably well.
In practice, we: 
\begin{inlinelist}
\item use $\mathcal{D}$ to estimate all expectations over $P_{\rv X, \rv Y}$;
\item use samples from Bob's encoder $z \sim P_{\rv Z \cond x}$;
\item estimate the average over $\rv N \in \decxy{}$ in \cref{eqn:DIFxzCy} using $K$ samples. 
\end{inlinelist}
\Cref{alg:pseudo_dib} shows a (naive) algorithm to compute the resulting estimate $\edib{}$.
Despite these approximations, we show in \cref{appx:practical_estimation} that $\dib{}$ inherits $\V$-information's estimation bounds \cite{xu2020theory}.
\begin{restatable}{proposition}{pac}\label{proposition:pac}
(Informal) Let $\mathfrak{R}_{|\mathcal{D}|}$ denote the $|\mathcal{D}|$ samples Rademacher complexity. 
Assuming the loss is always bounded $|\log f[x](y)| < C$ then with probability at least $1-\delta$, $\edib{}$ described in \cref{alg:pseudo_dib} approximates $\dib{}$ with error less than
$2 \mathfrak{R}_{|\mathcal{D}|}(\log \circ \V) + \beta \log |\mathcal{Y}| + C \sqrt{ \frac{2\log \frac{1}{\delta}}{M}}$.
\end{restatable}

The fact that the estimation error in \cref{proposition:pac} grows with the (Rademacher) complexity of $\V$, shows that the error is largest for $\V = \universal{}$ corresponding to $\ib{}$.
We also see a trade-off in Alice's choice of $\V$.
A more complex $\V$ means the estimation of $\edib{}$ is harder for Bob (\cref{proposition:pac}), but Alice's prediction will improve (smaller $\min_{\rv Z} \min_{f \in \V} \Risk{}$; \cref{theo:opt_qmin}).

\textbf{Case study: neural networks.} Suppose that $\V$ is a specific neural architecture, the encoder $P_{\rv Z \cond \rv X}$ is parametrized by a neural network $q_\theta$, and we are interested in cat-dog classification.
As shown in \cref{fig:dib_neural_net}, training $q_\theta$ with DIB corresponds to fitting $q_\theta$ with multiple classification heads, each having exactly the same architecture $\V$ but different parameters.
The $\V$-sufficiency head (in blue) tries to classify cats and dogs.
Each of the $K$ (typically 3-4, see \cref{sec:appx_reindexing}) $\V$-minimality heads (in orange) ensure that the representation cannot be used to classify an arbitrary (fixed) labeling of cats or dogs.
In practice, the encoder and heads are trained jointly but gradients from $\V$-minimality heads are reversed. 
The $\V$-minimality losses are also multiplied by a hyper-parameter $\beta$. 
\section{Experiments}
We evaluate our framework in practical settings, 
focusing on:
\begin{inlinelist}
\item the relation between $\V$-sufficiency and Alice's best achievable performance; 
\item the relation between $\V$-minimality and generalization; 
\item the consequence of a mismatch between  $\V_{Alice}$ and the functional family $\V_{Bob}$ w.r.t. which $\rv Z$ is sufficient or minimal --- especially in IB's setting $\V_{Bob}=\universal{}$;
\item the use of our framework to predict generalization of trained networks.
\end{inlinelist}
Many of our experiments involve sweeping over the complexity of families $\V^- \subseteq \V \subseteq \V^+$, we do this by varying widths of MLPs --- with $\V \to \universal{}$ in the infinite width limit \cite{cybenko_approximation_1989,Hornik91}. 
Alternative ways of sweeping over $\V$ are evaluated in \cref{sec:sweep}.
%
\subsection{\vsufficiency: Optimal Representations When the Data Distribution is Known}
\label{sec:qsuf}

\begin{figure}
\centering
 \begin{subfigure}{0.37\linewidth}
\includegraphics[width=\linewidth]{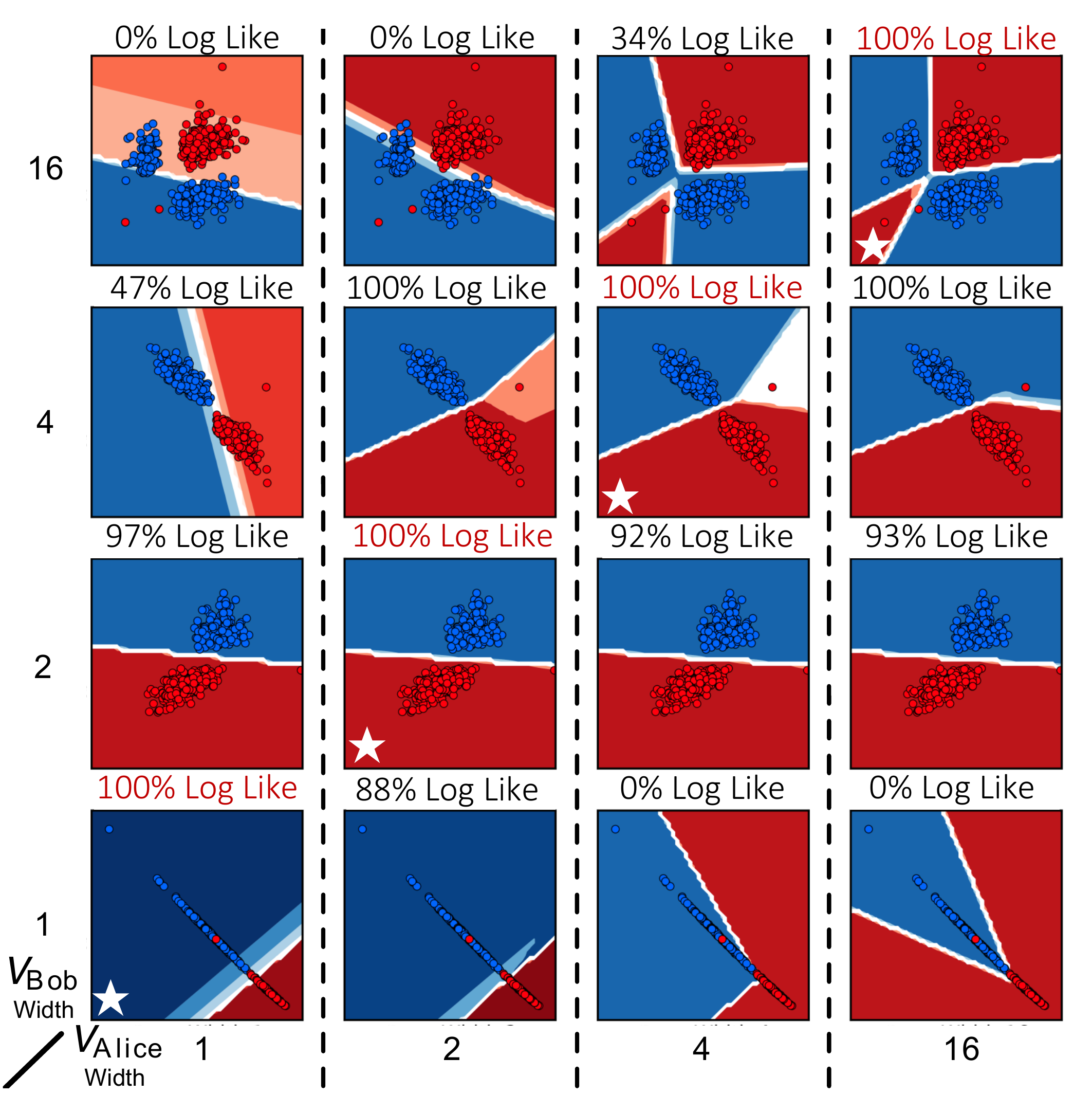} 
  \caption{2D Visualization}
 \label{fig:2d_qsuf}
 \end{subfigure}
 \quad\quad
  \begin{subfigure}{0.51\linewidth}
\includegraphics[width=\linewidth]{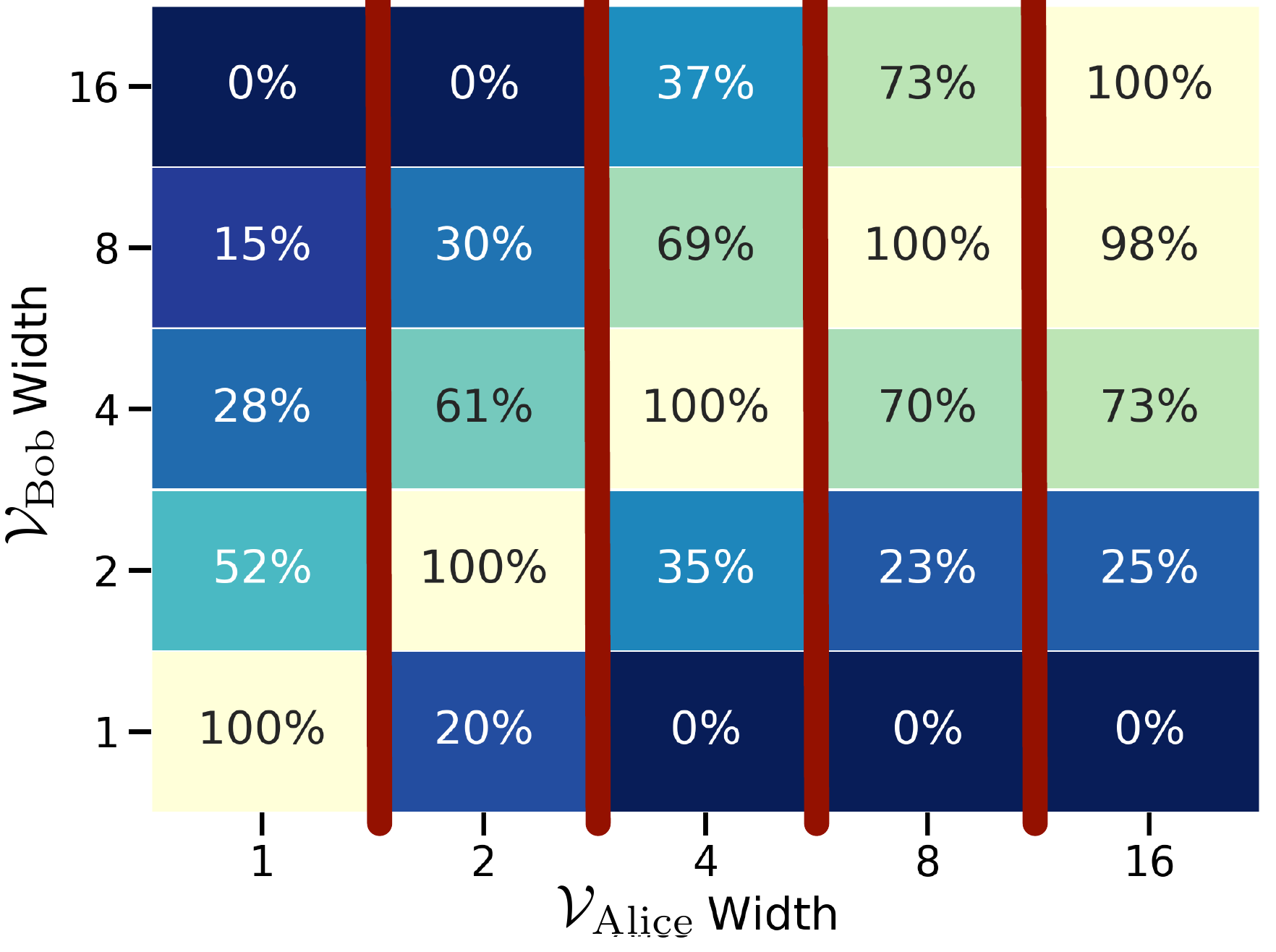} 
  \caption{Scaling Up}
 \label{fig:scaled_qsuf}
 \end{subfigure}
\caption{\textbf{Optimality of \vsufficiency.}
Plots of Alice's best possible performance with different $\V_{Bob}$-sufficient representations. The log likelihood is column-wise scaled from $0$ to $100$, and vertical separators are present to discourage between-column comparison.
The predictive families are MLPs with varying widths.
(a) Samples of Bob's 2D representations along with Alice's decision boundaries for an odd-even CIFAR100 binary classification;
(b) Same scaled log likelihood but using 8D representations,  the standard CIFAR100 dataset, and averaging over 5 runs.}
\label{fig:qsuf}
\end{figure}
We study optimal representations when Alice has access to the data distribution $P_{\rv Z \times \rv Y}$.
Alice's risk $\Risk{}$ in such setting is important as it is a tight lower bound on her performance in practical settings (see \cref{sec:theory_qsuf}).
We consider the following setting: Bob trains a ResNet18 encoder \cite{HeZRS16} by maximizing $\DIQbob{Y}{Z}$, Alice freezes it and trains her own classifier $f\in \V_{Alice}$ using the underlying $P_{\rv Z \times \rv Y}$, i.e., $f$ is trained and evaluated on the \textit{same} dataset.
See \cref{sec:hyperparam_qsmin} for experimental details.

\textbf{Which $\V$-sufficiency should Bob chose?}
\Cref{proposition:opt_qsuf} tells us that 
Bob's optimal choice is $\V_{Bob} = \V_{Alice}$. 
If he opts for a larger family $\V_{Alice} \subseteq \V_{Bob}$, the representation $\rv Z$ is unlikely to be decodable by Alice.
If $\V_{Bob} \subseteq \V_{Alice}$, he will unnecessarily constrain $\rv Z$. 
%
We first consider a setting that can be visualized: classifying the parity of CIFAR100 class index \cite{krizhevsky2009learning} using 2D representations.
\Cref{fig:2d_qsuf} shows samples from $\rv Z$ and Alice's decision boundaries.
To highlight the optimal $\V_{Bob}$ for a given $\V_{Alice}$, we scale the performance of each column from $0$ to $100$ in the figure.
As predicted by \cref{proposition:opt_qsuf}, the best performance is achieved at $\V_{Alice}=\V_{Bob}$. 
The worst predictions arise when $\V_{Alice} \subseteq \V_{Bob}$, as the representations cannot be separated by Alice's classifier (e.g. $\V_{Bob}$ width 16 and $\V_{Alice}$ width 1). 
This suggests that IB's sufficiency (infinite width $\V_{Bob}=\universal{}$) is undesirable when  $\V_{Alice}$ is constrained.
\Cref{fig:scaled_qsuf} shows similar results in 8D across 5 runs. 
See \cref{sec:QsufAll} for more settings.
\subsection{\vminimality: Optimal Representations for Generalization}
\label{sec:qmin}
\Cref{theo:opt_qmin} states that $\V$-minimality ensures all ERMs can generalize well.
We investigate whether this is still approximately the case in practical settings, i.e., when Bob optimizes $\edib{}$.

\noindent \textbf{Experimental Details.}
Our claim concerns \textit{all} ERMs $\hat{\V}^*(\mathcal{D})$, which cannot be supported by training a few $\hpred{} \in \hat{\V}^*(\mathcal{D})$. 
Instead, we evaluate the ERM that performs worst on the test set (Worst ERM), i.e., $\arg \max_{f \in \hat{\V}^*(\mathcal{D})} \Risk{}$.
We do so by optimizing the following Lagrangian relaxation $\arg \min_{f \in \V}  \eRisk{} - \gamma \Risk{}$ (see \cref{sec:appx_antireg}).
 As our theory does not impose constraints on $\rv Z$, we need an encoder close to a universal function approximator.
We use a 3-MLP encoder with around 21M parameters and a 1024 dimensional $\rv Z$.
Since we want to investigate the generalization of ERMs resulting from Bob's criterion, we do not use (possibly implicit) regularizers such as large learning rate~\cite{li2019genlr}.
For more experimental details see \cref{sec:hyperparam_qsmin}.

\begin{figure}
\centering
  \begin{subfigure}{0.31\textwidth}
\includegraphics[width=\textwidth]{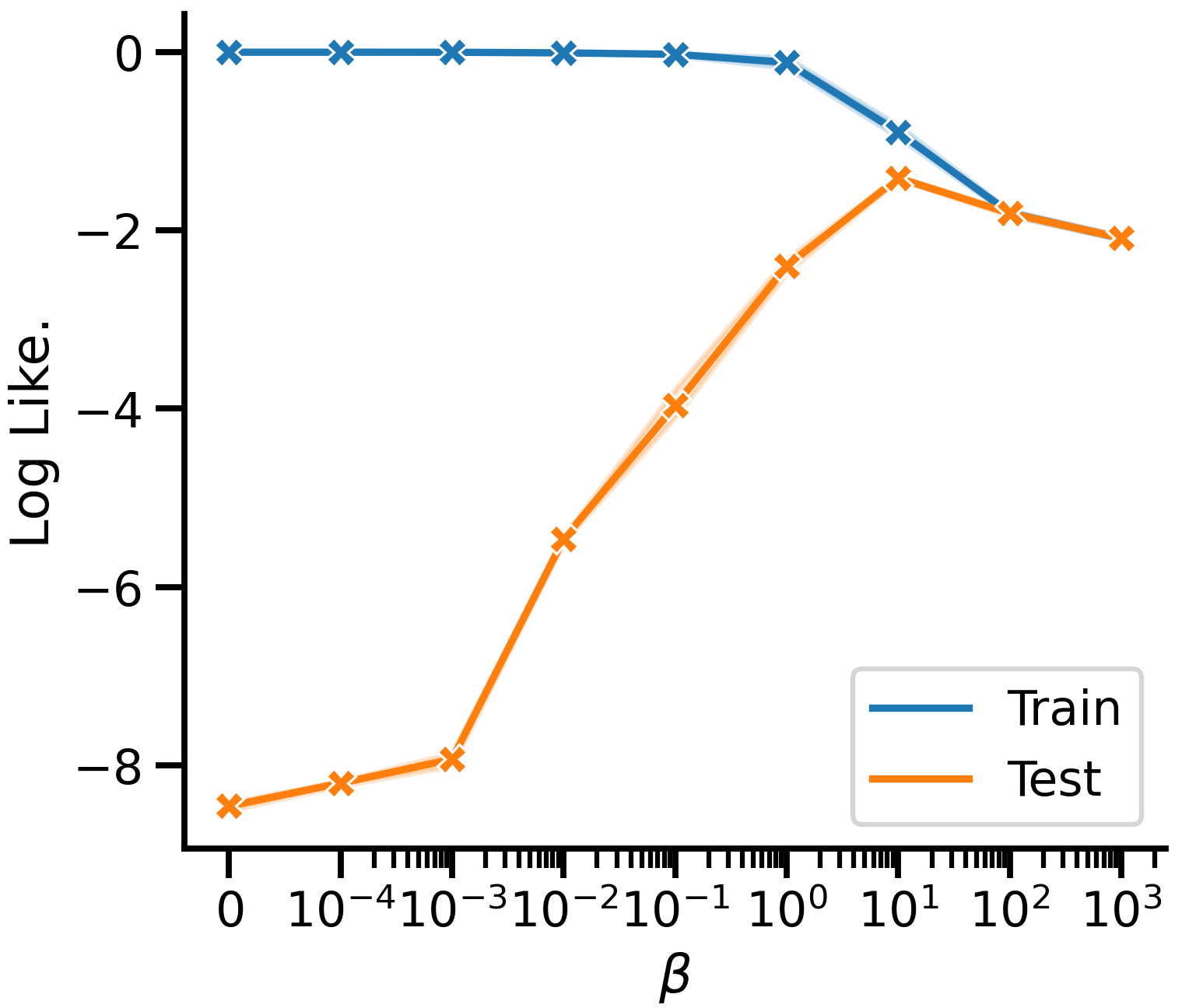} 
  \caption{Generalization Gap}
 \label{fig:qminimality_gap_loglike}
 \end{subfigure}
  \hfill
   \begin{subfigure}{0.31\textwidth}
\includegraphics[width=\textwidth]{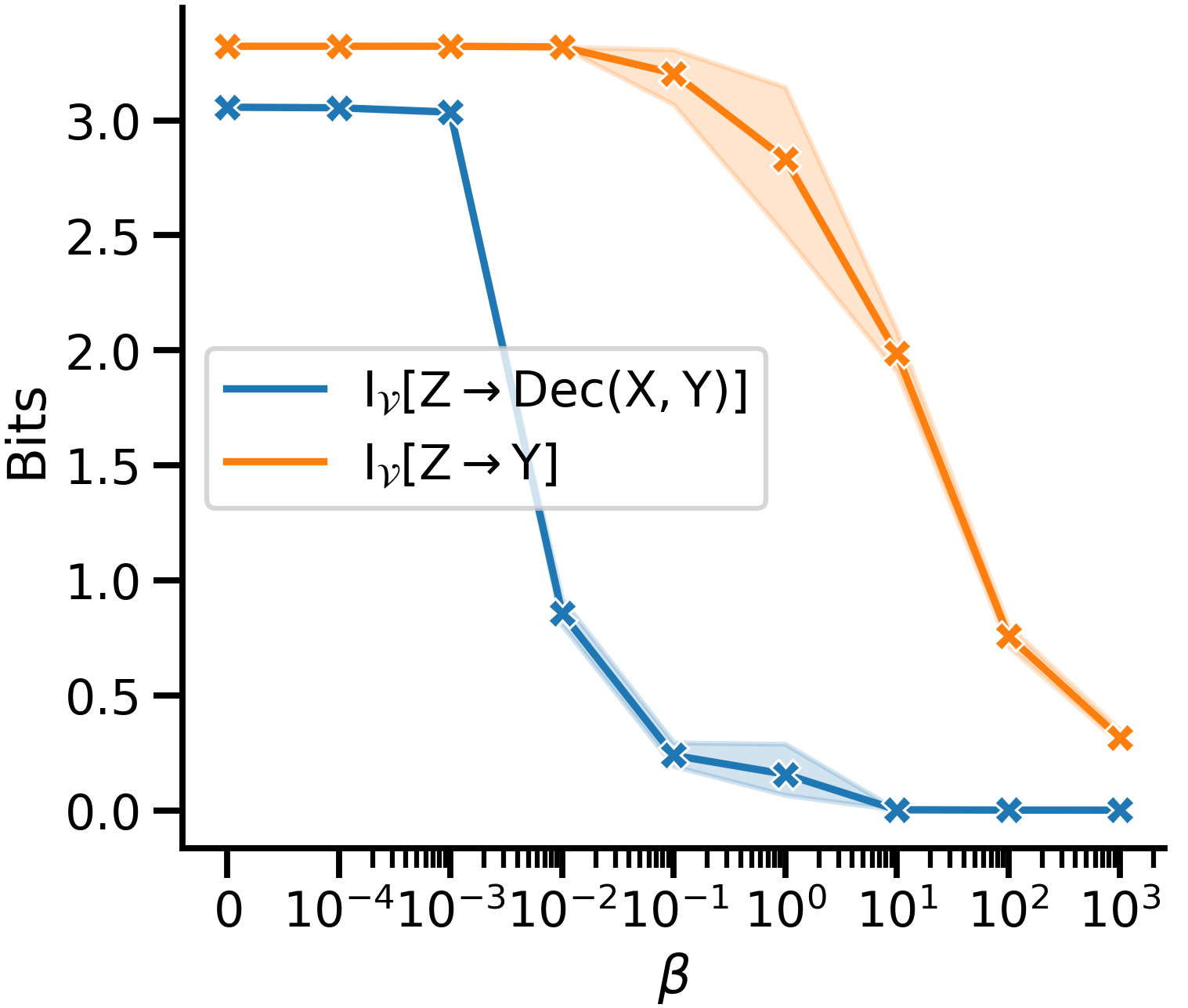} 
  \caption{$\V$-Sufficiency,-Minimality}
 \label{fig:qminimality_Vbits}
 \end{subfigure}
   \hfill
    \begin{subfigure}{0.31\textwidth}
\includegraphics[width=\textwidth]{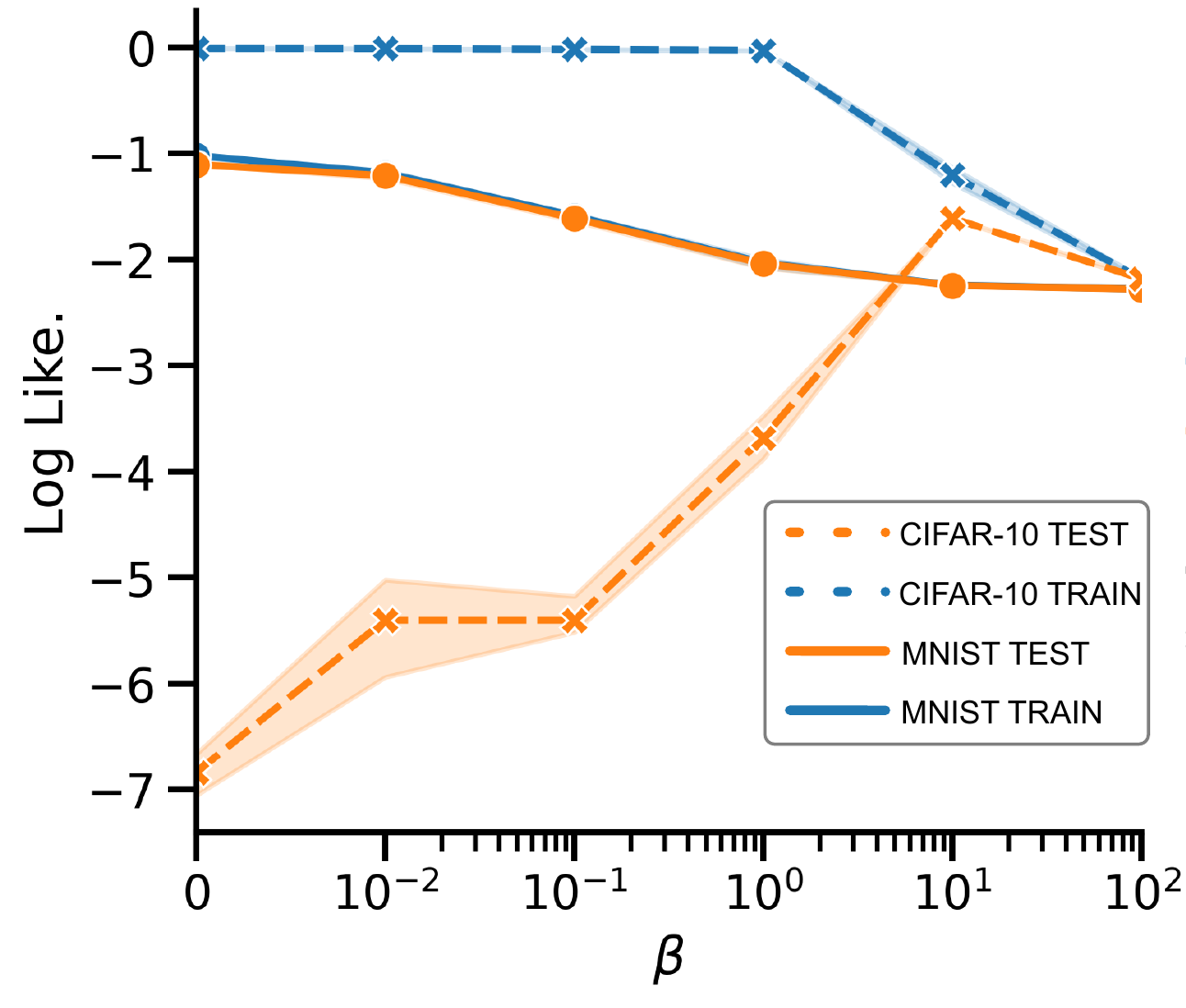}
  \caption{CIFAR10+MNIST}
 \label{fig:qminimality_cifar10mnist}
 \end{subfigure}
\caption{
\textbf{Effect of DIB on generalization}
Left two plots (CIFAR10):
Impact of DIB's $\beta$ on:
(a) the train and test performance of Alice worst ERM;
(b) the $\eDIFxzCy{}$ and $\eDIF{Z}{Y}$ of Bob's representation $\rv Z$.
As $\beta$ increases, $\rv Z$ becomes $\V$-minimal which increases the Alice's test performance until $\rv Z$ is far from $\V$-sufficient.
Right plot (CIFAR10+MNIST):
Same as (a) but images contain overlaid digits as distractors (see \cref{sec:cifar10mnist}).
$\V$-minimality avoids over-fitting by removing spurious MNIST information.
The shaded areas indicate $95\%$ bootstrap confidence interval on 5 runs.}
\label{fig:qminimality}
\end{figure}
\noindent\textbf{What is the impact of DIB's $\beta$ ?} 
We train representions on CIFAR10 with various $\beta$ to investigate the effect of $\eDIFxzCy{}$  and $\eDIF{Z}{Y}$ (\cref{fig:qminimality_Vbits}) on Alice's performance (\cref{fig:qminimality_gap_loglike}). 
Increasing $\beta$ results in a decrease in $\eDIFxzCy{}$ which monotonically shrinks the train-test gap.
This suggests that, although our theory only applies for \vminimality,  $\DIFxzCy{}$ is tightly linked to generalization even when it is non-zero.
After a certain threshold ($\beta=10$) the generalization gains come at a large cost in $\DIF{Z}{Y}$, which controls the best achievable loss.
This shows that a trade-off (controlled by $\beta$) between  \vminimality (generalization) and \vsufficient (lower bound) arises when Bob has to estimate $\dib{}$ using finite samples $\edib{}$.
%
%

\noindent\textbf{\vminimality and robustness to spurious correlations}.
We overlay MNIST digits as
a distractor on
CIFAR10 (see \cref{sec:cifar10mnist}).
We run the same experiments as with CIFAR10, but we additionally train an ERM from $\V$ to predict MNIST labels, i.e., test whether $\rv Z$ contains decodable information about MNIST.
\Cref{fig:qminimality_cifar10mnist} shows that as $\beta$ increases, predicting MNIST becomes harder.
Indeed decreasing $\eDIFxzCy{}$ removes all  $
\V$-information in $\rv Z$ which is not useful for predicting $\rv Y$.
As a result, Alice's ERM must generalize better as it cannot over-fit spurious patterns.\footnote{\citet{AchilleS18dropout} show that this happens for minimal $\rv Z$.
The novelty is that we obtain similar results when considering \vminimality, which is less stringent (\cref{proposition:Vminsuff}) and does not require the intractable $\MI{X}{Z}$.
}
\addtolength{\tabcolsep}{-3pt}

\begin{table}[ht!]
\centering
\caption{Alice's worst and average case log loss given different representation schemes used by
Bob (lower is better). Standard errors are across 5 runs.}
\label{table:worstcase}
\begin{tabular}{@{}lrrrrrrrr@{}}
\toprule
& No Reg.  %
& Stoch. Rep.    %
& Dropout  %
& Wt. Dec.   %
& VIB    %
& $\V^-$-DIB        %
& $\V^+$-DIB        %
& DIB %
\\ \midrule
Worst
& $10.23 { \scriptstyle \,\pm\, .13 }$   %
& $8.61 { \scriptstyle \,\pm\, .05 }$   %
& $1.90 { \scriptstyle \,\pm\, .00 }$  %
& $10.25 { \scriptstyle \,\pm\, .03 }$  %
& $1.82 { \scriptstyle \,\pm\, .02 }$  %
& $1.54 { \scriptstyle \,\pm\, .03 }$  %
& $1.94 { \scriptstyle \,\pm\, .33 }$  %
& $\textbf{1.41} { \scriptstyle \,\pm\, .01 }$  %
\\
Avg.
& $4.62 { \scriptstyle \,\pm\, .00 }$   %
& $4.34 { \scriptstyle \,\pm\, .04 }$   %
& $1.49 { \scriptstyle \,\pm\, .00 }$  %
& $4.96 { \scriptstyle \,\pm\, .03 }$  %
& $1.76 { \scriptstyle \,\pm\, .01 }$  %
& $1.47 { \scriptstyle \,\pm\, .01 }$  %
& $1.74 { \scriptstyle \,\pm\, .18 }$  %
& $\textbf{1.38} { \scriptstyle \,\pm\, .01 }$  %
\\ \bottomrule
\end{tabular}
\vspace{-10pt}
\end{table}

\addtolength{\tabcolsep}{3pt}    


\noindent\textbf{Which $\V$-minimality should Bob chose?}
We study the effect of $\rv Z \in \mathcal{S}_{\V_{Alice}}$ being minimal w.r.t. families which are larger ($\V^+$-DIB), smaller ($\V^-$-DIB), and equal to $\V_{Alice}$.
In theory, optimal representations would be $\V_{Alice}$-minimal (\Cref{theo:opt_qmin}), which are achieved by DIB and $\V^+$-DIB (Monotonicity).
$\V^+$-DIB should nevertheless be harder to estimate than DIB (\cref{proposition:pac}).
In the last 3 columns of \cref{table:worstcase} we indeed observe that DIB performs best.
$\V^+$-DIB performs worse, 
suggesting that IB's minimality is undesirable in practice.
We also minimize a known lower bound of $\MI{Z}{X}$ (VIB; \cite{alemi2016deep}) and find that it performs worse than DIB.\footnote{
VIB is hard to compare to DIB as it is unclear w.r.t. which family, if any, VIB's solutions are minimal.
}
We show results for different $\beta$ in \cref{sec:appx_Qs}.

\noindent\textbf{Comparison to traditional regularizers}. 
To ensure that the previous experimental gains support our theory and are not necessarily true for other regularizers, we test different regularizers on Bob and see whether they also learn representations that ensure Alice's ERM will generalize. 
In \cref{table:worstcase}, we show the results of:
\begin{inlinelist}
\item No regularization;
\item Stochastic representations (DIB with $\beta=0$);
\item Dropout~\citep{srivastava2014dropout};
\item Weight decay.
\end{inlinelist}
We find that DIB significantly outperforms other regularizers, which supports our claims that $\V$-minimality is well-suited for enforcing generalization. 
We emphasize that we evaluate the regularizers in a setting which is closer to our theory: two-stage game, no implicit regularizers, and evaluated on log likelihood.
We show in \cref{sec:appx_quant_2player} that DIB is a descent regularizer in standard classification settings but performs a little worse than dropout.
\subsection{Probing Generalization in Deep Learning}
\label{sec:beyondworst}
Methods that predict or correlate with the generalization of neural networks have been of recent theoretical and practical interest \cite{DziugaiteR17,Arora0NZ18,MorcosBRB18,JiangKMB19,jiang2019fantastic,NovakBAPS18}, as they can shed light on the inductive biases in deep learning \cite{HardtRS16,SmithL18,BartlettFT17} and prescribe better training procedures \cite{HochreiterS97a,ChaudhariCSLBBC17,KeskarMNST17,WeiM19,NeyshaburSS15}.
Having empirically shown a strong link between the degree of \vminimality{} and generalization (\cref{fig:dib_neural_net}), it is natural to ask whether it can predict the generalization of a trained model.
Specifically, consider the first $L$ layers as an encoder from inputs $\rv X$ to representations $\rv Z_L$, and subsequent layers as a classifier in $\V_L$.
We hypothesize that $\op{\hat{I}_{\V_L}}{\rv Z_{L} \rightarrow \decxY{}; \mathcal{D}}$ correlates well with the generalization of the network.

To test this, we follow \citet{jiang2019fantastic} and train convolutional networks (CNN) with varying hyperparameters (depth, width, dropout, batch size, weight decay, learning rate, dimensionality of $\rv Z$) and retain those that reach $0.01$ empirical risk.
From this set of 562 models, we measure Kendall's rank correlation \cite{KENDALL} between $\op{\hat{I}_{\V_L}}{\rv Z_{L} \rightarrow \decxY{}; \mathcal{D}}$ and the generalization gap of each CNN, i.e., the difference between their train and test performance. For experimental details see \cref{sec:hyperparam_corr}.
\begin{table}[ht!]
\centering
\caption{Rank correlation $\tau$ between different measures and generalization gap (in terms of accuracy $\tau_{acc}$ and log loss $\tau_{logloss}$) of 562 CNNs.
$\V_L$ denotes our $\op{\hat{I}_{\V_L}}{\rv Z_{L} \rightarrow \decxY{}; \mathcal{D}}$.
}
\label{table:correlation}
    \begin{tabular}{@{}lrrrrrrrrr@{}}
    \toprule
    & $\V^+_L$    %
    & $\V^-_L$   %
    & $\V_L$  %
    & Entropy   %
    & Path Norm    %
    & Var. Grad. %
    & Sharp. Mag. %
    \\ \midrule
    $\tau_{acc.}$~\citep{jiang2019fantastic} 
    &    %
    &   %
    &   %
    & $0.148$  %
    & $0.373$  %
    & $0.311$  %
    & $\textbf{0.484}$  %
    \\
    $\tau_{acc}$ (ours)
    & $\textbf{0.482} $   %
    & $0.391$  %
    & $0.471$   %
    & $0.234$  %
    & $0.347$  %
    & $0.332$  %
    & $0.385$  %
    \\
    $\tau_{log loss}$ (ours)
    & $\textbf{0.505} $   %
    & $0.435$  %
    & $0.498$   %
    & $0.164$  %
    & $0.357$  %
    & $0.167$  %
    & $0.233$  %
    \\ \bottomrule
\end{tabular}
\end{table}

\textbf{Does \vminimality correlate with generalization?}
In the last five columns of \cref{table:correlation}, we compare our results ($\V_L$) to the best generalization measure from each categories investigated in \cite{jiang2019fantastic}: the entropy of the output \cite{PereyraTCKH17}, the path norm \cite{NeyshaburTS15}, the variance of the gradients after training (Var. Grad. ; \cite{jiang2019fantastic}), and the ``sharpness'' of the minima  (Sharp. Mag.; \cite{KeskarMNST17}).\footnote{
We report Jiang et al.'s \cite{jiang2019fantastic} results since our experiments and Sharp. Mag. differs slightly from theirs.
}
As hypothesized, $\op{\hat{I}_{\V_L}}{\rv Z_{L} \rightarrow \decxY{}; \mathcal{D}}$ correlates with generalization and even outperforms the baselines.
Similarly to \cref{table:worstcase} we also evaluate minimality with respect to a family larger ($\V_L^+$) or smaller ($\V_L^-$) than $\V_L$.
Surprisingly, $\V_L^+$ performs better than $\V_L$, which might be because larger networks can help optimization of sub-networks $\V_L \subseteq \V_L^+$ as suggested by the Lottery Ticket Hypothesis \cite{FrankleC19}.

To the best of our knowledge $\V$-minimality is the first measure of generalization of a network that only considers a single internal representation $\rv Z_{L}$.
This could be of particular interest in transfer learning, as it can predict how well any model of a certain architecture will generalize when using a specific pretrained encoder.
As $\V$-minimality is a property of a representation rather than the architecture, we show in \cref{sec:appx_correlation} that it can be meaningfully compared \textit{across} different architectures and datasets. 


%
\section{Other Related Work}
\label{sec:previous}
\textbf{Generalized information, game theory and Bayes decision theory}.
If you need a distribution $P^*_{\rv X}$ to act as a representative $\Gamma \subseteq \mathcal{P}(\mathcal{X})$ you should follow the maximum entropy (MaxEnt) principle \cite{Jaynes,csiszar1991least} to minimize the worst-case log loss \cite{topsoe1979information,walley1991statistical}.
\citet{grunwald2004game} generalized MaxEnt to different losses by framing the problem as an adversarial game between nature and a decision maker.
Robust supervised learning \cite{LanckrietGBJ02} can also be framed in a way that suggests to maximize \textit{conditional} entropy \cite{globerson_minimum_2004,farnia2016minimax}.
This line of work focuses on prediction rules (Alice).
Our framing (\cref{sec:formal_prob_state}) extends this literature by incorporating a co-operative agent (Bob), which learns representations to minimize the worst-case loss of the decision maker (Alice).
Although \citep{nguyen2010estimating,duchi2018multiclass} also studied representations using generalized information, they focused on consistency rather than generalization.

\textbf{Extended sufficiency and minimality}. 
Linear sufficiency is well studied \cite{drygas1983sufficiency,baksalary1981linear,kala2017some} but only considers linear encoders and predictors and is used for estimation rather than predictions.
In ML, \citet{CvitkovicK19} incorporated the encoder's family (Bob) to characterize achievable $\rv Z$.
This is complementary to our incorporation of the decoder's family $\V$ (Alice) to characterize optimal $\rv Z$.

\textbf{Kernel Learning}.
There is a large literature in learning kernels \cite{LanckrietCBGJ03,BachLJ04,SonnenburgRSS06,GonenA11} for support vector machines \cite{cortes1995svm}, which implicitly learns a data representation \cite{MicchelliP07}.
The learning is either done by minimizing estimates  \cite{WestonMCPPV00,ChapelleVBM02} or bounds of the generalization error \cite{SrebroB06,CortesMR10a,KloftB12,CortesKM13,LiuLLYW17a}.
The major advantage of our work is that we are not restricted to predictors $\V$ that can be ``kernelized'' and provide an optimality proof.


%
\section{Conclusion and Future Work}
\label{sec:conclusion}
In this work, we propose a prescriptive theory for representation learning.
We first characterize optimal representations $\rv Z^*$
for supervised learning, by defining minimal sufficient representation with respect to a family of classifiers $\V$.
These representations $\rv Z^*$ guarantee that any downstream empirical risk minimizer $f \in \V$ will incur minimal expected test loss, by ensuring that $f$ can correctly predict labels but cannot distinguish examples with the same label. 
We then provide the decodable information bottleneck objective to learn $\rv Z^*$ with PAC-style guarantees.
We empirically show that using $\rv Z^*$ can improve the performance and robustness of image classifiers.
We also demonstrate that our framework can be used to predict generalization in neural networks.

In addition to supporting our theory, our experiments raise interesting questions for future work.
First, results in \cref{sec:qmin} suggest that performance is causally related with the degree of $\V$-minimality of a representation, even though we only prove it for ``perfect'' $\V$-minimality.
A natural question, then, is whether generalization bounds can be derived for approximate $\V$-minimality.
Second, the high correlation between generalization in neural networks and the degree $\V$-minimality (\cref{table:correlation}) suggest that it might be an important quantity to study for understanding generalization in deep learning. 

More generally, our work shows that information theory in theoretical and applied ML can benefit from incorporating the predictive family $\V$ of interest.
For example, we believe that many issues of
mutual information \cite{tschannen2019mutual} in self-supervised learning \cite{Linsker88,HjelmFLGBTB19,CPC}, and IB \cite{saxe2018information,kolchinsky2018caveats,amjad2019learning} 
in IB's theory of deep learning \cite{tishby2015deep,shwartz2017opening} could be solved by taking into account $\V$.
By extending $\V$-information to arbitrary r.v. (through decompositions) we hope to enable its use in those and many other domains.




\section*{Broader Impact}
Our work takes the perspective that an
``optimal'' representation is one such that any classifier
that fits the training data should generalize well to test.
In terms of potential practical benefits, it is possible that
using our optimal representations, one can alleviate the effort of hyperparameter search and selection currently required to tune
deep learning models. This could be a step towards democratizing machine learning to sections of the society without large computational resources – since hyperparameter search is often computationally expensive. We do not anticipate that our work will advantage or disadvantage any particular group.

\begin{ack}
We would like to thank: Naman Goyal for early feedback and best engineering practices;
Brandon Amos for support concerning min-max optimization;
Stephane Deny for suggesting to look for the Worst ERM;
Emile Mathieu, Sho Yaida, and Max Nickel for helpful discussions and feedback;
Jakob Foerster for the name ``decodable'' information bottleneck;
Dan Roy for suggesting to use the term ``distinguishability'' to understand $\V$-minimality;
Chris Maddison for finding typos and small mistakes in the proofs;
and Ari Morcos for tips to help Yann Dubois writing papers. DJS was partially supported by the NSF through the CPBF PHY-1734030, a Simons Foundation fellowship for the MMLS, and by the NIH under R01EB026943.

\end{ack}

\bibliographystyle{IEEEtranN}
\bibliography{bibliography}

\clearpage
\newpage

\appendix
In the following appendices we:
\begin{inlinelist}
\item Formalize our notation in \cref{appx:notation};
\item State and discuss our assumptions in  \cref{sec:appx_assumptions};
\item State and prove our theoretical results in \cref{sec:appx_theory};
\item Provide details for reproducing our results in \cref{sec:appx_reproducibility};
\item Provide and discuss additional results that shed light on many of our design choices \cref{sec:appx_additional}.
\end{inlinelist}

\section{Notation}
\label{appx:notation}
Letters that are upper-case non-italic $\rv Y$, calligraphic $\mathcal{Y}$, and lower-case $y$, represent, respectively, a random variable (r.v.), its associated codomain, and a realization of it. 
When necessary to be explicit, we will say that $\rv Y$ takes value in (t.v.i.) $\mathcal{Y}$.
Conditional distribution will be denoted $P_{\rv Y \cond \rv z} \in \mathcal{P}(\mathcal{Y} | \mathcal{Z})$, and the image of $z$ as $P_{\rv Y \cond z} \in \mathcal{P}({\mathcal{Y}})$, where $\mathcal{P}(\mathcal{Y})$  denotes the collection of all probability measures on $\mathcal{Y}$ with its $\sigma$-algebra and $\mathcal{P}(\mathcal{Y} | \mathcal{Z}) \defeq \{ P : \mathcal{Z} \to \mathcal{P}({\mathcal{Y}}) \}$ is used as a shorthand.
The composition of a function $f$ with a random variable $\rv Z$ will be denoted $f(\rv Z)$.  
Expectations will be written as: $\op{E_{y \sim  P_{\rv Y}}}{y} \defeq \int y\, dP_{\rv Y}$. 
Independence between two r.v.s will be denoted with $\cdot \perp \cdot$.
The indicator function is denoted as $\mathds{1}[\cdot]$.
The cardinality of a set is denoted by $|\cdot|$.
The preimage of $\{x\}$ by $f$ will be denoted $\preimage{f}{x}$.
Finally, a hat $\hat{\cdot}$ will be used to refer to empirical estimates: 
\begin{inlinelist}
\item $\hrv{Y}$  is an approximation of $\rv Y$ (so $P_{\rv Y},P_{\hrv Y} \in \mathcal{P}(\mathcal{Y})$);
\item $\hat{P}_{\rv Y}$ denotes an empirical distribution of $\rv Y$;
\item Functionals with expectations taken over empirical distributions inherit the hat (e.g. $\op{\hat{H}}{\rv Y \cond \rv X}$).
\end{inlinelist}

Letters $\rv X$,$\rv Z$,$\rv Y$ are respectively used to refer to the input, representation and target of a predictive task.
We use $\rv X_y$,$\rv Z_y$ to respectively denote an input and a representations that have the same distribution as $\rv X$,$\rv Z$ conditioned on $y$, i.e., $P_{\rv X_y} = P_{\rv X|y}$ and $P_{\rv Z_y} = P_{\rv Z|y}$.
We denote by $\V$ any predictive family, i.e, $\V \in \mathcal{P}(\mathcal{Y} | \mathcal{Z})$ and satisfies the assumptions in \cref{appx:assumptions_V}.
The largest such set if the universal predictive family $\universal{} \defeq \mathcal{P}(\mathcal{Y} | \mathcal{Z})$.
The probability of $y \in \mathcal{Y}$ given $z \in \mathcal{Z}$ as predicted by a classifier $f \in \V$ is denoted $f[z](y)$ to distinguish it from the underlying conditional probability $P_{\rv Y | \rv Z}(y|z)$.
We are interested in minimizing the expected loss of a classifier $f \in \V$, also called risk 
$\mathrm{R}(f,\rv Z) \defeq
\E{y,x \sim P_{\rv y, \rv x}} {\E{z \sim P_{\rv Z \cond x}}{S(y,f[z])}}$.
In practice we will be given a training set of $M$ input-target pairs $\sampleD{}$, in which case we can estimate the risk using the empirical risk $\eRisk{} \defeq \frac{1}{M} \sum_{y,x \in \mathcal{D}} \E{z \sim P_{\rv Z \cond x}}{S(y,f[z])}$.
The set of ERMs are denoted as $\hat{\V}(\mathcal{D}) \defeq  \arg \min_{f \in \V}  \eRisk{}$.
Finally, we will denote the best achievable risk for $\V$ as $\mathrm{R}^*(\V) \defeq \min_{\rv Z} \min_{f \in \V} \Risk{}$.

\section{Assumptions}
\label{sec:appx_assumptions}

\subsection{Generic Assumptions}
\label{appx:generic_assumptions}
We make a some assumptions throughout our paper to have concise statements.
First let us discuss generic assumptions about the setting we are studying:

\begin{description}
\item[At least one example per class] We assume that every training set has at least one example per label.
This is generally true in modern ML, where $|\mathcal{D}| \gg |\mathcal{Y}|$.
\Cref{theo:opt_qmin} would not hold without it, as ERMs could not perform optimally without having examples to learn from.
\item[Logarithmic score]
We only consider the log loss $S(y,f[z]) \coloneqq -\log f[z](y)$ as it is the most common scoring rule.
Indeed, it is (essentially) the only \textit{strictly proper} (strictly minimized by the underlying distribution $P_{\rv Y \times \rv z }$) and \textit{local} (depending only on predicted probability of the observed event $P_{\rv Y \times \rv z }(y|\rv Z)$) scoring rule \cite{parmigiani}, making it computationally attractive.
The framework can likely be extended to any proper scoring rule (e.g. pseudo-likelihood, Brier score, kernel scoring rule) by considering generalized predictive entropy \cite{dawid1998coherent,grunwald2004game,gneiting2007strictly}.
\item[Finite sample spaces] 
We restrict ourselves to finite $|\mathcal{X}|$,$|\mathcal{Y}|$,$|\mathcal{Z}|$, so as to avoid the use of measure theory and axiomatic set theory, which would obscure the main points of the paper.
While this assumption holds in computational ML (due to the use of digital computer or the fact that we can always restrict the sample spaces to the finite examples seen in our training and testing set), it is unsatisfactory from a theoretical standpoint and the general case should be investigated in future work.
We conjecture that \cref{theo:opt_qmin} extends to the uncountable case.
\item[At least as many representations as labels] 
The sample space of representions is at least as large as the one for labels: $|\mathcal{Z}| \geq |\mathcal{Y}|$. This holds in practice where there are usually less than a 1000 possible labels $\mathcal{Y}$ while even a single dimensional $\rv Z$ can often (depending on computer) take $2^{32} \approx 4 * 10^6$ values.
\item[Multi-class classification] The sample space of the target is $\mathcal{Y}=[0,\dots,|\mathcal{Y}|-1]$.
\end{description}

\subsection{Assumptions on Functional Families}
\label{appx:assumptions_V}

Now let us discuss the assumptions that we make about functional families. 
The following assumptions hold for many functional families that are used in practice, including neural networks, logistic regression, and decision tree classifiers.

\begin{description}
\item[Invariance of $\V$ to label permutations] All predictive families are invariant to permutation, i.e., $\forall \V, \ \forall \pi : \mathcal{Y} \to \mathcal{Y}$,  $\forall f \in \mathcal{V}, \ \exists f' \in \mathcal{V}$ s.t. $\forall z \in \mathcal{Z}, \forall y \in \mathcal{Y}$ we have $\forall f[z](y) = f'[z](\pi(y))$.
This holds in practice (neural networks, decision trees, \dots) as we usually do not want predictors to depend on the order of labels, e.g. $\mathcal{Y}=\{``cat",``dog"\}$ or $\mathcal{Y}=\{``dog",``cat"\}$.
We use this assumption to simplify the proof of \cref{theo:opt_qmin}.
\item[Non-empty preimage of labels] We consider predictive families that have a non empty preimage for each label: $\forall \V$, $\exists f \in \V$, s.t. $\forall y \in \mathcal{Y}, \ \exists z \in \mathcal{Z}$ we have $f[z](y)=1$.
This is usually true in ML.
In neural networks, this can be achieved by making the weights of your last layer very large such that the softmax will give the label a probability of 1 (achieved due to floating point representation).
We use this assumption to show that when the label is deterministic, $\mathrm{R}^*(\V)=0$.
\item[Arbitrary constant prediction of $\V$] 
We assume that in all functional families there always is a predictor which predicts any constant output regardless of the input: $\forall \V, \ \forall P_{\rv Y} \in \mathcal{P}(\mathcal{Y}), \ \exists f \in \V$ s.t. $\forall z \in \mathcal{Z}$ we have $f[z]=P_{\rv Y}$.
This is typically true in classification, when the last layer parametrizes a categorical distribution. In neural networks this can be achieved by setting all weights to $0$ and then the bias of the last layer (softmax) to the desired values.
Notice that this not true in the general case (regression and countable infinite sample space), in which case the assumption can be relaxed to optional ignorance as in \cite{xu2020theory}.
We use this assumption to simplify the definition of $\V$-information in \cref{proposition:marginalent}.
\item[Monotonic biasing of $\V$]
We assume that all functional families are closed under ``monotonic biasing towards a prediction $y$''.
Formally, $\forall f' \in \mathcal{V}, \ \forall y \in \mathcal{Y}, \ \forall z \in \mathcal{Z}, \ \forall  p \in [0,1]$, $\exists g \in \V$ s.t. $g[z](y)=p$ and  $\forall z',z'' \in \mathcal{Z},\forall y' \in \mathcal{Y}$ we have 
$\mathrm{sign}(f'[z'](y') - f'[z''](y'))=\mathrm{sign}(g[z'](y') - g[z''](y'))$.
In other words, it is possible to construct a $g \in \V$ that assigns to a (single) pair $z,y$ the probability $p$ of your choice and preserves the order --- if $z,y$ was assigned a higher probability than $z',y$ by $f'$ then the same holds for $g$.
Such assumption holds for neural networks, as it is always possible to construct $g$ by modifying the bias term of the final softmax layer.
This assumption is crucial for the proof of \cref{theo:opt_qmin}.
\end{description}

\subsection{Assumptions for the Theorem}
\label{appx:assumptions_theorem}

We make an additional assumptions for \cref{theo:opt_qmin} and  \cref{proposition:Vminsuff}.

\begin{description}
\item[Deterministic Labeling] We assume that labels are deterministic functions of the data $\exists t : \mathcal{X} \to \mathcal{Y}$ s.t. $\rv Y = t(\rv X)$.
This is generally true in ML datasets where every example is only seen once and thus every example is given a single label with probability 1. 
This does not necessarily hold in the real world.
We use this assumption to simplify the proofs, we believe that it is not necessary for the theorem to hold but should be investigated in future work.
\end{description}

\section{Theoretical Results and Proofs}
\label{sec:appx_theory}

\subsection{Background}
\label{sec:proof_bckgrnd}
\subsubsection{Minimal Sufficient Statistics}
\label{sec:proof_minsuffstat}
In the following, we clarify the link between minimal sufficient \textit{statistics} \cite{fisher1922mathematical} and  \textit{representations} \cite{shamir2010learning,achille2018emergence} of inputs $\rv X$.
The difference between a representation $\rv Z$ in IB and a statistic $T(\rv X)$, is that the mapping between the inputs $\rv X$ and the representation $\rv Z$ can be stochastic --- specifically a representation is a statistic of the input and independent noise $\epsilon \perp \rv X$, i.e., $\rv Z = T(\rv X, \epsilon)$.
We now prove that for (deterministic) statistics, the notion of 
minimal sufficient representation is equivalent to that of predictive minimal sufficient statistics \cite{bernardo2009bayesian}. 
\begin{restatable}[Minimal Sufficient Representations]{definition}{minimality}\label{def:min_suff_rep}
A representation $\rv Z= T(\rv X, \epsilon)$ is:
\begin{itemize}
\item \textbf{Sufficient for $\rv Y$} if $\rv Z \in \suff{} \defeq \arg \max_{\rv z'} \MI{Y}{z'} $
\item \textbf{Minimal Sufficient for $\rv Y$} if $\rv Z \in \smin{} \coloneqq \arg\min_{\rv Z' \in \suff{}} \MI{X}{z'}$
\end{itemize}
\end{restatable}
\begin{definition}[Sufficient Statistic]\label{def:standard_sufficiency}
A statistic $\rv Z=T(\rv X)$ is predictive sufficient for $\rv Y $ if $\rv Y - \rv Z - \rv X$ forms a Markov Chain.
\end{definition}
%
%
\begin{lemma}[Equivalence of Sufficiency]\label{def:equiv_sufficiency}
Let $\rv Z$ be a statistic $T(\rv X)$ or a representation $T(\rv X,\epsilon)$  of $\rv X$, then
$\rv Z$ is predictive sufficient for $\rv Y$ by \cref{def:standard_sufficiency} $\iff$ $\rv Z$ is sufficient for $\rv Y$ by \cref{def:min_suff_rep}.
\end{lemma}
\begin{proof}$ $ We prove the following for statistics $\rv Z=T(\rv X)$ but the same proof holds for representations $T(\rv X,\epsilon)$.
For both directions we use the fact that for any statistics $\max_{\rv z'} \MI{Y}{z'}= \MI{Y}{X} $.
Indeed, $\rv Y - \rv X - \rv Z$ constitutes a Markov Chain as $\rv Z=T(\rv X)$.
From the data processing inequality (DPI)  we have $\MI{Y}{X} \geq \MI{Y}{Z}$, where the equality is achieved by using the identity statistic $\rv Z = \rv X$.

($\implies$) Suppose $\rv Z$ is sufficient by \cref{def:standard_sufficiency}. 
Since $\rv Z$ is a statistic we again have $\MI{Y}{X} \geq \MI{Y}{Z}$.
From \cref{def:standard_sufficiency}, we also have $\rv Y - \rv Z - \rv X $ which implies (DPI) $\MI{Y}{Z} \geq  \MI{Y}{X}$. 
Due to the upper and lower bound we must have $\MI{Y}{X} = \MI{Y}{Z}$, which is equivalent to \cref{def:min_suff_rep}.

($\impliedby$) Assume that $\rv Z$ is sufficient by \cref{def:min_suff_rep}. 
 Using the chain rule of information we have 
 \begin{align*}
\MI{Y}{Z,\rv X} &= \MI{Y}{Z ,\rv X}\\
\MI{Y}{Z}+\MI{Y}{X \cond \rv Z}&=\MI{Y}{X}+\MI{Y}{Z \cond \rv X}  \\
\max_{\rv z'} \MI{Y}{z'} + \MI{Y}{X \cond \rv z}&= \MI{Y}{X}+ \MI{Y}{z \cond \rv X} & \text{\cref{def:min_suff_rep}} \\
\MI{Y}{X \cond \rv z}&= \MI{Y}{z \cond \rv X}  \\
\MI{Y}{X \cond \rv z}&= 0 & \rv Y - \rv X - \rv Z \\
 \end{align*}
The fourth line comes from $\arg \max_{\rv z'} \MI{Y}{z'}=\MI{Y}{X}$.
The last line holds as $\rv Z$ is a statistic of $\rv X$. 
$\MI{Y}{X \cond \rv Z}= 0$ implies that $\rv  Y \perp \rv X \cond \rv Z$ so $\rv Y - \rv Z - \rv x$ is a Markov Chain, which concludes the proof.
\end{proof}
\begin{definition}[Minimal Sufficient Statistic]\label{def:standard_minimality}
A sufficient statistic $\rv Z$ is \textit{minimal} if for any other sufficient statistic $\rv Z'$ , there exists a function $g$ such that $\rv Z = g(\rv Z')$.
\end{definition}
%
%
\begin{proposition}[Equivalence of Minimal Sufficiency]
Let $\rv Z=T(\rv X)$ be a (deterministic) statistic, 
then $\rv Z$ is minimal (by \cref{def:standard_minimality}) and sufficient for $\rv Y$ (by \cref{def:standard_sufficiency})   $\iff$ $\rv Z$ is minimal sufficient by \cref{def:min_suff_rep}.
\end{proposition}
\begin{proof} From \cref{def:equiv_sufficiency} we know that the sufficiency requirements are equivalent in \cref{def:standard_minimality} and \cref{def:min_suff_rep}. 
We now need to prove that the minimality requirements are also equivalent.

 ($\implies$) Let $\rv Z$ be minimal by \cref{def:standard_minimality}, then for all other sufficient $\rv Z'$ we have  the Markov Chain $\rv X - \rv Z' - \rv Z $. 
 From the DPI, $\MI{X}{Z} \leq \MI{ X}{Z'}$. This completes the first direction of the proof.
 
 ($\impliedby$) We will prove it by contrapositive. 
 Suppose $\rv Z \defeq T(\rv X) \in \suff{}$ is not minimal by \cref{def:standard_minimality}, i.e. there exists a sufficient statistic $\rv Z' \defeq \rv T'(\rv X) \in \suff{}$ s.t. no function $g$ satisfies $T(\rv X) = g(T'(\rv X))$. 
Then the binary relation $\{(T'(x), T(x)) \cond x \in \mathcal{X}\}$ is not univalent, therefore the converse relation $\{(T(x),T'(x)) \cond x \in \mathcal{X}\}$ is not injective.
As a result, there exists a non injective function $\Tilde{g}$ such that
$T'(\rv X) = \Tilde{g}(T(\rv X))$. 
From the DPI we have $ \MI{X}{Z'} < \MI{ X}{Z}$ with a strict inequality due to the non injectivity of $\Tilde{g}$.
So $\rv Z$ is not minimal by \cref{def:min_suff_rep}, thus concluding the proof.
\end{proof}
We emphasize that the second implication ($\impliedby$) does \textit{not} hold in the case of a representation $\rv Z=T(\rv X,\epsilon)$.
Indeed, \cref{def:standard_minimality} is not really meaningful for ``stochastic'' representations.

\subsubsection{Replacing $\CHF{Y}{\varnothing}$ by $\op{H}{\rv Y}$ }
\label{sec:proof_emptyset}

Due to our ``arbitrary constant prediction of $\V$'' assumption, we can replace $\CHF{Y}{\varnothing}$ by $\op{H}{\rv Y}$ in \citepos{xu2020theory} definition of $\V$-information.

\begin{restatable}{proposition}{marginalent}\label{proposition:marginalent}
For all predictive families $\V$ we have $\CHF{Y}{\varnothing}=\op{H}{\rv y}$. 
\end{restatable}
\begin{proof}
Denote $\V_{\varnothing} \subset \V$ the subset of $f$ that satisfy $f[x] = f[\varnothing], \ \forall x \in \mathcal{X}$.
\begin{align*}
\CHF{Y}{\varnothing} &\defeq \inf_{f \in \V } \E{z,y \sim P_{\rv Z, \rv Y}}{- \log f[\varnothing](y) } \\
&=  \inf_{f \in \V_{\varnothing} } \E{z,y \sim P_{\rv Z, \rv Y}}{- \log f[\varnothing](y) } \\
&=  \inf_{f \in \V_{\varnothing} } \E{z,y \sim P_{\rv Z,\rv Y}}{- \log f[z](y) } \\
&=  \E{y \sim P_{\rv Y}}{- \log P_{\rv Y} } & \text{Properness and Arbitrary Const. Pred.} \\
&=  \op{H}{\rv y} \\
\end{align*}
The penultimate line uses the properness of the log loss (best unconditional predictor of y is $P_{\rv Y}$) and our assumption regarding ``arbitrary constant prediction'', which implies that there exists $f \in V$ s.t. $\forall z \in \mathcal{Z}$ we have $f[z]=P_{\rv Y}$.
\end{proof}

\subsection{$\V$-Sufficiency}
\label{sec:proof_sufficiency}

In this subsection, we prove our claims in \cref{sec:theory_qsuf}.
First, let us show that $\CHF{Y}{Z}$ is indeed the best achievable risk for $\rv Z$.

\begin{lemma}
\label{lemma:Vent_risk}
For any predictive family $\V$, $\min_{f \in \V} \Risk{}=\CHF{Y}{Z}$.
\end{lemma}
\begin{proof}
This directly come from the definition of predictive information:
\begin{align*}
\CHF{Y}{Z} &\defeq  \inf_{f \in \V} \E{z,y \sim P_{\rv Z, \rv Y}}{- \log f[z](y) } \\
&= \inf_{f \in \V} \E{y \sim P_{\rv y}}{\E{z \sim P_{\rv Z|  y}}{- \log f[z](y) }} &  \\
&= \inf_{f \in \V} \E{y \sim P_{\rv y}}{\E{x \sim P_{\rv X| y}}{\E{z \sim P_{\rv Z|  x}}{- \log f[z](y) }}} & \rv Y - \rv X - \rv Z \\
&=  \inf_{f \in \V} \E{y,x \sim P_{\rv y, \rv x}} {\E{z \sim P_{\rv Z \cond x}}{- \log f[z](y) }} \\
&=  \inf_{f \in \V}\Risk{} & \text{Def. Risk} \\
&= \min_{f \in \V} \Risk{} & \text{Finite Sample Space}
\end{align*}
\end{proof}

\Cref{proposition:opt_qsuf} is a trivial corollary of the previous lemma.

\optimalQsuf*
\begin{proof}
\begin{align*}
\Qsuff{} &\defeq \arg \max_{\rv Z} \DIF{Y}{Z} \\
&= \arg \max_{\rv Z} \op{H}{\rv Y} - \CHF{Y}{Z} \\
&= \arg \min_{\rv Z}  \CHF{Y}{Z}  & \text{Const. } \op{H}{\rv Y}\\
&= \arg \min_{\rv Z} \min_{f \in \V} \Risk{} & \text{\cref{lemma:Vent_risk}} 
\end{align*}
\end{proof}

Let us now show that when the label is deterministic $\forall V, \ \mathrm{R}^*(\V)=0$.
This may be counterintuitive, but the following proof shows that we are simply shifting the burden of classification from the classifier to the encoder --- which is unconstrained.

\begin{proposition}
\label{proposition:zero_risk}
Assume that labels are a deterministic function of the data $\exists t : \mathcal{X} \to \mathcal{Y}$ s.t. $\rv Y = t(\rv X)$, then
for any predictive family $\V$ the best achievable risk is $\min_{\rv Z} \min_{f \in \V} \Risk{}=0$.
\end{proposition}
\begin{proof}
First notice that $\min_{\rv Z} \min_{f \in \V} \Risk{} \geq 0$ due to the non-negativity of the log loss. 
We show that the inequality is an equality by constructing a representation $\rv Z^*$ and a $f \in V$ such that $\mathrm{R}(f,\rv Z^*)=0$. Intuitively, we do so by finding ``buckets'' of $\rv Z$ that correspond to a certain label and then having an encoder which essentially classifies each input $x$ to the correct bucket. Formally:

Let $\preimage{f}{y} \defeq \{ z \in \mathcal{Z} \ s.t. \ f[z](y)=1  \}$ denote the preimage of a deterministic label by a classifier $f$.
By the ``Non-empty Preimage of Labels'' assumption we know that $\forall \V$ there exists $ f \in \family{}$ s.t. $\forall y \in \mathcal{Y}$, the preimage is non-empty $|\preimage{f}{y}| \geq 0$.
Let $f$ be one of those predictors.
We construct the desired $\rv Z^*$ by setting its probability mass function $\forall z \in \mathcal{Z}, x \in \mathcal{X}$ as a uniform distribution over the $f$ preimage of the label of $x$ (deterministic label assumption $\rv Y=t(x)$) .
\begin{align}
P_{\rv Z^* \cond \rv X}(z|x) \defeq  \begin{cases}
  \frac{1}{|\preimage{f}{t(x)}|} & \text{if } 
 z \in \preimage{f}{t(x)}\\
0 & \text{else}
\end{cases} \label{eqn:z_star}
\end{align}
We now show that the risk $\mathrm{R}(f,\rv Z^*)$ is indeed 0:
\begin{align*}
\mathrm{R}(f,\rv Z^*) 
&\defeq \E{y,x \sim P_{\rv y, \rv x}} {\E{z \sim P_{\rv Z \cond x}}{- \log f[z](y) }} \\
&= \sum_{y \in \mathcal{Y}} \sum_{x \in \mathcal{X}} \sum_{z \in \mathcal{Z}} P_{\rv{Y}}(y) P_{\rv{X} \cond \rv{Y}}(x \cond y) P_{\rv{Z}^{*}|\rv{x}}(z \cond x)[-\log f[z](y)] \\
&= \sum_{y \in \mathcal{Y}} 
\sum_{x \in \mathcal{X}}  
\sum_{z \in \preimage{f}{t(x)}}
P_{\rv Y}(y)
P_{\rv X \cond \rv Y}(x|y)
\frac{- \log f[z](y)}{|\preimage{f}{t(x)}|} & \text{\cref{eqn:z_star}}
\\
&= 
\sum_{y \in \mathcal{Y}}  \br{
\sum_{z \in \preimage{f}{y}}
P_{\rv Y}(y)
\frac{- \log f[z](y)}{|\preimage{f}{y}|}}
\br{\sum_{x \in \mathcal{X}}P_{\rv X \cond \rv Y}(x|y)} 
\\
&= 
\sum_{y \in \mathcal{Y}}  \br{
\sum_{z \in \preimage{f}{y}}
P_{\rv Y}(y)
\frac{- \log 1}{|\preimage{f}{y}|}} * 1 & \text{Def. }f \\
 &= 0 
\end{align*}
The fourth line uses $y=t(x)$, thus removing the dependence with $\rv X$.
The penultimate line uses $\forall z \in \preimage{f}{y}$, $f[z](y)=1$  which is the defining property of the selected $f$.
\end{proof}

\subsection{Theorem}
\label{appx:proof_optimal}

We will now prove the main result of our paper, namely that \textit{any} ERM that uses a $\V$-minimal $\V$-sufficient representation will reach the best achievable test loss. 

\begin{theorem*}\label{theo:opt_qmin_formal} 
Suppose $\rv Y$ is a deterministic labeling $t(\rv X)$. 
Let $\V \in \mathcal{P}{(\mathcal{Y} | \mathcal{Z})}$ be a predictive family satisfying the assumptions in \cref{appx:assumptions_V}.
Under the assumptions stated in \cref{appx:generic_assumptions} , we have that:
if $\rv Z$ is a $\V$-minimal $\V$-sufficient representation of $\rv X$ for $\rv Y$, then any ERM on any dataset will achieve the best achievable risk, i.e.
\[
\rv Z \in \Qsmin{} \implies 
\forall M \geq |\mathcal{Y}|, \
\sampleD{}, \ \forall \hpred{} \in \hat{\V}(\mathcal{D}) \text { we have } \mathrm{R}(\hpred{},\rv Z)=\min_{\rv Z} \min_{f \in \V} \Risk{}
\]
\end{theorem*}

\subsubsection{Lemmas for \cref{theo:opt_qmin}}
\label{appx:proof_lemmas}

In this subsection we show three simple lemmas that are useful for proving \cref{theo:opt_qmin}.
First we show that in the deterministic label setting, 
$\V$-sufficiency implies that the representaion space can be partitioned by $\mathcal{Y}$, i.e., the supports of each $\rv Z_y$ are non-overlapping.

\begin{restatable}{lemma}{VsuffSupp}\label{lemma:VsuffSupp}
Assume $\rv Y$ is a deterministic labeling $t(\rv X)$. Then
$\rv Z \in \Qsuff{} \implies \forall y\neq y', \ y,y' \in \mathcal{Y}$ we have $\mathrm{supp}(P_{\rv Z | y}) \cap \mathrm{supp}(P_{\rv Z | y'}) = \varnothing$.
\end{restatable}
\begin{proof}
Let us prove it by contrapositive. 
Namely, we will show $\mathrm{supp}(P_{\rv Z | y}) \cap \mathrm{supp}(P_{\rv Z | y'}) \neq \varnothing \implies \rv Z \not \in \Qsuff{}$.
$\mathrm{supp}(P_{\rv Z | y}) \cap \mathrm{supp}(P_{\rv Z | y'}) \neq \varnothing$ implies that $\exists y \neq y' \in \mathcal{Y}, \exists z \in \mathcal{Z}$ s.t. $P_{\rv Z | \rv Y}(z|y')  \neq 0$ and $P_{\rv Z | \rv Y}(z|y) \neq 0$.
Using Bayes rule (and the fact that $P_{\rv Y}$ has support for all labels), that means $P_{\rv Y | \rv Z}(y'|z)   \neq 0$ and $P_{\rv Y | \rv Z}(y|z)  \neq 0$ so there exists no $ y \in \mathcal{Y}$ s.t. $P_{\rv Y | \rv Z}(y'|z) =  1$.
Due to the finite sample space assumption and monotonicity of $\V$ predictive entropy, this implies $0 < \op{H}{\rv Y \cond \rv Z}  = \CHU{Y}{Z}  \leq \CHF{Y}{Z} $. 
From \cref{proposition:zero_risk} we conclude that $\rv Z \not \in \Qsuff{}$ as desired.
\end{proof}

We now show the simple fact that, if some classifier achieves zero test loss then being an ERM is equivalent to achieving zero training loss.

\begin{restatable}{lemma}{existperfect}\label{lemma:existperfect}
Let $\sampleD{}$ be a training dataset.
Suppose $\exists f \in \V$  s.t. $\Risk{}=0$, then:
$\hpred{} \in \hat{\V}^*(\mathcal{D}) \iff 
 \hat{\mathcal{R}}(\hpred{}, \rv Z ; \mathcal{D})=0$.
\end{restatable}

\begin{proof}
As $\Risk{}$ is an expectation of a non-negative discrete r.v., it is equal to zero if and only if  $\forall z \in \mathrm{supp}(\rv Z), \ \forall y \in \mathrm{supp}(P_{\rv Y | \rv z})$  we have $\log f[z](y)=0$.
$\hat{\mathcal{R}}(f, \rv Z ; \mathcal{D})$ is also a weighted average (discrete expectation) of $\log f[z](y)=0$ over a subset of the previous support $\hat{z} \in \mathrm{supp}(\hrv Z) \subseteq \mathrm{supp}(\rv Z), \ \forall \hat{y} \in \mathrm{supp}(\hrv Y) \subseteq \mathrm{supp}(\rv Y)$ so we conclude that $\hat{\mathcal{R}}(f, \rv Z ; \mathcal{D})=0$.
As the minimal training loss is always zero and the risk cannot be less than zero (non negativity of log loss and finite sample space) the definition of ERMs becomes $ \hat{\V}(\mathcal{D}) \defeq  \arg \min_{f \in \V}  \eRisk{} = \{ f \in \V  \ s.t. \  \eRisk{}=0 \}$ as desired.
\end{proof}

A representation is $\V$-minimal $\V$-sufficient if and only if it has no $\V$-information with any of the terms in any $y$ decomposition of $\rv X$.

\begin{restatable}{lemma}{zeroVminimal}\label{lemma:zeroVminimal}
Assume $\rv Y$ is a deterministic labeling $t(\rv X)$, then:

$\forall \rv Z \in \Qsmin{} \quad \iff \quad \rv Z \in \Qsuff{}$ and $\forall y \in \mathcal{Y}, \forall \rv N \in \decxy{}$ we have $\DIF{N}{Z_y}=0$. 
\end{restatable}

\begin{proof} $ $ \\
($\impliedby$) Due to the non negativity of $\V$-information, we have $\forall \rv Z' \in \Qsuff{}$,  $0=\DIF{N}{Z_y} \leq \DIF{N}{Z'_y} $ so $\rv Z$ reaches the minimal achievable value in each term and thus also on their expectation $\DIFxzCy{}=0$. We thus conclude that $\rv Z \in \Qsmin{} \defeq \arg \min_{\rv Z \in \Qsuff{}} \DIFxzCy{}$.

($\implies$) 
Let us show that there is at least one $\rv Z \in \Qsmin{}$ s.t. $\forall y \in \mathcal{Y}, \forall \rv N \in \decxy{}, \ \DIF{N}{Z_y}=0 $, from which we will conclude that they all have to satisfy the previous property
in order to minimize $\DIFxzCy{}$.
Let us consider $\rv Z^*$ as defined in \cref{eqn:z_star}.
Notice that $P_{\rv Z_y^* \cond \rv X_y}(z|x) = \frac{1}{|\preimage{f}{t(x)}|} = \frac{1}{|\preimage{f}{y}|} = P_{\rv Z_y^*}(z_y)$, where the last equality comes from the fact that $z_y$ is associated with a single label $y$.
An other way of saying it, is that $\rv X_y - y - \rv Z^*_y$ forms a Markov Chain, but $y$ is a constant.
We thus conclude $\forall y \in \mathcal{Y}$ $\rv Z_y^* \perp \rv X_y$.
By definition of $y$ decomposition of $\rv X$ (\cref{eqn:dec}), we also know that $\forall \rv N \in \decxy{}$, $\exists t' : \mathcal{X} \to \mathcal{Y}$ s.t. $\rv N = t'( \rv X_y)$, from which we conclude that $\rv Z_y^* \perp \rv N$.
Due to the independence property of $\V$-information, we have $\forall y \in \mathcal{Y}$, $\forall \rv N \in \decxy{}$, $\DIF{N}{Z_y}=0$ as desired.
As we found one $\rv Z^* \in \Qsmin{}$ s.t. this is true, it must be for all $\rv Z \in \Qsmin{}$.
Indeed, due to the  positivity property it is the only way of reaching the minimal $\DIFxzCy{}=0$.
\end{proof}

\subsubsection{Proof Intuition}

\begin{figure}
\centering
\includegraphics[width=\textwidth]{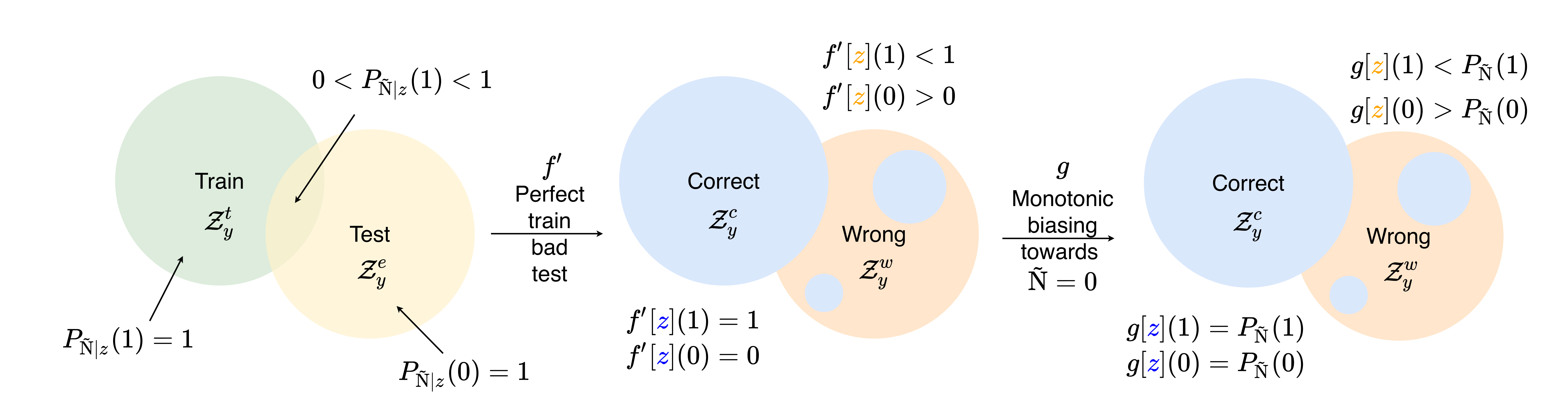}
\caption{
Intuition behind the construction of $g \in \V$ in the proof of \cref{theo:opt_qmin}.
The plot schematically represents the all the representations $\mathcal{Z}_y$ associated with a label $y=1$.
The representations 
$\mathcal{Z}_y^t$ (green) are associated with training examples,  $\mathcal{Z}_y^e$ (yellow) are associated with test / evaluation examples, $\mathcal{Z}_y^c$ (blue) are those that yield correct predictions of $y=1$ by $f'$, $\mathcal{Z}_y^w$ (orange) are those that yield wrong predictions  of $y=1$ by $f'$. 
}
\label{fig:schema_proof}
\end{figure}

The main difficulty in the proof is that $\V$-minimality removes information using deterministic labeling while the predictors $f \in \V$ are probabilistic.\footnote{
If both the labeling and predictors had been deterministic, the proof would be very simple and would go as follows:
assume $f$ is not optimal on test performance, show that the labels $\hat{\rv Y}$ predicted by $f$ and are in $\decxy{}$, conclude that $\rv Z$ does not minimize $\decxy{}$ as it perfectly predicts $\hat{\rv Y}$ by construction.
}
As a result the proof is relatively long, here is a rough outline:

\begin{enumerate}
\item As the theorem is about \textit{all} ERMs, use a proof by contrapositive to only talk about a single ERM $f'$ that performs optimally on train but not on test.
\item Construct a random variable $\tilde{\rv N}$ which labels train examples as 1 and test as 0. Show that $ \tilde{\rv N} \in \decxy{}$.
\item Using the ``monotonic biasing of $\V$'' assumption, construct $g\in \V$ from $f' \in \V$ by monotonically biasing the predictions towards the ``test'' label $\tilde{\rv N}=0$ s.t. $g[z](0)=P_{\tilde{\rv N}}(0)$ for every representation ${\color{blue} z^{(c)}}$ that are perfectly labelled by $f'$ (as shown in blue in \cref{fig:schema_proof}).
\item Show that $g$ predicts $\tilde{\rv N}$ better than the marginal distribution $P_{\tilde{\rv N}}$ for representations ${\color{burntorange} z^{(w)}}$ that are not perfectly labelled by $f'$ (as shown in orange in \cref{fig:schema_proof}), while predicting as well as $P_{\tilde{\rv N}}$ for ${\color{blue} z^{(c)}}$. Conclude that $g$ predicts $\tilde{\rv N}$ better than $P_{\tilde{\rv N}}$.
\item Show that the previous point entails $\op{I_{\V}}{\rv Z \to \tilde{\rv N}} \neq 0$. Conclude by \cref{lemma:zeroVminimal} that $\rv Z \not \in \Qsmin{}$ as desired. 
\end{enumerate}

\subsubsection{Formal Proof}

\begin{proof}
If $\rv Z \in \Qsmin{}$ is $\V$-minimal $\V$-sufficient then by definition it is also $\V$-sufficient, we thus restrict our discussion to $\V$-sufficient representations.
As $\rv Z$ is $\V$-sufficient, $\exists f \in V$ s.t. $\mathrm{R}(f,\rv Z)=\min_{\rv Z} \min_{f \in \V} \Risk{}=0$.
The first equality comes from \cref{proposition:opt_qsuf}.
The second equality comes from \cref{proposition:zero_risk} and the deterministic assumption $\rv Y=t(\rv X)$.
As there is some function $f$ with risk $\mathrm{R}(f,\rv Z)=0$, \cref{lemma:existperfect} tells us that being an ERM is equivalent to having zero empirical risk $\hpred{} \in \hat{\V}(\mathcal{D}) \iff \hat{\mathrm{R}}(\hpred{},\rv Z; \mathcal{D})=0$.
We thus only need to prove that $\V$-minimality implies that every function with zero empirical risk will get zero actual risk, i.e., they all generalize:
$\rv Z \in \Qsmin{} \implies \forall
\sampleD{}, \ \forall f \in \V$ s.t. $\eRisk{}=0$ will achieve $\mathrm{R}(f,\rv Z)=0$.

We will prove this statement by contrapositive, namely that the existence of a predictor with $0$ empirical risk but larger actual risk implies that the representation is not $\V$-minimal: $\exists \sampleD{} , \exists f'$ s.t. $\hat{\mathrm{R}}(f',\rv Z ; \mathcal{D})=0 \land \mathrm{R}(f',\rv Z) > 0 \implies$   $\rv Z  \not \in \Qsmin{}$.
For ease of notation let us assume that we are in a binary classification setting $|\mathcal{Y}|=2$.
We will later show how to reduce the multi-classification setting to the binary one.

Let $f'$ be a function with $\hat{\mathrm{R}}(f',\rv Z ; \mathcal{D})=0 \land \mathrm{R}(f',\rv Z) > 0$.
As the risk is positive, there must be some label $\exists y \in \mathcal{Y}$ s.t. when predicting from examples labeled as $y$ the expected loss will be positive $\mathrm{R}_y(f',\rv Z) \defeq \E{x \sim P_{\rv x \cond \rv y}}{\E{z \sim P_{\rv z \cond \rv x}}{- \log f'[z](y)}} > 0$.
Due to the deterministic labeling assumption and \cref{lemma:VsuffSupp}, we can study one single such label $y$ without considering any other label $y' \in \mathcal{Y}$ (neither the examples nor the representations interact between labels).
Without loss of generality --- due to the invariance of $\V$ to label permutations --- let us assume that this label is $y=1$.

Let $\mathcal{X}^{t}_y$ be the set of examples (associated with $y$) seen during training and  $\mathcal{X}^{e}_y = \mathcal{X}_y \setminus \mathcal{X}^{t}_y$ those only during test (eval). 
Let $\mathcal{Z}^{w}_y \defeq \{ z_y^{(w)} \in \mathcal{Z}_y | f'[z_y](y) \neq 1 \}$ be the representations who are (wrongly) not predicted $y$ by $f'$ with a probability of 1. 
Let $\mathcal{Z}^{c}_y \defeq \mathcal{Z}_y \setminus \mathcal{Z}^{w}_y $ be the set of representations that are (correctly) labeled $y$ by $f'$ with a probability of 1.
Notice that both these sets are non empty $|\mathcal{Z}^{w}_y|>0$ and $|\mathcal{Z}^{c}_y|>0$ because respectively $\mathrm{R}_y(f',\rv Z) > 0$ and $\hat{\mathrm{R}}_y(f',\rv Z,\mathcal{D}) = 0$.

Let $\tilde{\rv N} \defeq  \mathds{1}[\rv X_y \in \mathcal{X}^t_y ]$ be a binary ``selector'' of training examples.
Notice that $\tilde{\rv N} \in \decxy{}$, as $\mathds{1}[\cdot \in \mathcal{X}^t_y]$ is a deterministic function from $\mathcal{X}_y \to \mathcal{Y}$.
We want to show that $\op{H_{\mathcal{V}}}{\tilde{\rv N} \cond \rv Z_y} < \op{H}{\tilde{\rv N}}$.
To do so we have to find a function $g$ whose risk when predicting $\tilde{\rv N}$ is smaller than the entropy $\op{H}{\tilde{\rv N}}$.
Notice that $f'$ is close to being the desirable $g$, but not quite.
\footnote{Indeed, by construction $f'$ predicts perfectly the training examples $\tilde{\rv N}=1$, because $\hat{\mathrm{R}}_y(f',\rv Z,\mathcal{D})=0$ for $y=1$.
Unfortunately, its risk when predicting $\tilde{\rv N}$ is not always smaller than  $\op{H}{\tilde{\rv N}}$.
For example, $f'$ would incur infinite loss if some test examples $x \in \mathcal{X}^e_y$ was encoded to some some $\mathcal{Z}^{c}_y$ because $f'$ would predict that it comes from the training set with probability of 1.
}

We construct the desired $g$ by starting with $f'$ and monotonically increasing the probability of predicting $\tilde{\rv N}=0$ s.t. {\color{blue}$\forall z^{(c)} \in \mathcal{Z}^{c}_y$ we have  $g[z^{(c)}]=P_{\tilde{\rv N}}$}
as seen in \cref{fig:schema_proof}.
Specifically, from the ``monotonic biasing of $\V$'' we know that for $f' \in \V$, $y=0$, some $z^{(c)} \in \mathcal{Z}^{c}_y$, and $p=P_{\tilde{\rv N}}$ there exists $g \in \V$ s.t. $g[z^{(c)}](0)= P_{\tilde{\rv N}}(0)$ and $\forall z',z'' \in \mathcal{Z},\forall y' \in \mathcal{Y}$ we have 
$\mathrm{sign}(f'[z'](y') - f'[z''](y'))=\mathrm{sign}(g[z'](y') - g[z''](y'))$.
Notice that $\forall \tilde{z}^{(c)} \in \mathcal{Z}^{c}_y$ we have $g[z](0)= P_{\tilde{\rv N}}(0)$ due to the ordering requirement of monotonic biasing.
Indeed, we have just shown that this is true for one such $z^{c}$ and by construction $\forall \tilde{z}^{(c)} \in \mathcal{Z}^{c}_y$ we have  $\mathrm{sign}(f'[z^{(c)}](0) - f'[\tilde{z}^{(c)}](0)) = \mathrm{sign}(0 - 0) = 0$ so $0 = \mathrm{sign}(g[z^{(c)}](0) - g[\tilde{z}^{(c)}](0)) = \mathrm{sign}(P_{\tilde{\rv N}}(0) - g[\tilde{z}^{(c)}](0)) $ from which we conclude that $g[\tilde{z}^{(c)}](0)=P_{\tilde{\rv N}}(0)$ and $g[\tilde{z}^{(c)}](1)=P_{\tilde{\rv N}}(1)$ as we are in a binary setting.

Due to ordering requirement of monotonic biasing  {\color{burntorange}$\forall z^{(w)} \in \mathcal{Z}^{w}_y$ we have  $g[z^{(w)}](0) > P_{\tilde{\rv N}}(0)$} as seen in \cref{fig:schema_proof}.
Indeed, $\forall z^{(w)} \in \mathcal{Z}^{w}_y$ we have $f'[z^{(w)}](0) > 0$ by construction ($f'[z^{(w)}](1) \neq 1$) so by the ordering requirement $\mathrm{sign}(f'[z^{(c)}](0) - f'[z^{(w)}](0)) = \mathrm{sign}( - f'[z^{(w)}](0)) = -1 = \mathrm{sign}(g[z^{(c)}](0) - g[z^{(w)}](0)) = \mathrm{sign}(P_{\tilde{\rv N}}(0) - g[z^{(w)}](0)) $ from which we conclude that $g[z^{(w)}](0) > P_{\tilde{\rv N}}(0)$.

In other words, $g$ predicts training examples $\tilde{\rv N}=0$ in the {\color{blue}same way} as $P_{\tilde{\rv N}}$ but testing examples {\color{burntorange}better} than $P_{\tilde{\rv N}}$.
From here it should be clear that $\op{H_{\mathcal{V}}}{\tilde{\rv N} \cond \rv Z_y} < \op{H}{\tilde{\rv N}}$, which we prove below for completeness.
\begin{align*}
\op{I}{\rv Z_y \rightarrow \tilde{\rv N} } &\defeq  \op{H}{\tilde{\rv N}} - \op{H_{\mathcal{V}}}{\tilde{\rv N} \cond \rv Z_y} \\
 &= \sup_{f \in \V} \E{x,n,z \sim P_{\rv X_y, \tilde{\rv N} , \rv Z_y}}{\log \frac{f[z](n)}{P_{\tilde{\rv N}}(n)}}
 \\
 &\geq \E{x,n,z \sim P_{\rv X_y, \tilde{\rv N} , \rv Z_y}}{\log \frac{g[z](n)}{P_{\tilde{\rv N}}(n)}} & \text{Use } g
 \\
 &= \E{x,n \sim P_{\rv X_y, \tilde{\rv N}}}{\sum_{z \in \mathcal{Z}_y^{c}}  P_{\rv Z_y \cond \rv X_y}(z|x) \log \frac{g[z](n)}{P_{\tilde{\rv N}}(n)}} 
 \\
 &+
 \E{x,n \sim P_{\rv X_y, \tilde{\rv N}}}{\sum_{z \in \mathcal{Z}_y^{w}}  P_{\rv Z_y \cond \rv X_y}(z|x) \log \frac{g[z](n)}{P_{\tilde{\rv N}}(n)}}  & \mathcal{Z}^{c}_y \defeq \mathcal{Z}_y \setminus \mathcal{Z}^{w}_y 
 \\
&= 0
 +
 \E{x,n \sim P_{\rv X_y, \tilde{\rv N}}}{\sum_{z \in \mathcal{Z}_y^{w}}   P_{\rv Z_y \cond \rv X_y}(z|x)  \log \frac{g[z](n)}{P_{\tilde{\rv N}}(n)}} & {\color{blue} \forall z \in \mathcal{Z}^{c}_y, g[z]=P_{\tilde{\rv N}}(n)}
\\
&= \sum_{x \in \mathcal{X}_y^t} \sum_{z \in \mathcal{Z}_y^{w}}   P_{\rv Z_y, \rv X_y}(z, x)  \log \frac{g[z](1)}{P_{\tilde{\rv N}}(1)} \\
&+ \sum_{x \in \mathcal{X}_y^e} \sum_{z \in \mathcal{Z}_y^{w}}   P_{\rv Z_y, \rv X_y}(z, x)  \log \frac{g(z)[0]}{P_{\tilde{\rv N}}(0)} & \text{Def. } \tilde{\rv N} \text{ and } \mathcal{X}^{e}_y = \mathcal{X}_y \setminus \mathcal{X}^{t}_y
 \\
 &=0 +  \sum_{x \in \mathcal{X}_y^e} \sum_{z \in \mathcal{Z}_y^{w}}   P_{\rv Z_y, \rv X_y}(z, x)  \log \frac{g(z)[0]}{P_{\tilde{\rv N}}(0)} & \hat{\mathrm{R}}(f',\rv Z ; \mathcal{D})=0 \\
 &> 0 & {\color{burntorange} \forall z \in \mathcal{Z}^{w}_y, g[z](0) > P_{\tilde{\rv N}}(0)}
\end{align*}

Where the penultimate line comes form the fact that all train example $x\in \mathcal{X}^{t}_y$ must get encoded to some $z \in \mathcal{Z}^{c}_y$ as the empirical risk of $f'$ is 0.
We conclude that $\op{I}{\rv Z_y \rightarrow \tilde{\rv N} } \neq 0$, so by \cref{lemma:zeroVminimal}, $\rv Z \not \in \Qsmin{}$ which concludes the proof for the binary case.

For the multi-classification setting, the same proof holds by taking $f'$ and effectively reducing it to a binary classifier.
This is possible by starting from $f'$ and monotonically biasing it to construct $f'_{binary}$ which predicts with zero probability for all but two labels $y,y' \in \mathcal{Y}$.
One of those labels (say $y$) has to be the correct label (to ensure that $f'_{binary}$ still reaches 0 empirical risk), while the other $y'$ can be any label s.t. $\exists z \in \mathcal{Z}_y$ with $f'[z](y') \neq 0$.
Such a $y'$ always exists as the risk of $f'$ is not $0$.
Due to the monotonic biasing, $\forall z \in \mathcal{Z}_y$ we have $f'_{binary}[z](y) \geq f'[z](y)$ and $f'_{binary}[z](y') \geq f'[z](y')$.
As a result $f'_{binary}$ still reaches 0 empirical risk but non zero actual risk, it can thus be used to construct the desired $g$ as before.

\end{proof}

\subsection{$\V$-minimal $\V$-sufficient Properties}
\label{appx:proof_prop}

In this section we prove \cref{proposition:Vminsuff}.
We first show that universal sufficient representation corresponds to the subset of sufficient representations that t.v.i. in the domain of predictors in $\universal{}$.

\begin{restatable}[Recoverability of sufficiency]{lemma}{recsuff}\label{proposition:recsuff}
 Let $\universal{}$ be the universal family, then  $\Usuff{} = \suff{} \cap \mathcal{Z}$
\end{restatable}
\begin{proof}
In the following we abuse notation by using $\rv Z \in \mathcal{Z}$
to denote that $\rv Z$ t.v.i. $\mathcal{Z}$. Let us denote $\Omega = \bigcup \mathcal{Z}$ as the set of all possible finite sample spaces.
 From the recoverability property of $\V$-information, we know that  $\DI{\universal{}}{Y}{Z} = \MI{Y}{Z}$ so $\arg \max_{\rv Z \in  \mathcal{Z}} \DI{\universal{}}{Y}{Z} = \arg \max_{\rv Z \in \mathcal{Z}} \MI{Y}{Z}$. 
Suppose that $\suff{} \cap \mathcal{Z}$ is non empty, then $\suff{} \cap \mathcal{Z} = ( \arg  \max_{\rv Z \in \Omega} \MI{Y}{Z} ) \cap \mathcal{Z}  =  \arg \max_{\rv Z \in \mathcal{Z}} \MI{Y}{Z} = \Usuff{}$ as desired.
To show that $\suff{} \cap \mathcal{Z}$ is indeed non empty, notice that one can effectively learn the sufficient representation $\rv Z =\rv Y $ due to the assumption $|\mathcal{Y}| \leq |\mathcal{Z}|$ (when $|\mathcal{Y}| < |\mathcal{Z}|$, restrict the support to $\mathcal{Y}$).
\end{proof}

Let us characterize the set of minimal sufficient representations in terms of independence.

\begin{restatable}[Characterization of Minimal Sufficient Representations]{lemma}{charminsuff}\label{proposition:charminsuff}
Suppose $\rv Y$ is a deterministic labeling $t(\rv X)$, then the set of minimal sufficient representations $\smin{}$ correspond to  $\{ \rv Z \in \suff{} | \ \ \rv Z \perp \rv X | \rv Y  \ \}$.
\end{restatable}
\begin{proof}
\begin{align*}
\smin{} &\defeq \arg \min_{\rv Z \in \suff{}} \MI{X}{Z} \\
&= \arg \min_{\rv Z \in \suff{}}  \MI{X}{Z} + \op{I}{\rv Z; \rv Y | \rv X}& \rv Y - \rv X - \rv Z \\
&= \arg \min_{\rv Z \in \suff{}} \MI{X,\rv Y}{Z}& \text{Chain Rule} \\
&= \arg \min_{\rv Z \in \suff{}} \op{I}{\rv X; \rv Z | \rv Y} + \MI{Z}{Y}   & \text{Chain Rule} \\
&= \arg \min_{\rv Z \in \suff{}} \op{I}{\rv X; \rv Z | \rv Y}   & \text{Sufficiency} 
\end{align*}
The last line uses the fact that $\MI{Z}{Y}$ is a constant as the optimization is contrained to sufficient $\rv Z \in \suff{}$.
Notice that $\op{I}{\rv X; \rv Z | \rv Y} \geq 0$ and we know that it can reach zero when the labels are deterministic, for example with $\rv Z = \rv Y$.
So $\smin{} = \{ \rv Z \in \suff{} | \op{I}{\rv Y; \rv Z | \rv Y} = 0\} = \{ \rv Z \in \suff{} | \ \  \rv Z \perp \rv X | \rv Y  \ \}$ which concludes the proof.
\end{proof}

Similarly, let us characterize $\universal{}$-minimal $\universal{}$-sufficient representations in terms of independence.

\begin{restatable}[Characterization of $\universal{}$-Minimal $\universal{}$-Sufficient Representations]{lemma}{charUminUsuff}\label{proposition:charUminUsuff}
Suppose $\rv Y$ is a deterministic labeling $t(\rv X)$, then $\universal{}$-minimal $\universal{}$-sufficient representations $\Usmin{}$ correspond to  $\{ \rv Z \in \Usuff{} | \ \forall y \in \mathcal{Y},  \forall \rv N \in \mathrm{Dec}(\rv X, y), \ \rv Z \perp \rv N  \  \ \}$
\end{restatable}
\begin{proof}
Because of \cref{lemma:zeroVminimal}, $\forall y \in \mathcal{Y}, \forall \rv N \in \decxy{}$ we have $\DIF{N}{Z_y}=0$.
As $\universal{}$-information recovers MI that means $\rv N \perp \rv Z_y$. So $\Usmin{} = \{ \rv Z \in \Usuff{} | \ \forall y \in \mathcal{Y},  \forall \rv N \in \mathrm{Dec}(\rv X, y), \ \rv Z \perp \rv N  \  \ \}$ as desired.
\end{proof}

\begin{restatable}[Monotonicity of $\V$-information]{lemma}{monotonicity}\label{lemma:monotonicity}
Let $\V \subseteq \V^+$
be two predictive families, then $\forall \rv Z$ t.v.i $\mathcal{Z}$ and $\forall \rv Y$ t.v.i. $\mathcal{Y}$ we have $\op{I_{\V^+}}{\rv Z \rightarrow \rv Y} \geq \DIF{Y}{Z}$.
\end{restatable}

\begin{proof}
Let us start from the monotonicity of the predictive entropy \cite{xu2020theory}, which comes directly from the fact that we are optimizing over a larger functional family:
\begin{align*}
\op{H_{\V}}{\rv Y | \rv Z} &\geq \op{H_{\V^+}}{\rv Y | \rv Z} & \text{Monotonicity $\V$-ent.} \\
\op{H}{\rv Y } - \op{H_{\V}}{\rv Y | \rv Z} &\leq \op{H}{\rv Y } - \op{H_{\V^+}}{\rv Y | \rv Z} \\
 \DIF{Y}{Z} 
&\leq  \op{I_{\V^+}}{\rv Z \rightarrow \rv Y} & \text{\cref{proposition:marginalent}}
\end{align*}
\end{proof}

\Vminsuff*

\begin{proof}$ $
\begin{description}
\item[Existence]
In  \cref{eqn:z_star} we show how to construct such a $\rv Z \in \Qsmin{}$. 
\item[Characterization]
As $\DIFxzCy{}$ is an average over a non negative (positivity of $\V$-information) $\DIF{N}{Z_y}$ it is equal to zero if and only if all the terms are zero: $\DIFxzCy{} = 0  \iff \forall y \in \mathcal{Y}, \ \forall \rv N \in \decxy{}$ we have $\DIF{N}{Z_y}=0$.
We conclude the proof using \cref{lemma:zeroVminimal}.
\item[Monotonicty] Following the same steps as the proof in  ($\implies$) of \cref{lemma:zeroVminimal}, we get that $\V^+$-minimal $\V$-sufficient representation implies $\forall y \in \mathcal{Y}, \forall \rv N \in \decxy{}$ we have $\op{I_{\V^+}}{\rv Z_y \rightarrow \rv N}=0$.
Using the monotonocity of $\V$-information in our setting (\cref{lemma:monotonicity}) we have $0 = \DIF{N}{Z_y} \leq   \op{I_{\V^+}}{\rv Z_y \rightarrow \rv N} $.
As $\V$-information is always positive, we conclude that $\forall y \in \mathcal{Y}, \forall \rv N \in \decxy{}$,  $\DIF{N}{Z_y}=0$.
By \cref{lemma:zeroVminimal}, we conclude that $\rv Z$ is $\V$-minimal $\V$-sufficient as desired.
\item[Recoverability]
Using \cref{proposition:recsuff} we know that $\Usuff{}=\suff{} \cap \mathcal{Z}$ so the domain of optimization for minimality and $\universal{}$-minimality is the same.
Using \cref{proposition:charminsuff} and \cref{proposition:charUminUsuff} we only need to show that  $\forall \rv Z \in \Usuff{}$ we have $\rv Z \perp \rv X | \rv Y \iff  \forall y \in \mathcal{Y}, \ \forall \rv N \in \mathrm{Dec}(\rv X, y)$ we have $\rv N \perp \rv Z_y$.
As a reminder, $\rv Z_y$ and $\rv X_y$ have distribution $P_{\rv X_y} = P_{\rv X | y}$, $P_{\rv Z_y} = P_{\rv Z | y}$ and $P_{\rv X_y, \rv Z_y}=P_{\rv X, \rv Z |y}$ by definition. 

($\implies$) Starting from minimality $\rv Z \perp \rv X | \rv Y $ so $\forall y \in \mathcal{Y}$ we have $P_{\rv X_y, \rv Z_y}=P_{\rv X, \rv Z |y}=P_{\rv X | y}*P_{\rv Z |y}=P_{\rv X_y}*P_{\rv Z_y}$ from which we conclude that $\rv Z_y \perp \rv X_y$.
As $\rv N = t'( \rv X_y)$, for all $\rv N \in \decxy{}$ we have $\rv Z_y \perp \rv N$  as desired.\\

($\impliedby$) Let us prove it by contrapositive. I.e. we show that $\rv Z \not \perp \rv X | \rv Y \implies \exists y \in \mathcal{Y}, \exists \rv N \in \mathrm{Dec}(\rv X, y) \ s.t. \  \rv N \not\perp \rv Z_y $.
As $\rv Z \not \perp \rv X | \rv Y$ then  $\exists \tilde{x} \in \mathcal{X}, \tilde{y} \in \mathcal{Y}$ s.t. $P_{\rv Z |\rv X, \rv Y } (\cdot \cond \tilde{x},\tilde{y})  \neq P_{\rv Z| \rv Y}(\cdot \cond \tilde{y})$.
Let us define  $ \tilde{\rv N} = \mathds{1}[\rv X_{\tilde{y}}  = \tilde{x}_{\tilde{y}}] $,
as $\mathds{1}[\rv X_{\tilde{y}}  = \cdot ]$ is a deterministic function from $\mathcal{X} \to \mathcal{Y}$, we have $\tilde{\rv N} \in \mathrm{Dec}(\rv X, \tilde{y})$.
Then:
\begin{align*}
P_{\rv Z_{\tilde{y}}| \tilde{\rv N}}(\cdot \cond 1) 
&= P_{\rv Z_{\tilde{y}}| \rv X_{\tilde{y}}}(\cdot \cond \tilde{x}_{\tilde{y}})  & \text{Def. }  \tilde{\rv N} \\
&= P_{\rv Z |\rv X, \rv Y } (\cdot \cond \tilde{x},\tilde{y})  & \text{Def. } \rv Z_{\tilde{y}} \\
&\neq P_{\rv Z| \rv Y}(\cdot \cond \tilde{y}) & \text{Def. }  \tilde{x},\tilde{y} \\
&= P_{\rv Z_{\tilde{y}}} & \text{Def. }  \rv Z_{\tilde{y}}
\end{align*}
As $P_{\rv Z_{\tilde{y}} | \tilde{\rv N}}(\cdot \cond 1)  \neq P_{\rv Z_{\tilde{y}}}$ we conclude that $\exists \tilde{y} \in \mathcal{Y}, \exists \rv N \in \mathrm{Dec}(\rv X, \tilde{y}) \ s.t. \  \rv N \not\perp \rv Z_{\tilde{y}} $ as desired.
\end{description}
\end{proof}

\begin{corollary}\label{corollary:optimalIB}
Suppose $\rv Y$ is a deterministic labeling $t(\rv X)$. 
Let $\universal{}$ be the universal predictive family.
Under the assumptions stated in \cref{appx:generic_assumptions}, we have that:
if $\rv Z$ is a minimal sufficient representation of $\rv X$ for $\rv Y$ that t.v.i. $\mathcal{Z}$, then any ERM on any dataset will achieve zero risk, i.e.
\[
\rv Z \in \smin{} \cap \mathcal{Z} \implies 
\forall M \geq |\mathcal{Y}|, \
\sampleD{}, \ \forall \hpred{} \in \hat{\universal{}}(\mathcal{D}) \text { we have } \mathrm{R}(\hpred{},\rv Z)=0
\]
\end{corollary}
\begin{proof}
First notice that $\universal{}$ is unconstrained and thus satisfies the assumptions in \cref{appx:assumptions_V}.
As a result, \cref{theo:opt_qmin} tells us that $\rv Z \in \mathcal{M}_{\mathcal{U}} \implies 
\forall M \geq |\mathcal{Y}|, \
\sampleD{}, \ \forall \hpred{} \in \hat{\universal{}}(\mathcal{D}) \text { we have } \mathrm{R}(\hpred{},\rv Z)=\min_{\rv Z} \min_{f \in \universal{}} \Risk{}$.
Because the labeling is deterministic we have by \cref{proposition:zero_risk} that the best achievable risk is $\min_{\rv Z} \min_{f \in \universal{}} \Risk{}=0$.
From the recoverability property of $\V$-minimality and $\V$-sufficiency (\cref{proposition:Vminsuff}) we have $\smin{} \cap \mathcal{Z} = \mathcal{M}_{\universal{}}$ so $\rv Z \in \mathcal{M}_{\mathcal{U}} \iff \rv Z \in \mathcal{M}_{\mathcal{U}} \implies 
\forall M \geq |\mathcal{Y}|, \
\sampleD{}, \ \forall \hpred{} \in \hat{\universal{}}(\mathcal{D}) \text { we have } \mathrm{R}(\hpred{},\rv Z)=\min_{\rv Z} \min_{f \in \universal{}} \Risk{} = 0$ as desired.
\end{proof}


\subsection{Estimation Bounds}
\label{appx:practical_estimation}

In this section we will prove and formalize \cref{proposition:pac}, namely that $\edib$ estimates $\dib{}$ with PAC-style guarantees.
First, let us formalize $\dib{}$ and $\edib$ described respectively in \cref{eqn:dib} and \cref{alg:pseudo_dib}.
For simplicity, in the following we will use $\tdecxy{} \defeq \{ t \cond \rv N = t(\rv X), \ \forall \rv N \in \decxy{} \}$ to denote the labeling that gave rise to $\decxy{}$. 
For notational convenience we use the following shorthands throughout this section:
\begin{align}
 \CHFZ{\mathcal{T}(\rv X_y)} &\defeq \frac{1}{|\decxy{}|} \sum_{t \in \tdecxy{}}  \CHFt{} \label{eqn:CHFZT_y} \\
\eCHFZ{\rv Y} &\defeq \inf_{f \in \V} \frac{1}{|\mathcal{D}|} \sum_{x,y \in \mathcal{D}} - \log f[z \sim P_{\rv Z|x} ](y) \label{eqn:eCHFZY} \\
\eCHFZ{t(\rv X_y)} &\defeq \inf_{f \in \V}
\frac{1}{|\mathcal{D}_y|} \sum_{x \in \mathcal{D}_y} - \log f[z \sim P_{\rv Z|x} ](t(x)) \label{eqn:eCHFZt_y} \\   
\eCHFZ{\mathcal{T}(\rv X_y)} &\defeq \frac{1}{K} \sum_{t \in \tdecxyk{}} \eCHFZ{t(\rv X_y)} \label{eqn:eCHFZT_y}
\end{align}

\begin{definition}[DIB]\label{def:dib}
Let $\beta \in \mathbb{R}_{>0}$ be a hyper-parameter controlling the importance of $\V$-minimality.
The $\beta$-DIB criterion for the encoder $P_{\rv Z| \rv X}$ and predictions of $\rv Y$ from $\rv X$ is:
\begin{align*}
\dibAll{} &\defeq - \DIF{Y}{Z} + \beta * \DIFxzCy{} \\
&= -\op{H}{\rv Y} + \CHF{Y}{Z} \notag \\
&+ \sum_{y \in \mathcal{Y}} \frac{\beta}{|\mathcal{Y}|} \sum_{t \in \tdecxy{}} \left( \op{H}{t(\rv X_y)} -  \CHFt{} \right) \\
&= (const) + \CHF{Y}{Z} - \sum_{y \in \mathcal{Y}} \frac{\beta}{|\mathcal{Y}|} \sum_{t \in \tdecxy{}}  \CHFt{} \\
&= (const) + \inf_{f \in \V} \E{x,y \sim P_{\rv X, \rv Y} } {\E{z \sim P_{\rv Z \cond x}}{- \log f[z](y) }} \notag \\
&- \left( \sum_{y \in \mathcal{Y}} \frac{\beta}{ \text{\tiny{$|\decxy{}| |\mathcal{Y}|$}}} \sum_{t \in \tdecxy{}} \inf_{f \in \V} \E{x \sim P_{\rv X| y} } {\E{z \sim P_{\rv Z \cond x}}{- \log f[z](t(x)) }} \right) \\
&= (const) +  \CHF{Y}{Z} - \frac{\beta}{|\mathcal{Y}|} \CHFZ{\mathcal{T}(\rv X_y)} 
\end{align*}
Where $(const)$ does not depend on $P_{\rv Z | X}$.
\end{definition}

The empirical DIB is very similar but:
\begin{inlinelist}
\item uses $\mathcal{D}$ to estimate all expectations over $P_{\rv X, \rv Y}$;
\item uses a single sample from Bob's encoder $z \sim P_{\rv Z \cond x}$;
\item estimates the average over $t \in \tdecxy{}$ using $K$ samples $\tdecxyk{} \defeq \{ t_i \}_{i=1}^{K}$ where  $\forall i=1,\dots,K, t_i \sim \mathrm{Unif}(\tdecxy{})$ . 
\end{inlinelist}

\begin{definition}[Empirical DIB]
Let $\beta \in \mathbb{R}_{>0}$ be a hyper-parameter controlling the importance of $\V$-minimality, $\sampleD{}$ be a training set of $M$ i.i.d. input-output pairs $(x,y)$, and $K \in \mathbb{N}_{>0}$ denote the number of r.v. to sample from each $y$ decomposition of $\rv X$.
The empirical (under $\mathcal{D},K$) $\beta$-DIB criterion for the encoder $P_{\rv Z| \rv X}$ and predictions of $\rv Y$ from $\rv X$ is:
\begin{align*}
\edibAll{} &\defeq (const) + 
\inf_{f \in \V} \frac{1}{M} \sum_{x,y \in \mathcal{D}} - \log f[z \sim P_{\rv Z|x} ](y)   \notag \\
&- \left( \sum_{y \in \mathcal{Y}} \frac{\beta}{K |\mathcal{Y}| |\mathcal{D}_y|} \sum_{t \in \tdecxyk{}}
\inf_{f \in \V}
\sum_{x \in \mathcal{D}_y}
- \log f[z \sim P_{\rv Z|x} ](t(x))  \right) \\
&= (const) +  \eCHFZ{\rv Y} - \frac{\beta}{|\mathcal{Y}|} \eCHFZ{\mathcal{T}(\rv X_y)} 
\end{align*}

Where we use $z \sim P_{\rv Z|x}$ to denote that z is one sample from $P_{\rv Z|x}$,  $(const)$ is the same constant as in \cref{def:dib} ,  and $\mathcal{D}_y \defeq \{ x \cond (x,y) \in \mathcal{D} \}$ is the subset of input examples labeled $y$.
\end{definition}

\cref{proposition:pac} says that despite the previous approximations, $\edibAll{}$ still inherits $\V$-information's PAC estimation guarantees. More formally:

\begin{proposition}[PAC Estimation Guarantees]\label{proposition:pac_formal}
Let $\mathfrak{R}_{M}(\log \circ \V)$ denote the $M$-samples Rademacher complexity of $\log \circ \V \defeq \{ g \cond g(z,y) = \log f[z](y), \forall f \in \V \}$, 
$\rv X, \rv Y, \rv Z$ be r.v.s, 
$\sampleD{}$ be a dataset of $M$ i.i.d. input-output pairs $(x,y)$,
$\beta \in \mathbb{R}_{>0}$, and $K \in \mathbb{N}_{>0}$.
Assume that $\forall f \in \V, \ \forall z \in \mathcal{Z}, \ \forall y \in \mathcal{Y}$ we have $|\log f[z](y) | \leq C$, then for any $\delta \in ]0,1[$, with probability at least $ 1 - \delta $ we have that the estimation error\footnote{Up to terms that are constant in $P_{\rv Z | \rv X}$. We can provide similar guarantees when incorporating these constants due to Lemma 4 of \citepos{xu2020theory} but there is no reason to estimate these constants in our framework.} 
$\text{err} \defeq | \dibAll{} - \edibAll{}|$ is bounded by:
\begin{align}
\text{err}  \leq 2 \mathfrak{R}_M(\log \circ \V) + \beta \log |\mathcal{Y}| + C \sqrt{  \frac{2 \log \frac{1}{\delta}}{M}}
\end{align}
\end{proposition}

In order to prove \cref{proposition:pac_formal}, we need two key lemmas, one for PAC-estimation guarantees of the $\V$-sufficiency term and the other for estimation bounds of the $\V$-minimality term.

\subsubsection{Lemmas for Estimation Bounds}

The estimation guarantees that we will use for the $\V$-sufficiency term essentially comes from Lemma 3 of \citet{xu2020theory}, which we state here with a slight modification to incorporate the sampling from an encoder. 

\begin{lemma}[Estimation error $\V$-sufficiency; Lemma 3 of \citet{xu2020theory}]\label{lemma:pacVsuff}
Let $\sampleD{}$ be a dataset of $M$ input-output pairs.
Assume that $\forall f \in \V, \ \forall z \in \mathcal{Z}, \ \forall y \in \mathcal{Y}$ we have $|\log f[z](y) | \leq C$, then for any $\delta \in ]0,1[$, with probability at least $ 1 - \delta $, we have:
\begin{equation}
\Bigl \lvert \CHF{Y}{Z}  - \eCHFZ{\rv Y} \Bigr \rvert \leq \mathfrak{R}_M(\log \circ \V ) + C \sqrt{\frac{2 \log \frac{1}{\delta}}{M}} \label{eqn:pacVsuff}
\end{equation}
\end{lemma}

\begin{proof}
For the bulk of the proof, we refer the reader to the proof in \citet{xu2020theory} which uses the standard Rademacher machinery (McDiarmid's inequality and a symmetrization argument for Rademacher random variables) to prove that with probability at least $1 - \delta$, we have:
\begin{equation*}
\Bigl \lvert \CHF{Y}{Z}  - \inf_{f \in \V} \frac{1}{|\mathcal{D}|} \sum_{z,y \in \mathcal{D}_{z,y}} - \log f[z](y) \Bigr \rvert \leq \mathfrak{R}_M(\log \circ \V ) + C \sqrt{\frac{2 \log \frac{1}{\delta}}{M}}
\end{equation*}
The only difference with \cref{eqn:pacVsuff} is that the dataset $\mathcal{D}_{z,y}$ consist of i.i.d. samples from $P_{\rv Z, \rv Y}$, while our $\mathcal{D}$ consists of samples from $P_{\rv X, \rv Y}$.
Sampling a pair $(x,y) \sim P_{\rv X, \rv Y}$ and then a single from $z \in P_{\rv Z | x} $ is nevertheless equivalent to sampling directly from $(x,y,z) \sim P_{\rv X, \rv Y, \rv Z}$ so $x,y \in \mathcal{D}$ and $z \sim P_{\rv Z|x}$ in \cref{eqn:eCHFZY} can be replaced by $z,y \in \mathcal{D}_{z,y}$ to get the desired \cref{eqn:pacVsuff}.
\end{proof}

We now provide a bound on the error of the $\V$-minimality.

\begin{lemma}[Estimation error $\V$-minimality]\label{lemma:pacVmin}
$\sampleD{}$ be a dataset of $M$ i.i.d. input-output pairs $(x,y)$,
$\beta \in \mathbb{R}_{>0}$, and $K \in \mathbb{N}_{>0}$.
We have:
\begin{equation}
\frac{\beta}{|\mathcal{Y}|} \sum_{y \in \mathcal{Y}} |\CHFZ{\mathcal{T}(\rv X_y)} - \eCHFZ{\mathcal{T}(\rv X_y)} | \leq \beta \log |\mathcal{Y}|
\end{equation}
\end{lemma}
\begin{proof}

Suppose that at $y = \arg \max_{y \in \mathcal{Y}} (\CHFZ{\mathcal{T}(\rv X_y)} - \eCHFZ{\mathcal{T}(\rv X_y)})$ we have $\CHFZ{\mathcal{T}(\rv X_y)} > \eCHFZ{\mathcal{T}(\rv X_y)}$ then:

\begin{align*}
errMin &\defeq \frac{\beta}{|\mathcal{Y}|} \sum_{y \in \mathcal{Y}} |\CHFZ{\mathcal{T}(\rv X_y)} - \eCHFZ{\mathcal{T}(\rv X_y)} | \\
&\leq \beta \max_{y \in \mathcal{Y}}  | \CHFZ{\mathcal{T}(\rv X_y)} - \eCHFZ{\mathcal{T}(\rv X_y)}|  &  \text{Max > Mean}  \\
&= \beta \max_{y \in \mathcal{Y}} (\CHFZ{\mathcal{T}(\rv X_y)} - \eCHFZ{\mathcal{T}(\rv X_y)})  & \text{Assumption} \\ 
&\leq \beta \max_{y \in \mathcal{Y}} \CHFZ{\mathcal{T}(\rv X_y)}   & \text{Non negativity} \\ 
&= \beta \max_{y \in \mathcal{Y}} \frac{1}{|\decxy{}|} \sum_{t \in \tdecxy{}}  \CHFt{} & \text{\cref{eqn:CHFZT_y}} \\ 
&\leq \beta \max_{y \in \mathcal{Y}} \max_{t \in \tdecxy{}}  \CHFt{} & \text{Max > Mean} \\ 
&=\beta \max_{y \in \mathcal{Y}} \max_{t \in \tdecxy{}} \inf_{f \in \V} \E{x \sim P_{\rv X| y} } {\E{z \sim P_{\rv Z \cond x}}{- \log f[z](t(x)) }}   & \text{Def.} \\ 
&\leq \beta \max_{y \in \mathcal{Y}} \max_{t \in \tdecxy{}} \E{x \sim P_{\rv X| y} } {\E{z \sim P_{\rv Z \cond x}}{- \log P_{t(\rv X_y)}(t(x))}}   & f[\cdot]=P_{t(\rv X_y)} \\ 
&= \beta \max_{y \in \mathcal{Y}} \max_{t \in \tdecxy{}} \E{x \sim P_{\rv X|y} } {- \log P_{t(\rv X_y)}(t(x))}   & \op{E}{const} \\ 
&\leq \beta \max_{y \in \mathcal{Y}} \max_{t \in \tdecxy{}} \log \E{x \sim P_{\rv X|y} } { \frac{1}{P_{t(\rv X_y)}(t(x))}} & \text{Jensen's Ineq.}  \\ &\leq \beta \max_{y \in \mathcal{Y}} \max_{t \in \tdecxy{}} \log \sum_{n=1}^{|\mathcal{Y}|} P_{\rv N}(n) \frac{1}{P_{\rv N}(n)} & t(\rv X_y)=\rv N  \\
&= \beta \max_{y \in \mathcal{Y}} \max_{t \in \tdecxy{}} \log |\mathcal{Y}|  \\ 
&= \beta \log |\mathcal{Y}|  
\end{align*}

The fourth line uses the non-negativity of $\V$-entropy in the finite sample setting which can be shown using the non-negativity of entropy and the monotonicity of $\V$-entropy.
The fact that $\exists f \in \V, \ s.t. \ \forall z \in \mathcal{Z}$ we have $f[z]=P_{t(\rv X_y)}$ comes from the arbitrary biasing assumption of $\V$. 
The third to last line uses the fact that $t(\rv X_y)$ t.v.i. the co-domain of $t$ which is $\mathcal{Y}$. 

 In the above we assumed that $\CHFZ{\mathcal{T}(\rv X_y)} > \eCHFZ{\mathcal{T}(\rv X_y)}$ at the arg max $y$. When $\CHFZ{\mathcal{T}(\rv X_y)} \leq \eCHFZ{\mathcal{T}(\rv X_y)}$, we have:
\begin{align*}
errMin &\leq \beta \max_{y \in \mathcal{Y}}  | \CHFZ{\mathcal{T}(\rv X_y)} - \eCHFZ{\mathcal{T}(\rv X_y)}|  &  \text{Max > Mean}  \\
&= \beta \max_{y \in \mathcal{Y}} (\CHFZ{\mathcal{T}(\rv X_y)} - \eCHFZ{\mathcal{T}(\rv X_y)})  & \text{Assumption} \\ 
&\leq \beta \max_{y \in \mathcal{Y}} \eCHFZ{\mathcal{T}(\rv X_y)}  & \text{Non negativity} \\ 
&= \beta \log |\mathcal{Y}|  
\end{align*}

Where we get the last line by applying the same steps as before to bound $\eCHFZ{\mathcal{T}(\rv X_y)}$ instead of $\CHFZ{\mathcal{T}(\rv X_y)}$.

\end{proof}

Note that the latter bound is loose and is not a PAC-style bound. 
To derive a tighter PAC-style bound one can use the fact that each $\CHFt{}$ term in $\CHFZ{\mathcal{T}(\rv X_y)}$ is $\frac{C}{\sqrt{M}}$-sub-Gaussian due to \cref{lemma:pacVsuff}.
We do not provide such bounds as the current looser bounds are more succinct and sufficient to show that $\dib{}$ is easier to estimate than $\ib{}$ with finite samples. 

\subsubsection{Proof for Estimation Bounds}

We are now ready to prove \cref{proposition:pac_formal}

\begin{proof}

Due to the triangular inequality the error is:

\begin{align}
\text{err} &\defeq \Bigl| \dibAll{} - \edibAll{}\Bigr\rvert  \notag \\
&= \Bigl\lvert (const) - (const) + \CHF{Y}{Z} - \eCHFZ{\rv Y}  - \frac{\beta}{|\mathcal{Y}|} \sum_{y \in \mathcal{Y}}  \left(  \CHFZ{\mathcal{T}(\rv X_y)} - \eCHFZ{\mathcal{T}(\rv X_y)} \right) \Bigr\rvert \notag \\
&\leq |0| + \Bigl\lvert \CHF{Y}{Z} - \eCHFZ{\rv Y} \Bigr\rvert  + \frac{\beta}{|\mathcal{Y}|} \sum_{y \in \mathcal{Y}}  \Bigl\lvert   \CHFZ{\mathcal{T}(\rv X_y)} - \eCHFZ{\mathcal{T}(\rv X_y)}  \Bigr\rvert \label{eqn:trianglular}
\end{align}

We can now compute the probability of not being approximately correct:

\begin{align*}
\overline{\mathrm{PAC}} &\defeq 
\mathbb{P}\left( \text{err} > 2 \mathfrak{R}_M(\log \circ \V) + \beta \log |\mathcal{Y}| + C \sqrt{\frac{2 \log \frac{1}{\delta} }{M}} \right) \notag \\
&\leq \mathbb{P} \Biggl( \Bigl\lvert \CHF{Y}{Z} - \eCHFZ{\rv Y} \Bigr\rvert  + \frac{\beta}{|\mathcal{Y}|} \sum_{y \in \mathcal{Y}}  \Bigl\lvert   \CHFZ{\mathcal{T}(\rv X_y)} - \eCHFZ{\mathcal{T}(\rv X_y)}  \Bigr\rvert  \notag \\
&\ \ \ \ >  2 \mathfrak{R}_M(\log \circ \V) + \beta \log |\mathcal{Y}| + C \sqrt{\frac{2 \log \frac{1}{\delta} }{M}} \Biggl) & \cref{eqn:trianglular} \notag \\
&\leq \mathbb{P} \Biggl( \left( \Bigl\lvert \CHF{Y}{Z} - \eCHFZ{\rv Y} \Bigr\rvert 
> 2 \mathfrak{R}_M(\log \circ \V) + C \sqrt{ \frac{2 \log \frac{1}{\delta}}{M}}   \right) \notag\\
&\ \ \ \ \lor \left(  \frac{\beta}{|\mathcal{Y}|} \sum_{y \in \mathcal{Y}} \Bigl\lvert   \CHFZ{\mathcal{T}(\rv X_y)} - \eCHFZ{\mathcal{T}(\rv X_y)}  \Bigr\rvert 
> \beta \log |\mathcal{Y}|  \right) \Biggr)\notag \\
&\leq \mathbb{P} \Biggl( \Bigl\lvert \CHF{Y}{Z} - \eCHFZ{\rv Y} \Bigr\rvert 
> 2 \mathfrak{R}_M(\log \circ \V) + C \sqrt{ \frac{2 \log \frac{1}{\delta}}{M}}   \Biggl) \notag \\
&\ \ \ \ + \mathbb{P} \Biggl( \frac{\beta}{|\mathcal{Y}|} \sum_{y \in \mathcal{Y}} \Bigl\lvert   \CHFZ{\mathcal{T}(\rv X_y)} - \eCHFZ{\mathcal{T}(\rv X_y)}  \Bigr\rvert 
>  \beta \log |\mathcal{Y}|   \Biggl) & \text{Union Bound} \notag \\
&\leq \delta + 0 & \text{\cref{lemma:pacVsuff} and \cref{lemma:pacVmin}} \label{eqn:Npac} 
\end{align*}

So, as desired, the probability of being approximately correct is:

\begin{equation*}
\mathbb{P}\left( \text{err} \leq 2 \mathfrak{R}_M(\log \circ \V) +  \beta \log |\mathcal{Y}| + C \sqrt{ \frac{2 \log \frac{1}{\delta}}{M}}   \right)  
= 1- \overline{\mathrm{PAC}} = 1 - \delta 
\end{equation*}
\end{proof}

\section{Reproducibility}
\label{sec:appx_reproducibility}

In this section we provide further details of the hyperparameters
chosen for the various experiments in the main
text. 
Unless stated otherwise, all the models are trained for 300 epochs,  using Adam \cite{KingmaB14} as the optimizer, a learning rate of $5e-5$,  at every epoch we decay all learning rates by $(1/100)^{(1/300)}$ (so that the learning rate is decayed by $100$ during the entire training), a batch-size of 256, without data augmentation, and using 5 and 3 random seeds respectively for experiments in the main text and appendices. We checkpoint and use the model which achieves the smallest \textit{training} loss for evaluation.\footnote{Notice that we use a small learning rate, a large number of epochs, and checkpoint based on training loss because we are interested in studying the generalization ability of a model depending solely on the criterion being optimized over.}
Activation functions are $\mathrm{LeakyReLU}(x)=\max(x,0.01 * x)$ while other unspecified parameters are PyTorch \cite{PaszkeGMLBCKLGA19} defaults.
The code can also be found at \url{github.com/YannDubs/Mini_Decodable_Information_Bottleneck}.

\subsection{$\mathfam{V}$-Minimality $\mathfam{V}$-Sufficiency}
\label{sec:hyperparam_qsmin}

Bob's encoder is a neural network which maps the input $\rv X$ to a mean $\mu_z$ and standard deviation $\sigma_z$ used to parametrize a multivariate normal distribution with diagonal Gaussian: $P_{\rv Z | \rv X} = \mathcal{N}(\rv Z; \mu_z, \mathrm{softplus}(\sigma_z - 5))$, where $\mathrm{softplus}(\cdot)=\log (1+\exp(\cdot))$.
Note that we use $-5$ as done in VIB  \cite{alemi2016deep} to make the methods more comparable. 
During training we sample a single $z \sim P_{\rv Z | \rv X}$, while we sample 12 during evaluation (as done in VIB \cite{alemi2016deep}).
\footnote{Contrary to VIB, DIB does not require the use of an encoder that parameterizes a Gaussian distribution. We use a Gaussian to make it more comparable to VIB.}
The representation then goes through a batch normalization layer without trainable parameters (setting the mean to 0 and standard deviation to 1), which ensures that the representation cannot diverge as discussed in \cref{sec:appx_exploding_norm}.

The encoder is trained using two losses which are weighted by a hyperparameter $\beta$,  $\edib{}=\mathcal{L}_{\V \text{suff}} - \beta \mathcal{L}_{\V \text{min}}$:
\begin{itemize}
\item \textbf{$\V$-sufficiency} $\mathcal{L}_{\V \text{suff}}$. The representation $z$ goes through a head of architecture $\V$. The last layer of this head goes through a softmax to parametrize a distribution over of labels, i.e., $f[z](y)$ corresponds to the $y^{th}$ neuron in that layer. 
The resulting loss $\mathcal{L}_{\V \text{suff}}$ is the standard cross entropy.
We then back-propagate to jointly minimize the loss with respect to the head and the encoder.
\item \textbf{$\V$-minimality} $\mathcal{L}_{\V \text{min}}$. 
In addition to being used for $\mathcal{L}_{\V \text{suff}}$, the representation $z$ is used as input to $\V$-minimality heads that each predict a different $\rv N \in \decxy{}$ in the same way as how the $\V$-sufficiency head predicts the label $\rv Y$.
We get each $\rv N$ using \cref{algo:baseb}, i.e., assigning each example $x \in \mathcal{X}$ some index and then performing base $|\mathcal{Y}|$ expansion (see \cref{sec:appx_base}).
For the case of CIFAR10 this corresponds to: 
\begin{inlinelist}
\item assigning each image some index between $0$ and the number of examples ($\sim 6000$);
\item having 4 nuisance labels $\rv N$ corresponding to each digit of the new index, e.g., the cat number 627 will have $\rv N_1=0$, $\rv N_2=6$, $\rv N_3=2$, $\rv N_4=7$
\end{inlinelist}.

Each $\V$-minimality head predicts the corresponding $\rv N$.
Having to treat every example differently based on their underlying label $y$ (``for loop'' in \cref{alg:pseudo_dib}) is not amenable to batch GPU training, which assumes that every example in a batch is treated the same way.
We thus use the same predictor for a set $\{ \decxy{} \}_{y \in \mathcal{Y}}$ (see \cref{sec:appx_cdib}), i.e., instead of having one predictor for $\mathrm{Dec}(\rv X, \text{cat})$ and another for $\mathrm{Dec}(\rv X, \text{dog})$ where representations are $z \sim \rv Z_y$ (as shown in \cref{sec:appx_reproducibility}) we use a \textit{single} head to predict both $\mathrm{Dec}(\rv X, \text{cat}),\mathrm{Dec}(\rv X, \text{dog})$ using representations $z \in \rv Z$ from cats or dogs as inputs.
By taking an average over the loss of each head we get the $\mathcal{L}_{\V \text{min}}$ term of \cref{alg:pseudo_dib}.
Throughout the paper we unroll the optimization of $\V$-minimality heads for 5 steps, i.e., for every batch $\mathcal{L}_{\V \text{min}})$ is \textit{minimized} by $\V$-minimality heads while the encoder \textit{maximizes} it.
We show in \cref{sec:appx_minimax} that, as seen in \cref{fig:dib_neural_net}, DIB can perform similarly well with joint gradient ascent descent --- by reversing gradients which is more efficient and easier to implement.
\end{itemize}

Once the encoder is trained, we can train Alice's classifier by:
\begin{itemize}
\item \textbf{Standard (Avg, ERM)}. We freeze the trained encoder and use the representations as inputs to Alices head of architecture $\V$. Alice then trains her classifier by minimizing the usual cross-entropy.
\item \textbf{Worst ERM}. In some cases we want to explicitly find a Classifier from Alice that will perform well on train but bad on test. 
To do so, we optimize $\arg \min_{f \in V} \eRisk{}  - 0.1 *  \Risk{}$   (see \cref{sec:appx_antireg}), which corresponds to minimizing the training cross-entropy while directly \textit{maximizing} the \textit{test} cross-entropy.
\end{itemize}

Finally, Alice's classifier is then evaluated by its test log loss (risk).

\subsubsection{$\mathfam{V}$-Sufficiency}
\label{sec:hyperparam_qsuff}

For the $\V$-Sufficiency experiments (\cref{sec:qsuf}), we use a ResNet18 for Bob's encoder and a single-hidden layer MLP for Alice with varying width (see \cref{sec:sweep} for details and justification).
For \cref{fig:2d_qsuf} we use a 2 dimensional $\rv Z$ and odd-even classification of CIFAR100.
For \cref{fig:scaled_qsuf} we use a 8 dimensional $\rv Z$ and full CIFAR100.
The encoder is trained to be $\V_{Bob}$-sufficient and so we do do not use $\mathcal{L}_{\V \text{min}}$.
Alice uses an architecture $\V_{Alice}$ and is trained using standard cross-entropy.
To support \cref{proposition:opt_qsuf} we want to show that $\V$-sufficient representations are optimal when Bob and Alice \textit{have access to the entire underlying distribution}.
As a result, we evaluate Alice's classifier on the \textit{training} set. 

\subsubsection{$\mathfam{V}$-Minimality $\mathfam{V}$-Sufficiency}
\label{sec:hyperparam_qmin}

For the $\V$-minimality experiments (\cref{sec:qmin}), our goal is to show that if Bob trains $\V$-minimal $\V$-sufficient representations, any ERM trained by Alice will perform well on test (supporting \cref{theo:opt_qmin}).

Since our theory does not impose any limitation on the possible representations $\rv Z$, we need an encoder that is very flexible and as close as possible to a universal function approximator.
Thus, we use a large MLP with three hidden layers each with 2048 hidden units, for a total of around 21M parameters.
Furthermore, we use a 1024 dimensional $\rv Z$ in order to avoid constraints arising from a dimensionality bottleneck rather than from the criterion that Bob uses to train $\rv Z$.
Alice's predictive family  $\V$ is a single hidden layer MLP with 128 hidden units.
As the encoder is much larger than $\V$ we increase the learning rate of $\V$-minimality heads by a factor of $50$ to make sure that they can ``keep up'' with the changing encoder.

For \cref{fig:qminimality_gap_loglike} and \cref{fig:qminimality_Vbits} (Effect of DIB on generalization), we use the CIFAR10 dataset, and train Alice's classifier in the ``Worst ERM'' setting.
The same holds for \cref{fig:qminimality_cifar10mnist}, but uses the CIFAR10+MNIST dataset (see \cref{sec:cifar10mnist}).
In this case, the Bob's encoder is still trained using only CIFAR10.
Once the encoder is trained and frozen, we evaluate how well Alice's (worst) ERM can predict the CIFAR10 labels (as before).
In addition, we also train another classifier in $\V$ to predict the MNIST labels, using the same encoder.

For \cref{table:worstcase} (performance of Avg. and worst ERM for different regularizers) we train the encoder in different ways, and Alice's classifier in the ``Worst ERM'' (top row) and ``Avg. ERM'' (bottom row) settings.
Importantly, each regularizer is used only during Bob's training, as we are interested to know how DIB performs compared to other regularizers for representation learning, when the Alice's downstream classifier is an empirical risk minimizer as in our problem formulation (\cref{sec:formal_prob_state}).
The regularizers are as follows (we tuned all models):
\begin{inlinelist}
\item ``No Reg.'' does not use any regularizer and the encoder directly outputs the representation $z$ rather than a distribution from which to sample (this is the only such deterministic encoder in these results);
\item ``Stoch Rep.'' does not use any regularizer but the encoder is the same as the one used in DIB, i.e., it parametrizes a Gaussian distribution from which 12 $z$ are sampled and the predictions are marginalized over these samples;
\item ``Dropout'' uses 50\% dropout after every layer in the encoder and is kept when training Alice's encoder;
\item ``Wt. Dec.'' uses 1e-4 weight decay during Bob's training;
\item ``VIB'' uses a KL-divergence ``regularizer'' to force the parametrized to be closer to a standard normal distribution (as described in \cite{alemi2016deep}), the weight of the regularizer $\beta=1e-1$.
\item ``$\V^-$-DIB'' uses a one hidden layer MLP with 2 hidden units (instead of 128) and $\beta=100$.
\item ``$\V^+$-DIB'' uses a one hidden layer MLP with 8192 hidden units (instead of 8192) and $\beta=0.01$.
\item ``$\V$-DIB'' uses the correct one hidden layer MLP with 128 hidden units and $\beta=10$.
\end{inlinelist}

\subsection{CIFAR10+MNIST Dataset}
\label{sec:cifar10mnist}

We follow \citet{AchilleS18dropout} and overlay MNIST digits on top of CIFAR10 images to create the CIFAR10+MNIST dataset. 
Concretely, we pick a CIFAR10 image and on top of it overlay an MNIST image selected uniformly at random. 
The code used to generate the dataset as well as some samples can be found in \url{https://github.com/YannDubs/Overlayed-Datasets}.

\subsection{Correlation}
\label{sec:hyperparam_corr}

For the correlation experiments in \cref{sec:beyondworst}, we largely follow previous work
by \citet{jiang2019fantastic} in the sweeps over 
hyperparameters to get an initial set of models with potentially
different generalization errors. 

Let \texttt{Conv(kernel\_size, stride)} denote a convolutional layer.
The basic block of the convolutional networks consist of (in order): \texttt{Conv(3, 2, padding=1), BatchNorm, Relu, Conv(1, 1), BatchNorm, Relu, Conv(1,1), Relu, dropout}.
The final networks consist of one \texttt{Conv(1,1),Relu} used to set the correct number of channels, followed by ``depth'' number of blocks, followed by \texttt{Conv(1,1)} that set the number of channels to the dimensionality of the representation, followed by an average pooling over the spatial dimensions.
The resulting output is a (deterministic) representation which will go through an MLP with 2 hidden layers of width 128 to perform classification.

In order to be comparable to \citet{jiang2019fantastic} we sweep over the following hyperparameters:
\begin{inlinelist}
\item the learning rate (1e-3, 3e-4, 1e-4);
\item the batch size (32, 64, 128);
\item the dropout rate (0, 0.25, 0.5);
\item the width/channel size (192, 384, 768);
\item the depth/number of blocks (2, 4, 8);
\item the dimensionality of the representation (32, 128, 512).
\end{inlinelist}
We train every models with combination of these parameters on CIFAR10 and stop once the train log likelihood is better than 0.01 (with a hard stop at 300 epochs if the model did not reach that threshold by then).
The resulting subset of 562 models thus all (approximately) perform equally well on training, which enables to study the effect of hyperparameters on generalization in isolation without
the influence of performance on training as an indicator for
generalization. 
For each resulting model we compute the difference between performance on train and test (generalization gap), both in terms of accuracy and log likelihood.
We then compare the rank correlation between the desired measure and the observed generalization gap. 

The methods that we compare to are the best performing in each section of \citet{jiang2019fantastic}, namely:
\begin{inlinelist}
\item ``Entropy'' is the average entropy of the predicted probabilities;
\item ``Path Norm'' takes an input of all ones and passes it through the network where all the parameters are squared and returns the square root of the sumed logits;
\item ``Var. Grad.'' computes the average gradients at the end of training;
\item ``Sharp. Mag'' essentially finds the maximum (relative) perturbation that can be applied to the weights to get less than 0.1 log likelihood difference. We use $\frac{1}{\alpha'}$ version of sharpness magnitude as described in \citet{jiang2019fantastic}.
\end{inlinelist}

The code of \citet{jiang2019fantastic} is not (currently) public but we did our best to follow their work on all but the following three points:
\begin{inlinelist}
\item We use an MLP after the CNN, which was used to evaluate $\V^+$,$\V^-$ minimality as in the rest of the paper,
\item We do not sweep over the weight decay and optimizer but instead we vary the size of the representation, to try to incorporate a representation-specific parameter,
\item our implementation of the sharpness magnitude measures differences in log-likelihood instead of accuracy. 
\footnote{We initially only wanted to consider generalization in terms of log-likelihood since our theory only talks about log-likelihood.
For this reason, the correlation of sharpness magnitude in \cref{table:correlation} is lower than in Jiang \citet{jiang2019fantastic}, which is why we also transcribe their results.}
\end{inlinelist}

\section{Additional Experiments}
\label{sec:appx_additional}

In the rest of the appendices we provide additional experiments and results.
Specifically we investigate:
\begin{inlinelist}
\item How to obtain meaningful nested predictive families  $\mathcal{V}^- \subset \mathcal{V} \subset \mathcal{V}^+$  which we use throughout our paper (\cref{sec:sweep});
\item Different methods to deal with the min-max optimization in $\dib{}$ (\cref{sec:appx_minimax});
\item An important ``trick'' that helps the min-max optimization of $\dib{}$ (\cref{sec:appx_exploding_norm});
\item The effect of using Monte Carlo estimation of $\DIFxzCy{}$ (\cref{sec:appx_reindexing});
\item How to improve $\dib{}$ by sampling approximately independent nuisance r.v.s from $\decxy{}$ (\cref{sec:appx_base});
\item How to efficiently implement $\dib{}$ for standard GPU batch training (\cref{sec:appx_cdib});
\item How to obtain an ERM which performs well on train but bad on test (``Worst ERM'';\cref{sec:appx_antireg});
\item Whether our $\V-sufficiency$ results in \cref{sec:qsuf} hold across various settings (\cref{sec:QsufAll});
\item Why $\V$-sufficiency is not as important in large networks trained with SGD (\cref{sec:QsufXL});
\item The effect of $\beta$ on different  $\V$-minimality terms (\cref{sec:appx_Qs});
\item The performance of DIB as a standard regularizer (\cref{sec:appx_quant_2player});
\item Whether the degree of $\V$-minimality is correlated with generalization across different neural networks and datasets (\cref{sec:appx_correlation}).
\end{inlinelist}

\subsection{Sweeping over Predictive Families}
\label{sec:sweep}

\begin{figure}
\centering
 \begin{subfigure}{0.45\textwidth}
\includegraphics[width=\textwidth]{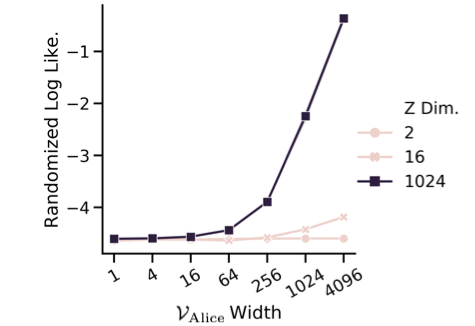} 
  \caption{Width}
 \end{subfigure}
  \begin{subfigure}{0.45\textwidth}
\includegraphics[width=\textwidth]{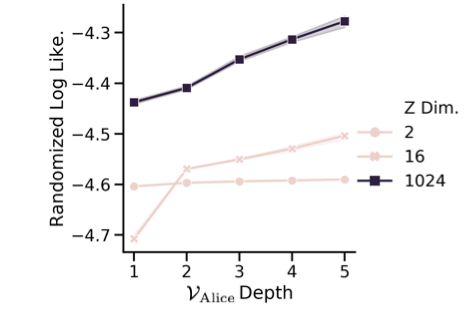} 
  \caption{Depth}
 \end{subfigure}
  \begin{subfigure}{0.45\textwidth}
\includegraphics[width=\textwidth]{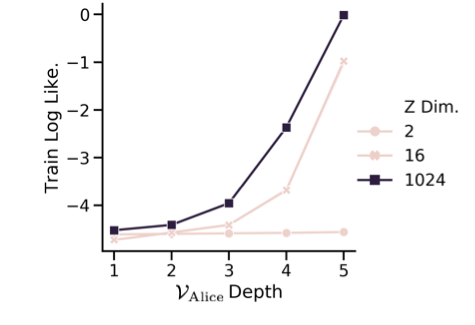} 
  \caption{Width and Depth}
 \end{subfigure}
 \begin{subfigure}{0.45\textwidth}
\includegraphics[width=\textwidth]{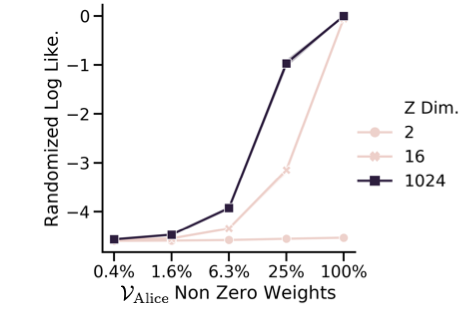} 
  \caption{Weight Pruning}
 \end{subfigure}
\caption{Sweeping over functional families. 
Each plot shows how the complexity of a functional family (measured by its ability to learn arbitrary CIFAR10 labels from a $2,16,1024$ dimensional representation) increases by sweeping over the following properties of an MLP: (a) Width, (b) Depth, (c) Width and Depth ($width=32*2^{depth}$), (d) Weight Pruning. The results in all panels are averaged over 3 runs with 95\% bootstrap confidence interval.
}
\label{fig:sweep}
\end{figure}

A core set of our experiments involve using nested predictive families $\V^- \subset \V \subset \V^+$.
In this appendix, we study different ways of ``sweeping'' over functional families, i.e. finding some parameter s.t. increasing the value it can take $k<k'$ means increasing the family $\V_{k} \subset \V_{k'}$.
Using neural networks with varying architectures, we investigate the following possibilities:

\begin{description}
\item[Width] Sweeping over the width ($w=4^k$) of a single layer MLP ($d=1$).
\item[Depth] Sweeping over the depth ($d=k$) of an MLP ($w=128$) \footnote{We also tried with a width of 32 but the differences due to depth was surprisingly less pronounced there.}.
\item[Width and Depth] Simultaneously sweeping over the depth ($d=k$) and width ($w=32*2^k$) of an MLP.
\item[Weight Pruning] Sweeping over the percentage of non-pruned weights ($\%_{\neq 0}=\frac{2^k}{2^8}$) of an MLP ($d=3$,$w=2048$). 
\footnote{To implement that, we start with a usual MLP and then prune recursively 50\% of the weights. The recursiveness ensures that every weight which were previously pruned will also be pruned in the next round. }
\end{description}

To see whether the aforementioned methods are effective ways of sweeping over functional families, we analyze their respective complexity by looking at how well they can fit arbitrary labelling, which was proposed by \cite{zhang2016understanding} as a measure of complexity intuitively similar to Rademacher complexity.
Specifically, we train a $\mathcal{X}-1024-1024-\mathcal{Z}-64-\mathcal{Y}$ MLP on CIFAR10, then freeze the encoder $\mathcal{X}-1024-1024-\mathcal{Z}$, shuffle the labels $P_{\rv Z \times \rv Y}=P_{\rv Z }\times P_{\rv Y}$, and compute the training log likelihood achieved by the predictive family $\min_{f\in \V} \eRisk$. 
We do so for increasing dimensionality (2,16,1024) of the representations $\rv Z$.

\Cref{fig:sweep} shows that each of the four aforementioned sweeping methods increases their respective complexities (besides for two dimensional representations which appears constant).
Sweeping only over depth does not appear to significantly increase the complexity of the functional family \footnote{This is likely because the sweeping interval $[1,\dots,5]$ is quite small. When using a larger sweeping interval, the optimization was harder, often yielding smaller randomized log likelihood.}.
Sweeping over weight pruning fraction is a very effective method to increase the functional family.
Sweeping over the width and depth together is also very effective to increase the complexity of the family, but we found that in some experiments the deepest MLPs were too difficult to optimize.
Sweeping over the width of the MLP increases the complexity significantly.
This last method is simple and effective, so we decided to use it as the sweeping method in the main text.
\subsection{Min-Max Optimization}
\label{sec:appx_minimax}
\begin{figure}
\centering
\includegraphics[width=\textwidth]{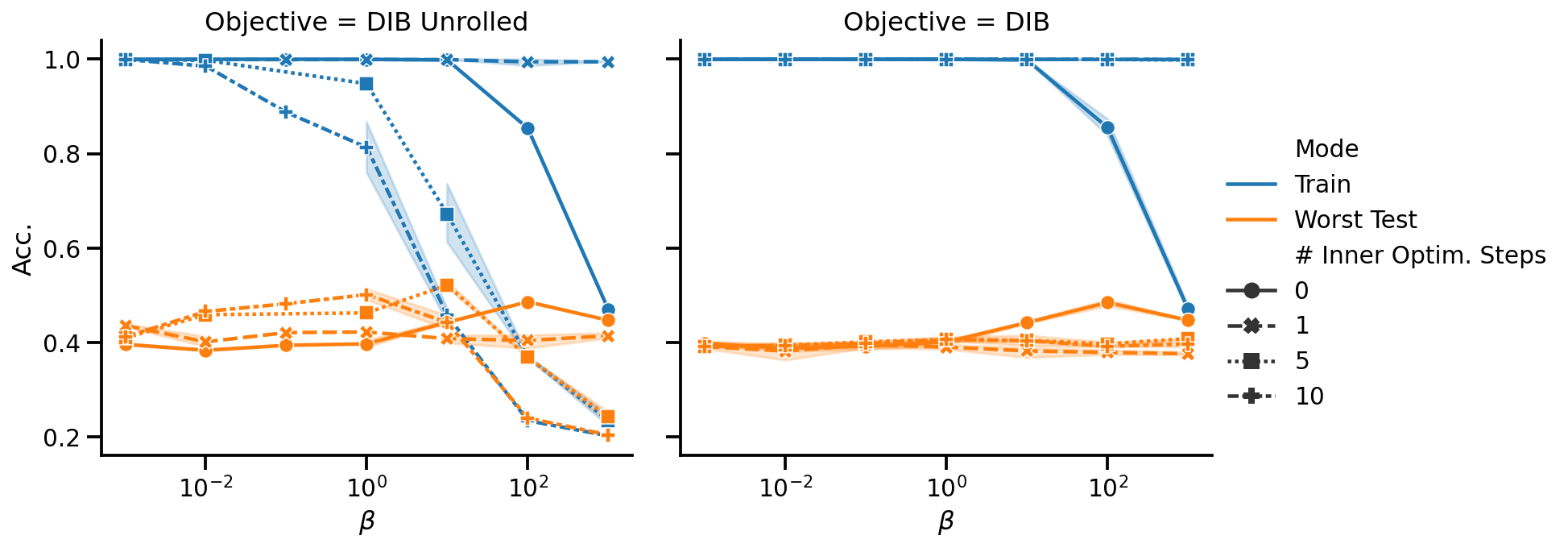}
\caption{Effect of taking multiple inner optimization steps (over $\V$) on Alice's generalization gap during the min-max optimization in DIB. 
The left figure shows the gap when higher order gradients are computed through the unrolled internal optimization.
The right figure is a baseline showing the result of taking the same number of internal optimization steps but not tracking gradients through the internal optimization.
An inner optimization of 0 means joint optimization using gradient reversing.}
\label{fig:minimax}
\end{figure}

As mentioned in \cref{sec:practical_optim}, optimizing the DIB involves a min-max procedure that is hard to optimize. 
In all our experiments, we optimize over $\rv Z$ by using variants of stochastic gradient descent (SGD) to learn the parameters of the encoder and thus require the gradient of the DIB objective \cref{eqn:dib} with respect to the encoding model's parameters.
To see where the issues arise, we show  how to compute the gradients $\frac{\partial }{\partial \rv Z} \dib{}$\footnote{We use the notation $\frac{\partial }{\partial \rv Z} \dib{}$ to denote $\frac{\partial }{\partial \theta} \dib{}$ where $\rv Z = \mathrm{encoder}_{\theta}(\rv X, \epsilon), \ \epsilon \sim \mathcal{U}(0,1)$.
In other words the encoder will take some noise $\epsilon$ and some input $x$ and will output a representation $z$, we are interested in the parameters of that encoder.
}:

\begin{align}
\frac{\partial }{\partial \rv Z} \dib{}&=  - \frac{\partial }{\partial \rv Z} \DIF{Y}{Z} + \frac{\partial }{\partial \rv Z} \beta *  \DIFxzCy{} \notag \\
&= - \frac{\partial }{\partial \rv Z} \DIF{Y}{Z} \notag \\
&+    \frac{\partial }{\partial \rv Z}  \sum_{y \in \mathcal{Y}}  \frac{\beta}{|\mathcal{Y}|*|\decxy{}|} \sum_{\rv N \in \decxy{}} \DIF{N}{Z_y}  \notag \\
&=  - \frac{\partial }{\partial \rv Z} \op{H}{\rv Y} + \frac{\partial }{\partial \rv Z} \CHF{Y}{Z} \notag \\
&+  \sum_{y \in \mathcal{Y}} \beta' \sum_{\rv N \in \decxy{}} (\frac{\partial }{\partial \rv Z} \op{H}{\rv Y} - \frac{\partial }{\partial \rv Z} \CHF{N}{Z_y})  \notag \\
&= \frac{\partial }{\partial \rv Z} \min_{f\in \V} \mathrm{R}^{(\rv Y)}(f,\rv Z) -   \sum_{y \in \mathcal{Y}} \beta' \sum_{\rv N \in \decxy{}} \frac{\partial }{\partial \rv Z} \min_{f\in \V} \mathrm{R}^{(\rv N)}(f,\rv Z) &  \text{\cref{lemma:Vent_risk}}  \notag \\
&= \left( \frac{\partial }{\partial \rv Z} \min_{f\in \V} \mathrm{R}^{(\rv Y)}(f,\rv Z) +  \sum_{y \in \mathcal{Y}} \beta' \sum_{\rv N \in \decxy{}} \frac{\partial }{\partial \rv Z} \max_{f\in \V} \mathrm{R}^{(\rv N)}(f,\rv Z) \right) \label{eqn:grad} 
\end{align}

Where we used $\mathrm{R}^{(\rv Y)}$ and $\mathrm{R}^{(\rv N)}$ to make it explicit that the risk terms are for different predictions.
For the first term in \cref{eqn:grad}, we follow the \textit{de facto} method of computing gradients, i.e., to treat the problem simply as joint optimization over $\rv Z$ and $\V$. 
Complications arise, however, because the $\V$-minimality term involves a maximization, thus giving rise to a min (over the encoding) - max (over classifiers) optimization.
There exist at least three ways of estimating such gradients:
\begin{description}
\item[Exact] Assuming that we can perform the inner optimization exactly $f^*_{\rv Y} = \arg \min_{f\in \V} \mathrm{R}^{(\rv Y)}(f,\rv Z)$ and $f^*_{\rv N} = \arg \max_{f\in \V} \mathrm{R}^{(\rv N)}(f,\rv Z) $, then we we know by the Envelop theorem \cite{milgrom2002envelope} that the gradients are simply:
\begin{equation*}
 \frac{\partial }{\partial \rv Z}  \mathrm{R}^{(\rv Y)}(f^*_{\rv Y},\rv Z) +   \beta' \sum_{y \in \mathcal{Y}} \sum_{\rv N \in \decxy{}}  \frac{\partial }{\partial \rv Z} \mathrm{R}^{(\rv Y)}(f^*_{\rv N},\rv Z) 
\end{equation*}
This exact method is very restrictive, as we can essentially only find the optimal functions if we add strong restrictions on $\V$ (e.g. linear classifiers).
\item[Joint Optimization] One could disregard the issues that arise from min-max optimization and optimize everything jointly.
This can easily be implemented by reversing the sign (sometimes referred to as a gradient reversal layer \cite{ganin2014unsupervised}). 
This is what we show in \cref{fig:dib_neural_net}.
Note that there are no guarantees of convergence, even to a local minimum.  
\item[Unrolling Optimization] A third possibility consists in ``unrolling'' the inner optimization \cite{pearlmutter2008reverse,maclaurin2015gradient,metz2016unrolled,grefenstette2019generalized} by taking a few SGD steps in the internal optimization loop  (over the functions $f$) and computing the gradients with respect to the $\rv Z$.
Note that there are again no guarantees of converging even to a local minimum.
Nevertheless, the gradients are better estimates of the true gradients than in the joint case.
A key hyper-parameter then becomes the number of inner optimization steps to perform for each $Z$ update.
\end{description}

We experimented with the three aforementioned approaches to estimating the gradients.
While preliminary results suggested that the ``Exact'' method is, unsurprisingly, better than the two other methods, we did not want to restrict the function families and thus opted for the other two approaches.
In the following, we compare two performance of the other methods that do not necessitate any restriction on $\V$.
For all of experiments, we employ an additional trick that arises due to the inner optimization not being run until convergence (see \cref{sec:appx_exploding_norm}).

\Cref{fig:minimax} shows the effect of the number of inner optimization steps on Alice's generalization.
We see that she achieves best performance by either joint optimization (which is noted as 0 inner optimization steps) or unrolling optimization with multiple inner optimization steps.
Although it is not clear from \cref{fig:minimax} using $5$ inner optimization steps is significantly better than performing joint optimization.
Indeed, at the best $\beta$ ($\beta=10$ for 5 inner steps, and $\beta=1000$ for joint optimization) taking 5 inner steps gives an average test log likelihood of $-1.56 \pm 0.03$ against $-1.65 \pm 0.00$ for joint optimization.
We also see that increasing the number of steps results in an objective which appears more robust to the choice of $\beta$.
This comes, however, at the cost of increased computational complexity.
Throughout the paper we use five inner optimization steps, but note that for larger problems, it would be advisable to use joint optimization in order to decrease the computational complexity.

\subsection{Diverging Representation from Min-Max}
\label{sec:appx_exploding_norm}

\begin{figure}
\centering
 \begin{subfigure}{0.49\textwidth}
\includegraphics[width=\textwidth]{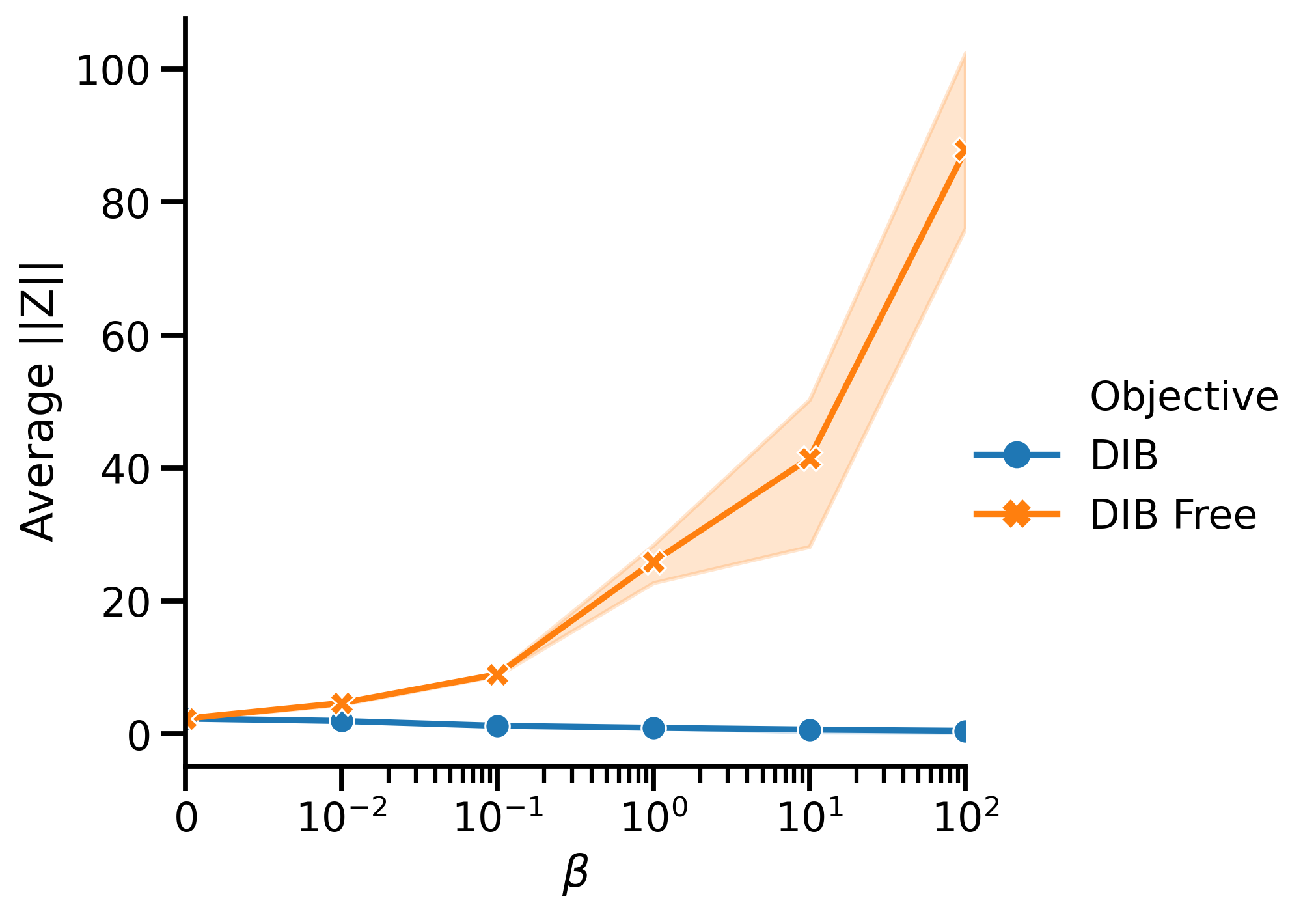} 
  \caption{Absolute Value of Mean of $\rv Z$}
  \label{fig:zmean}
 \end{subfigure}
  \begin{subfigure}{0.49\textwidth}
\includegraphics[width=\textwidth]{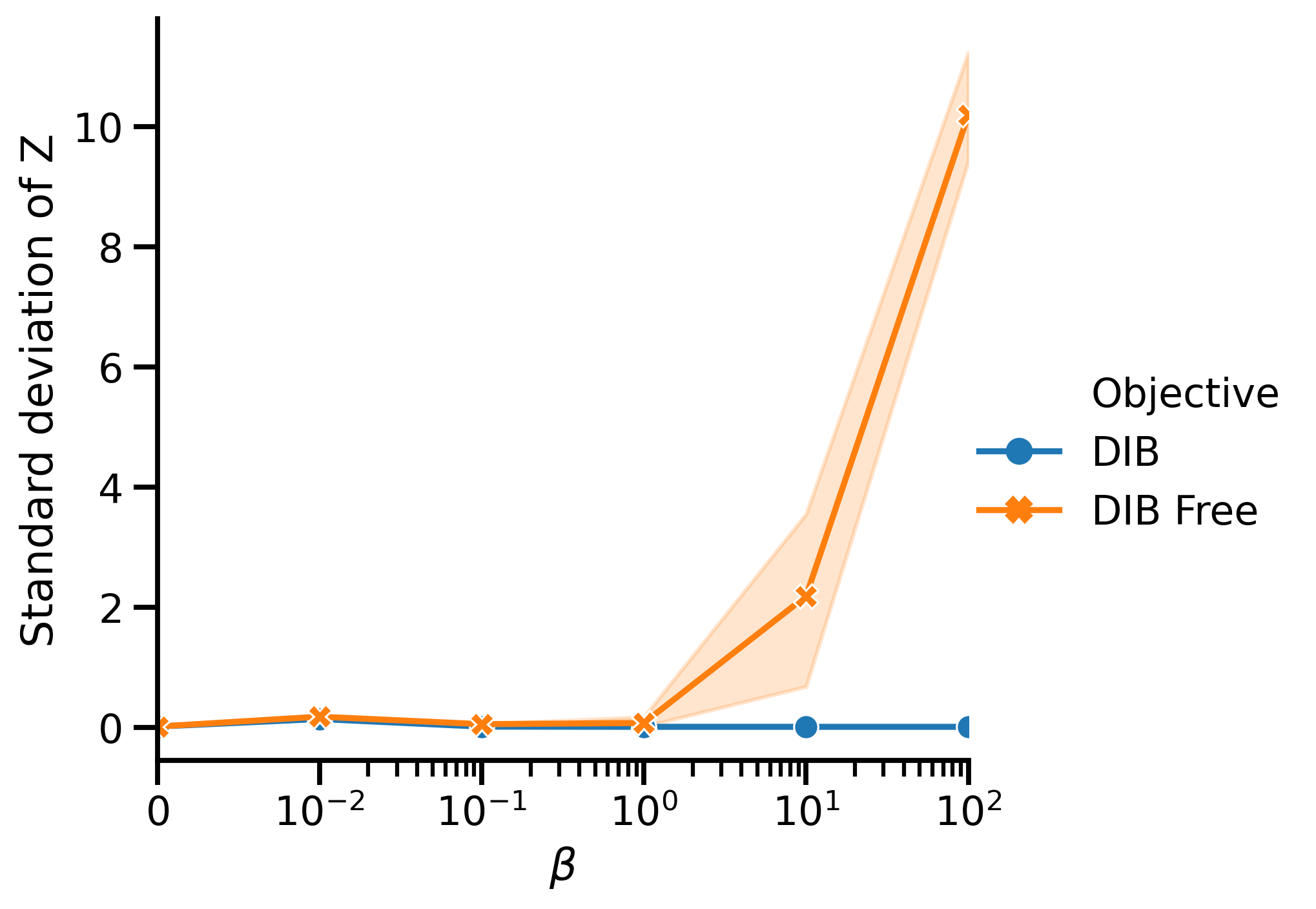} 
  \caption{Standard Deviation of $\rv Z$}
  \label{fig:zstd}
 \end{subfigure}
\caption{
Consequences of not performing the internal optimization to convergence on (a) the average (across batches) absolute value of the mean of $\rv z$; (b) the average (across batches) of the standard deviation of $\rv Z$.
In both cases we plot ``DIB Free'', which consists of the naive DIB, and ``DIB'' which uses our batch normalization solution.
}
\label{fig:znorm}
\end{figure}

As discussed in \cref{sec:appx_minimax}, the DIB objective requires a minimax optimization, which we solve using 5 steps of inner optimization.
A major issue that arises from with this approach is that the encoder can ``cheat'' because the inner optimization is not done until convergence.
As a result the encoder can learn representations $\rv Z$ that are highly variable such that the decoder $f \in \V$, which tries to predict $\rv N$, cannot adapt quickly enough.

To solve this issue we pass the sampled representations through a batch normalization layer \cite{ioffe2015bn} but without trainable hyper-parameters, i.e. we normalize each batch of representations to have a mean of zero and a variance of one.
As this is simply a rescaling, it could easily be learned by any $f \in \V$ if the inner optimization were performed until convergence (it does not modify $\V$).
Nevertheless, it does give much better results since it ensures that the encoder learns a meaningful representation, rather than taking advantage of the limited number of steps in the internal maximization.
Note that the encoder has many more parameters than the classifier, allowing it to alter the representation such that the classifier cannot ``keep up''.
\Cref{fig:znorm} shows that without this ``trick'' the mean and standard deviation in fact diverges as $\beta$ increases (labeled DIB Free).
This is solved by the normalization trick (labeled DIB), which we use throughout the paper.

\ydnote{also talk about the fact that you compute the upper bound and then set to it if larger. Maybe only arxive ?}
\subsection{Monte Carlo Estimation of $\DIFxzCy{}$}
\label{sec:appx_reindexing}
\begin{figure}
\centering
\begin{subfigure}{\textwidth}
\includegraphics[width=\textwidth]{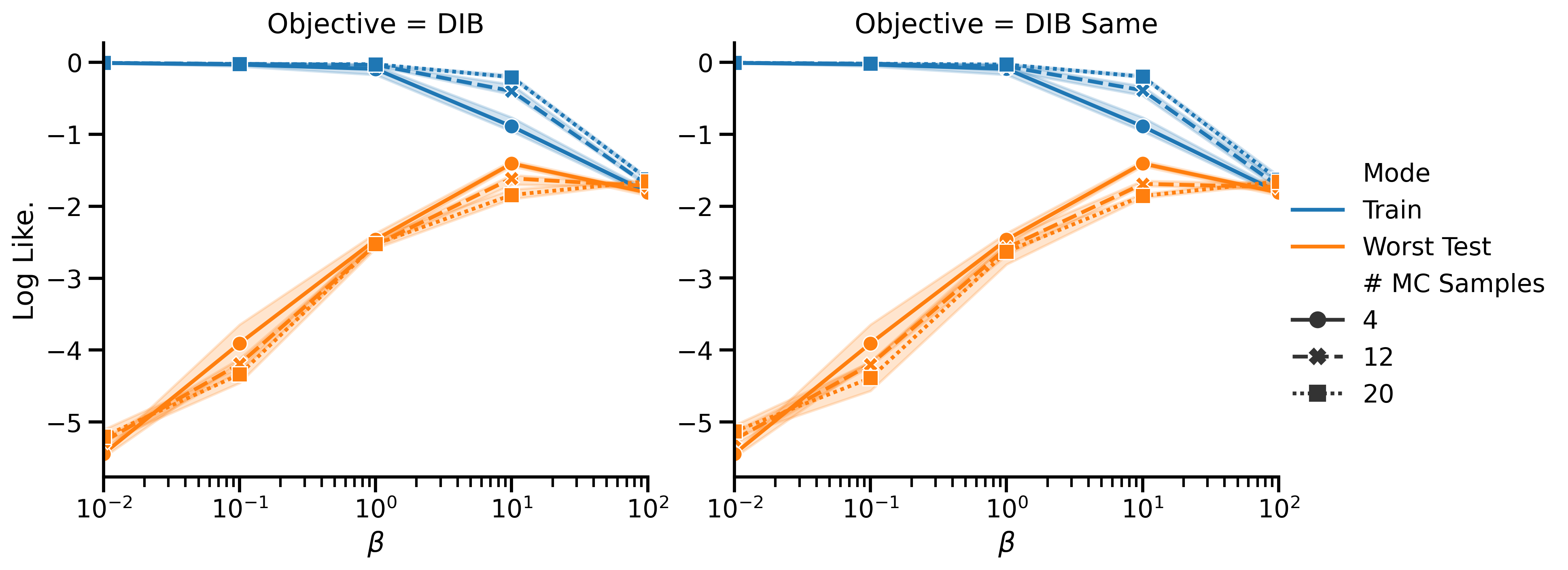} 
  \caption{Generalization Gap}
  \label{fig:montecarlo_gap_loglike}
 \end{subfigure}
  \hfill
 \begin{subfigure}{\textwidth}
\includegraphics[width=\textwidth]{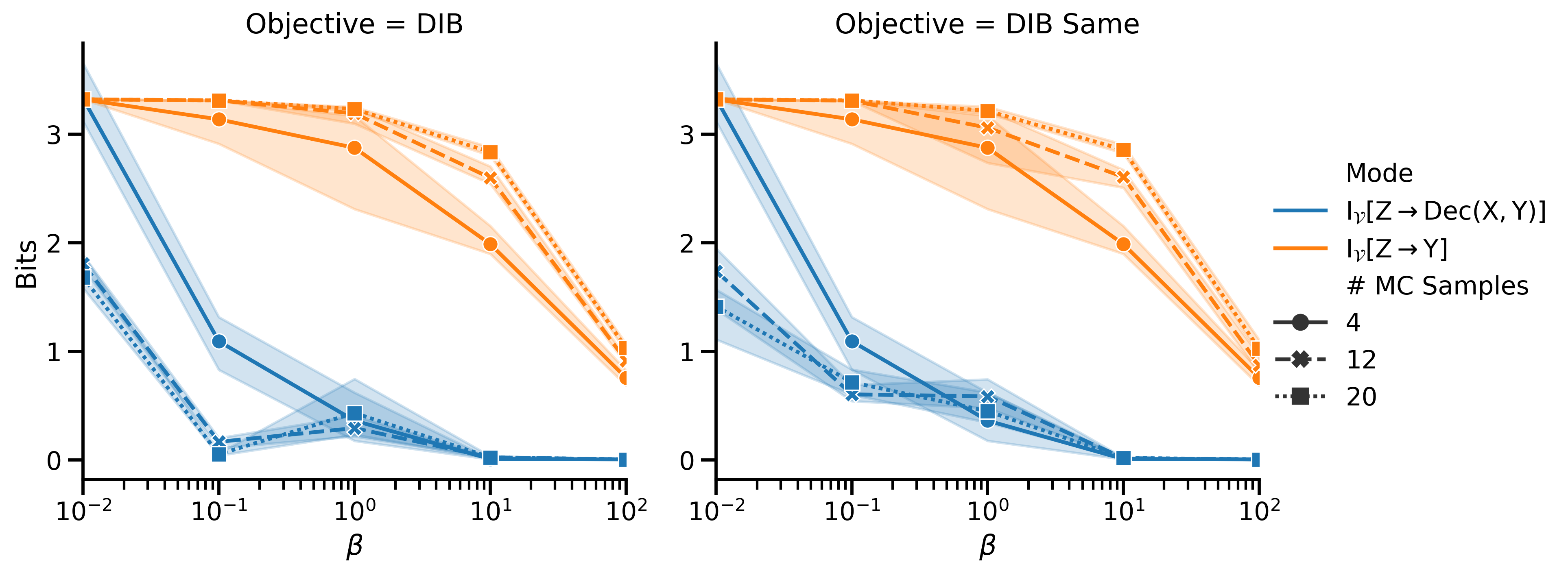} 
  \caption{Terms in DIB}
  \label{fig:montecarlo_Vbits_all}
 \end{subfigure}
\caption{
Effect of number of Monte Carlo Samples on (a) the worst case generalization gap; (b) the terms estimated by DIB. 
In both (a) and (b) the left plot show [4, 12, 20] Monte Carlo samples per label.
The right plot (labeled DIB Same) is a baseline that always has four Monte Carlo samples, but uses [1, 3, 5] predictors with different decoders with different initializations each predicting the \textit{same} $\rv N$ (for a total of [4, 12, 20] predictors as in the left plots).
All other hyperparameters are the same as for \cref{fig:qminimality}.
}
\label{fig:montecarlo}
\end{figure}

Optimizing DIB involves the task of minimizing $\DIFxzCy{}$, which requires (\cref{eqn:DIFxzCy}) computing an average over all $\rv y$ decompositions of $\rv X$ --- of which there are $|\mathcal{Y}|^{\frac{|\mathcal{X}|}{|\mathcal{Y}|}+1 }$.
As a result, even though the $\V$-information terms are sample efficient (due to the estimation bounds given in \cite{xu2020theory}), estimating it directly is not computationally efficient.
To estimate the $\DIFxzCy{}$ in a computationally efficient manner, we thus perform a Monte Carlo estimation of the average (corresponding to \texttt{random\_choice} in \cref{alg:pseudo_dib}).
In this section we show that in practice, we only require a very small number of Monte Carlo samples, allowing DIB to be implemented in a computationally efficient manner.

\Cref{fig:montecarlo} shows the result of using a different number of Monte Carlo samples (4, 12, or 20  r.v.s $\rv N$ per label $y$).
All other hyperparameters are identical to those used in \cref{fig:qminimality}.
In order to ensure that the gains come from sampling different $\rv N$s rather than from using a larger number of predictors, we also trained a model (labeled ``DIB Same'') which always uses four $\rv N$ labelings multiple predictors per labeling to match the total number of different predictors.
For example, ``DIB Same'' with 20 predictors corresponds to sampling four $\rv N$ and then having 5 predictors (different initializations) per $\rv N$ that each try to predict the same arbitrary labels $\rv N$.
Indeed, increasing the number of predictors (even with the same $\rv N$) might help as each will converge to a different local minimum.
We see that increasing the number of predictors does seem to have an effect on DIB, but the number of Monte Carlo estimates does not seem to change much compared with using more predictors.
Interestingly, the best test log likelihood comes from using the \textit{fewest} number of predictors.

This finding that the number of Monte Carlo samples has little effect on DIB might seem surprising as we only use four instead of the $|\mathcal{Y}|^{\frac{|\mathcal{X}|}{|\mathcal{Y}|}+1} \approx 10000$ different $\rv N$.
But it is important to notice that many of these $\rv Y$ decompositions of $\rv X$ are redundant (i.e. they contain the same $\V$-information).
For example, due to the invariance of $\V$ to permutations, minimizing $\DIF{N}{Z}$ also minimizes $\DIF{\pi{N}}{Z}$, for all permutations $\pi$ on $\mathcal{Y}$.
Generally speaking, the larger the functional family $\V$, the more that $\rv N \in \decxy{}$ will be redundant in that minimizing the $\V$-information with respect to some subset of N will also minimize the $\V$-information of a different subset of $\rv N \in \decxy{}$.

\subsection{$y$ Decomposition of $\rv X$ Through Base Expansion}
\label{sec:appx_base}

\begin{algorithm}[t]

\SetAlgoLined
\DontPrintSemicolon
\SetKwInOut{Input}{Input}
\SetKwInOut{Output}{Output}

\Input{All possible inputs $\mathcal{X}$, labels $\rv Y$ associated with each $\mathcal{X}$, all possible labels $\mathcal{Y}$ }
\Output{A matrix $\rv N$, where the $i^{th}$ column is the value of $\rv N_i$ for the corresponding $\mathcal{X}$}
indices $\leftarrow$ zeros($|\mathcal{X}|$) \;
Ns $\leftarrow$ zeros($|\mathcal{X}|$, $\ceil{\log_{|\mathcal{Y}|}(|\mathcal{X}|)} -2$) \;
\For{$y \in \mathcal{Y}$}{
idcs[$\rv Y == y$] $\leftarrow$  range(0,\ len($\rv Y == y$)) \;
}
\For{i $\leftarrow 0$ \KwTo $|\mathcal{X}|$}{
Ns[i,:] $\leftarrow$ base $|\mathcal{Y}|$ expansion of idcs[i] \;
}
\caption{$\rv Y$ decomposition of $\rv X$ through base expansion}\label{algo:baseb}
\end{algorithm}

In the main paper and \cref{sec:appx_reindexing}, we discussed how estimate to estimate $\DIFxzCy{}$ by uniformly sampling $\rv N \in \decxy{}$.
As previously mentioned, many $\rv N \in \decxy{}$ will actually be redundant and have the same $\V$-information.
It thus makes sense to only using $\rv N$s which are (approximately) mutually independent so as to minimize redundancies.
We do so by assigning to each $\mathcal{X}_y$ a certain index and then computing the base $|\mathcal{Y}|$ expansion of that index.
For example, in the case of binary cat-dog classification, we would assign some index to all cats and have $\rv N$ be the binary expansion of that index. 
Using base $|\mathcal{Y}|$ indexing gives a set $\{ \rv N_i \}_i$ of $\ceil{\log_{|\mathcal{Y}|} |\mathcal{X}| - 2}$ elements, which ensures that
\begin{inlinelist}
\item each of the $\rv N$ is a deterministic function from $\mathcal{X}_y \to \mathcal{Y}$ and thus part of $\decxy{}$;
\item each of the $\rv N$ are (approximately) uncorrelated and thus will not be redundant.
\end{inlinelist}
The algorithm to compute the set of $\rv N$ from which we estimate $\DIFxzCy{}$ is described in \cref{algo:baseb}.

\begin{figure}
\centering
  \begin{subfigure}{0.49\textwidth}
\includegraphics[width=\textwidth]{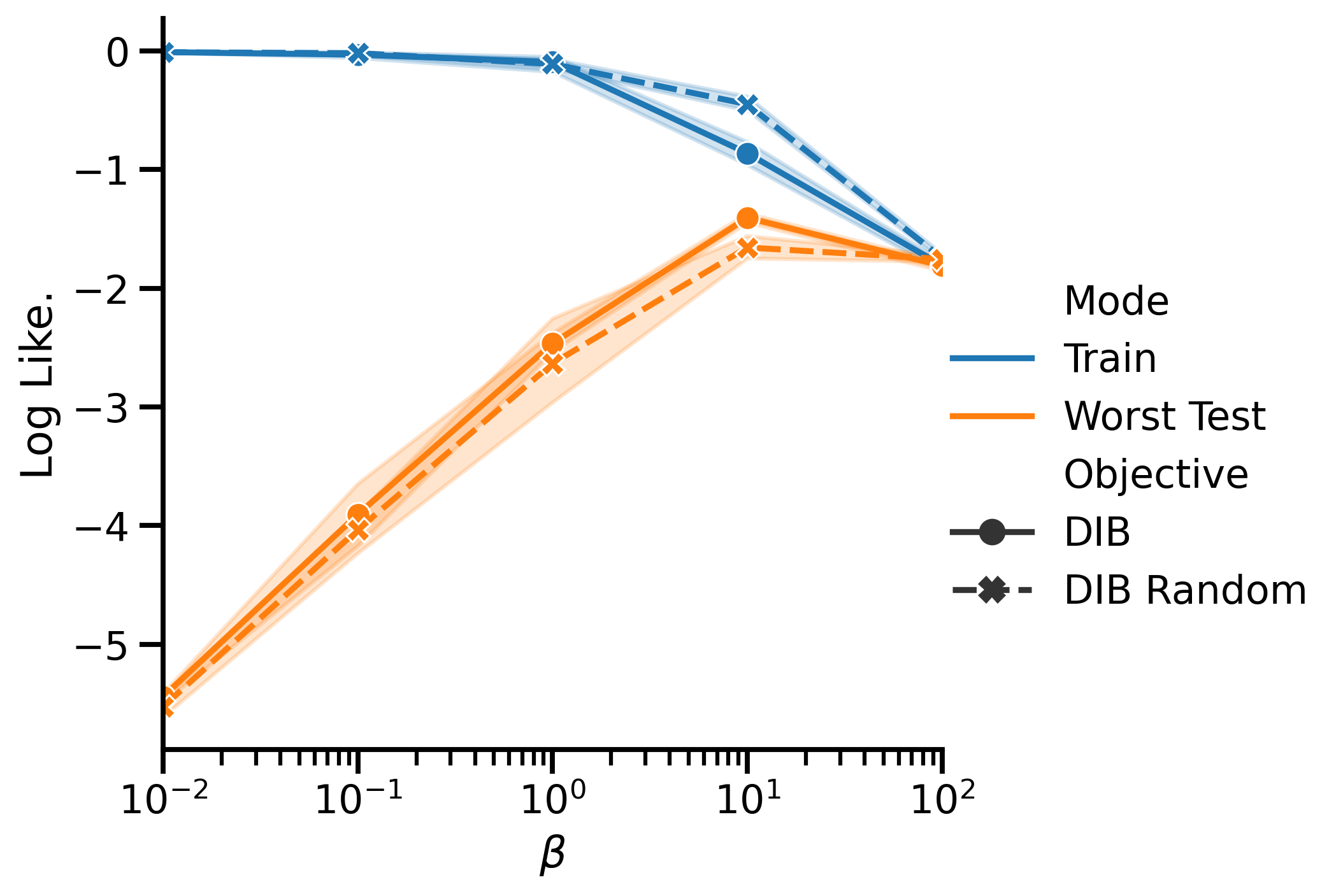} 
  \caption{Generalization Gap}
  \label{fig:randlab_gap_loglike}
 \end{subfigure}
  \begin{subfigure}{0.49\textwidth}
\includegraphics[width=\textwidth]{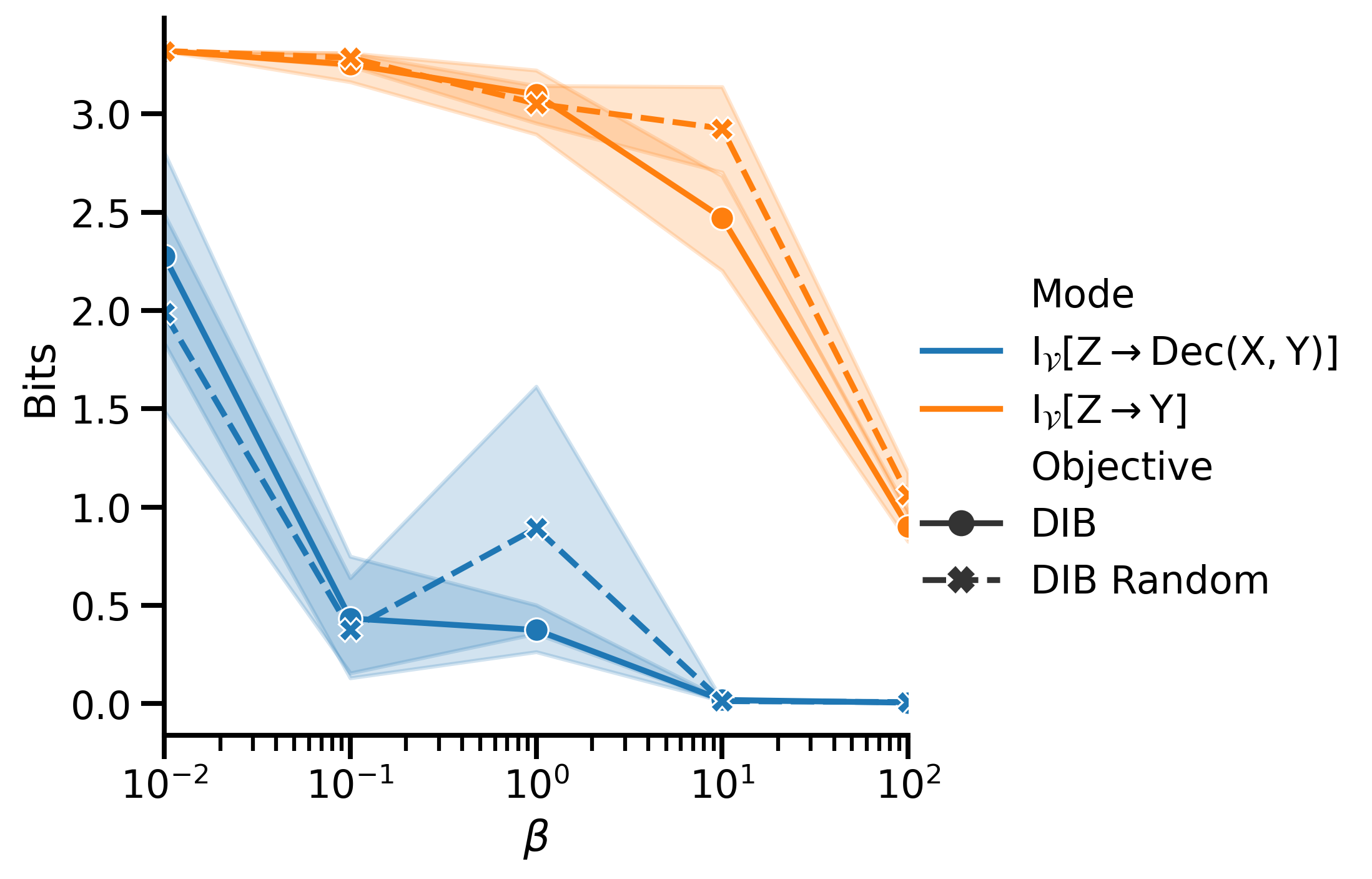} 
  \caption{Terms in DIB}
  \label{fig:randlab_Vbits_all}
 \end{subfigure}
\caption{
Effect of using Base $|\mathcal{Y}|$ expansion (labeled ``DIB'') vs. randomly selecting $\rv N$ from the set of $\rv Y$ decompositions of $\rv X$ (labeled ``DIB Random'')  on
(a) the worst case generalization gap; (b) the terms estimated by DIB. 
All other hyper-parameters are the same as for \cref{fig:qminimality}.
}
\label{fig:randlab}
\end{figure}

\Cref{fig:randlab} shows the effect of using the $y$ decomposition of $\rv X$ through base $|\mathcal{Y}|$ expansion, instead of randomly sampling labels $\rv N \in \decxy{}$.
We see that although differences are not large, the base expansion is better.
At the optimal $\beta=10$, using base $|\mathcal{Y}|$ expansion gives a test log likelihood of $-1.41 \pm 0.05$ vs. $-1.66 \pm 0.09$.
Note that the base $|\mathcal{Y}|$ expansion does not incur any additional computational costs nor does it have any other drawbacks that we know about.

\subsection{Sharing Predictors of $\{ \decxy{} \}_{y \in \mathcal{Y}}$ for Batch Training}
\label{sec:appx_cdib}

\begin{figure}
\centering
\begin{subfigure}{0.49\textwidth}
\includegraphics[width=\textwidth]{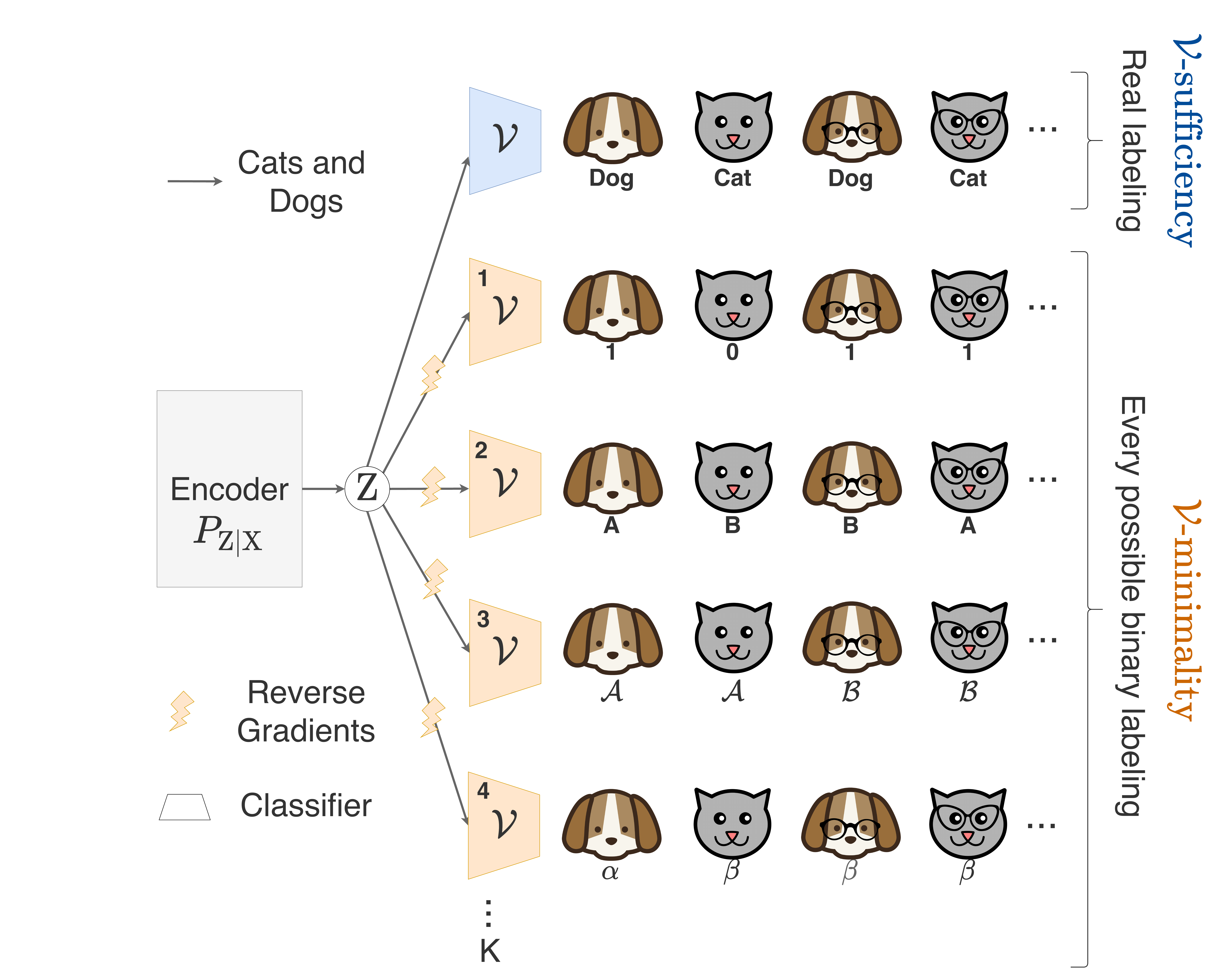} 
  \caption{DIB for Batch Training}
  \label{fig:dib_batch}
 \end{subfigure}
  \begin{subfigure}{0.49\textwidth}
\includegraphics[width=\textwidth]{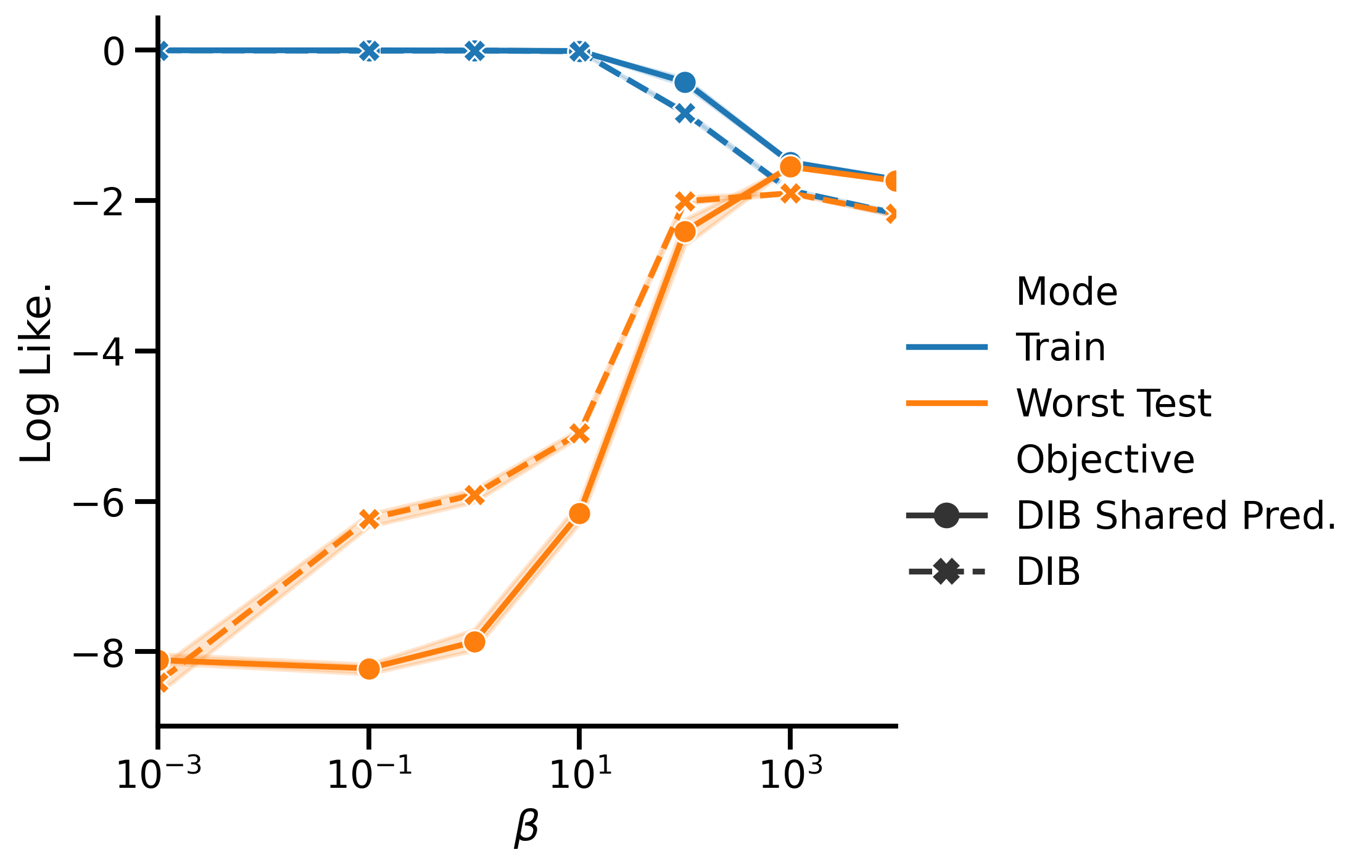} 
  \caption{Generalization Gap}
  \label{fig:conditional_dib}
 \end{subfigure}
\caption{(a) Schematic illustration of using the predictor / $\V$-minimality head for a set $\{ \decxy{} \}_{y \in \mathcal{Y}}$ which is more amenable to batch training than the standard way of one predictor / $\V$-minimality head for each $\decxy{}$ shown in \cref{fig:dib_neural_net}.
(b) Effect on Alice's log likelihood when sharing the predictors (labeled ``DIB Shared Pred.'') compared to no sharing (labeled ``DIB.'').}
\label{fig:cdib}
\end{figure}

In \cref{eqn:DIFxzCy} we see that every example has to be treated differently depending on its underlying label $y \in \mathcal{Y}$.
Indeed, $\decxy{}$ depends on the underlying label $y$.
In practice this means having a ``for loop'' over $y \in \mathcal{Y}$ (see \cref{fig:dib_practice}) and using a different $\V$-minimality head for each $\decxy{}$.
This makes DIB hard to take advantage of the standard batch GPU training, where all examples in a batch are assumed to go through the same predictor regardless of their underlying label.
Here we investigate whether DIB can be modified to take advantage of batch training by having a \textit{single} predictor for a set of nuisance r.v. $\{ \decxy{} \}_{y \in \mathcal{Y}}$ as seen in \cref{fig:dib_batch}, i.e. treating all representations $z \sim \rv Z$ the same way instead of having to distinguish them based on their underlying label $z \sim \rv Z_y$.
This has the same under underlying computational complexity, but it has the advantage of being trainable in batches.
Interestingly, \cref{fig:conditional_dib} shows that sharing the predicors (``DIB shared Pred.'') reaches a better test performance in practice. 
This is probably an artefact of the values $\beta$ we are sweeping over, but it nevertless shows that one can perform well by sharing the predictors and this take advantage of batch training. 
\subsection{Searching for an ERM That Does Not Generalize}
\label{sec:appx_antireg}

In \cref{sec:qmin} we briefly outlined a method to test \cref{theo:opt_qmin}, which states that \textit{all} ERMs should generalize well when trained from $\V$-minimal $\V$-sufficient representations.
In other words, no ERM should have a non-zero generalization gap.
Since we can only approximate $\V$-minimal $\V$-sufficient representations, our aim is to show that no ERM predicting from such a representation will incur a large generalization gap. Of course, it is infeasible to train all possible ERMs and then check that each generalizes well. So instead we directly search for the ERM with the largest generalization gap (worst case). We expect from \cref{theo:opt_qmin} that even this ERM will have a small gap.
Specifically, we want to maximize the test loss under the empirical risk minimization constraint:

\begin{align}
\begin{split} \label{eq:bad_erm}
\arg \max_{f \in V} & \quad  \Risk{}
 \\
\text{s.t.}
&  \quad  \eRisk{} = \min_{f \in V} \eRisk{}
\end{split}
\end{align}

Using a Lagrangian relaxation of \cref{eq:bad_erm} and flipping the sign, our objective is then:

\begin{equation}
\label{tomtrick}
\arg \min_{f \in V} \eRisk{}  - \gamma  \Risk{} 
\end{equation}

We thus minimize the training loss as usual while maximizing the test loss times a factor $\gamma$. This can easily be optimized by training on training \textit{and} test examples, but multiplying all the gradients of test examples by $- \gamma$.
Note that this is the same loss used by \cite{tom2019trick}.

\begin{figure}
\centering
\includegraphics[width=\textwidth]{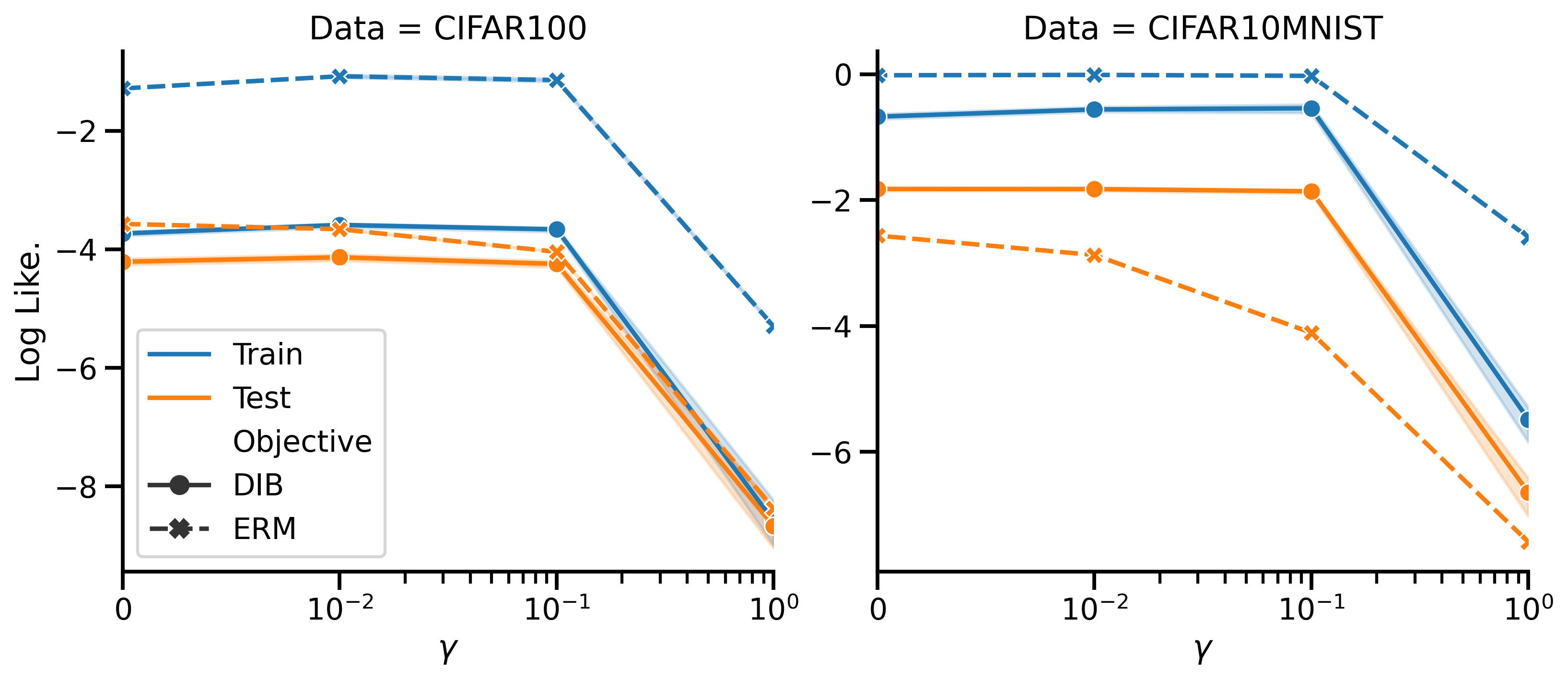}
\caption{Sweeping over $\gamma$ to find a poorly generalizing ERM. 
As $\gamma$ increases, the test performance decreases without having much effect on the training performance, until approximately $\gamma=0.1$. In these experiments, Bob learns representations with either joint ERM (labeled ERM) or DIB. Alice then trains a decoder from Bob's representation using Eq. \ref{tomtrick}. Left plot shows results on CIFAR100, right plot is for our CIFAR10+MNIST dataset.
}
\label{fig:gammas}
\end{figure}

In order to find an $f$ that is a poorly generalizing ERM, we sweep over values of $\gamma$ and select the largest such that $f$ is (approximately) an ERM.
\Cref{fig:gammas} shows that $\gamma=0.1$ seems to be a good value for both datasets, to ensure that $f$ is approximately an ERM but performs as poorly as possible on test.
We thus use this value for all ``worst case'' experiments in the paper.
\ydnote{Add more sweeps of gamma for arxive}
\subsection{Optimality of $\V$-Sufficiency in Various Settings}
\label{sec:QsufAll}

\begin{figure}
\centering
\begin{subfigure}{0.77\textwidth}
\includegraphics[width=\textwidth]{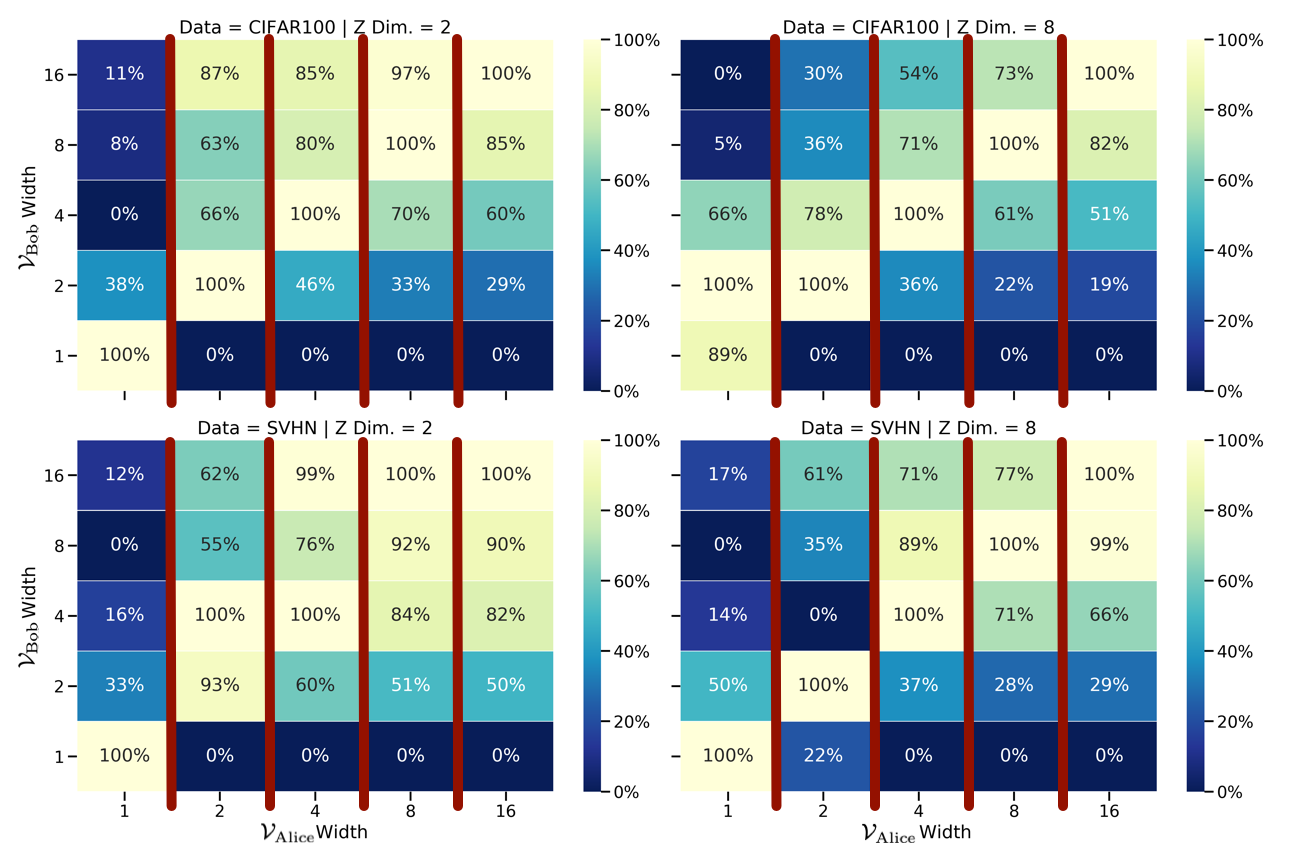} 
  \caption{MLP}
  \label{fig:qsuff_all_mlp}
 \end{subfigure}
  \begin{subfigure}{0.77\textwidth}
\includegraphics[width=\textwidth]{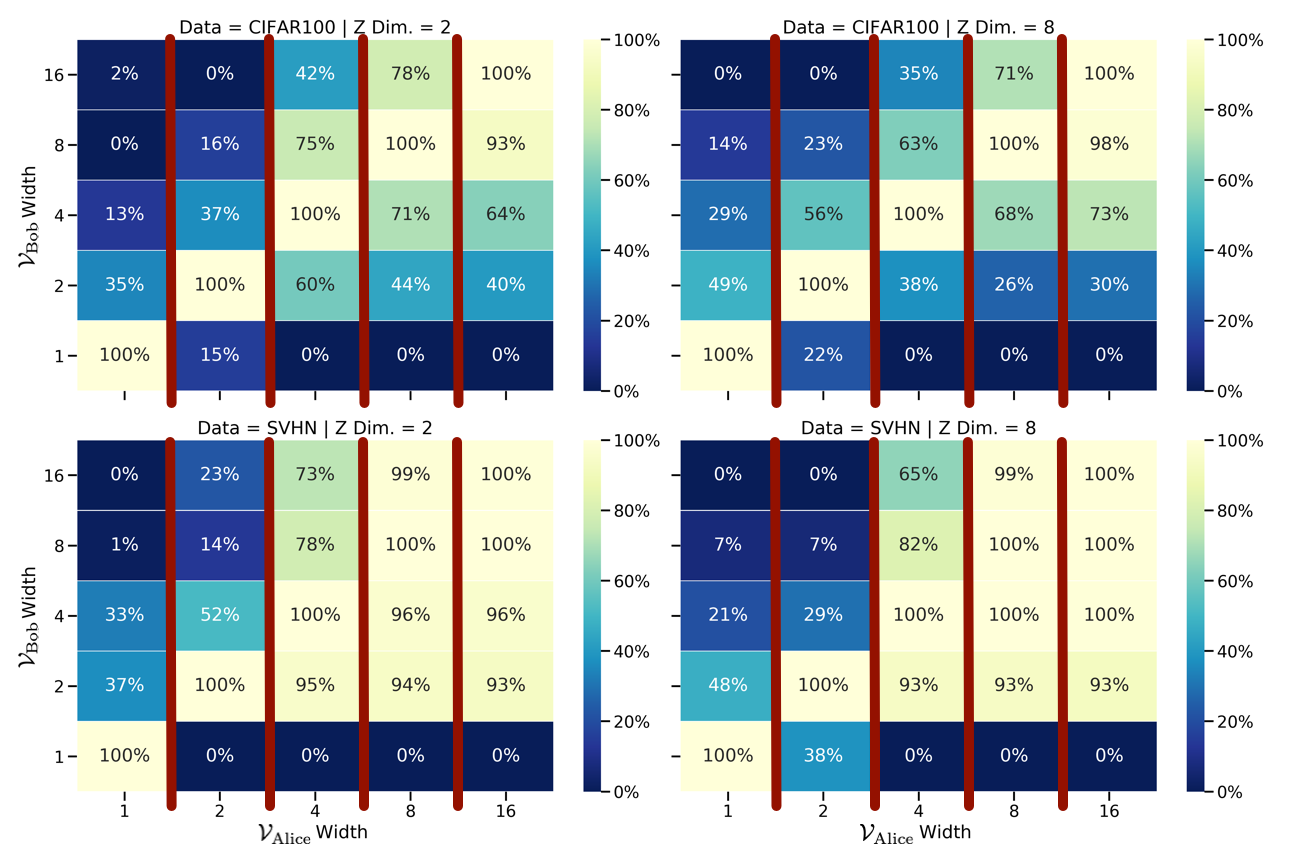} 
  \caption{ResNet18}
  \label{fig:qsuff_all_resnet18}
 \end{subfigure}
\caption{Optimality of $\V$-sufficiency for different hyperparameters. 
Both plots show the comparative training performance of $\V_{Bob}$-sufficient representations for classifiers in $\V_{Alice}$. As in the main text, the log likelihood is scaled to lie in the range $[0,\dots,100]$ for each column.
The predictive families $\V_{*}$ are single MLPs with varying width. (a) shows \MLP{} encoders, where the left and right columns use 2 and 8 dimensional Z, respectively, and the rows are CIFAR100 and SVHN. (b) is the same as in (a) but with a ResNet18 encoder.
}
\label{fig:qsufAll}
\end{figure}

In \cref{fig:scaled_qsuf} we have provided experimental evidence for the optimality of $\V$-sufficient representations for the considered setting (CIFAR100, 8-dimensional representations, ResNet18).
Here we show that similar conclusions hold for CIFAR100 and SVHN, with 2- or 8-dimensional representations, and with ResNet18 or \MLP{} encoders.

\cref{fig:qsufAll} summarizes all the results under various settings. 
Similarly to \cref{fig:scaled_qsuf} we see that for most $\V_{Alice}$ the empirical optimal representation is recovered by maximizing $\DIF{Y}{Z}$. The 3 exceptions (e.g. SVHN, MLP, 2 dimenional representation, width 2 $\V_{Alice}$) out of the 40 possible $\V_{Alice}$ in each setting are likely due to optimization issues.
Notice that the results for SVHN with a ResNet18 encoder show that when the width of $\V_{Alice}$ and $\V_{Bob}$ are both larger, performance becomes less dependent on $\V_{Bob}$.
We investigate this phenomenon in the following subsection \cref{sec:QsufXL}.

\subsection{The Surprising Effect of Large Neural Families Trained with SGD}
\label{sec:QsufXL}

\begin{figure}
\centering
\begin{subfigure}{0.47\textwidth}
\includegraphics[width=\textwidth]{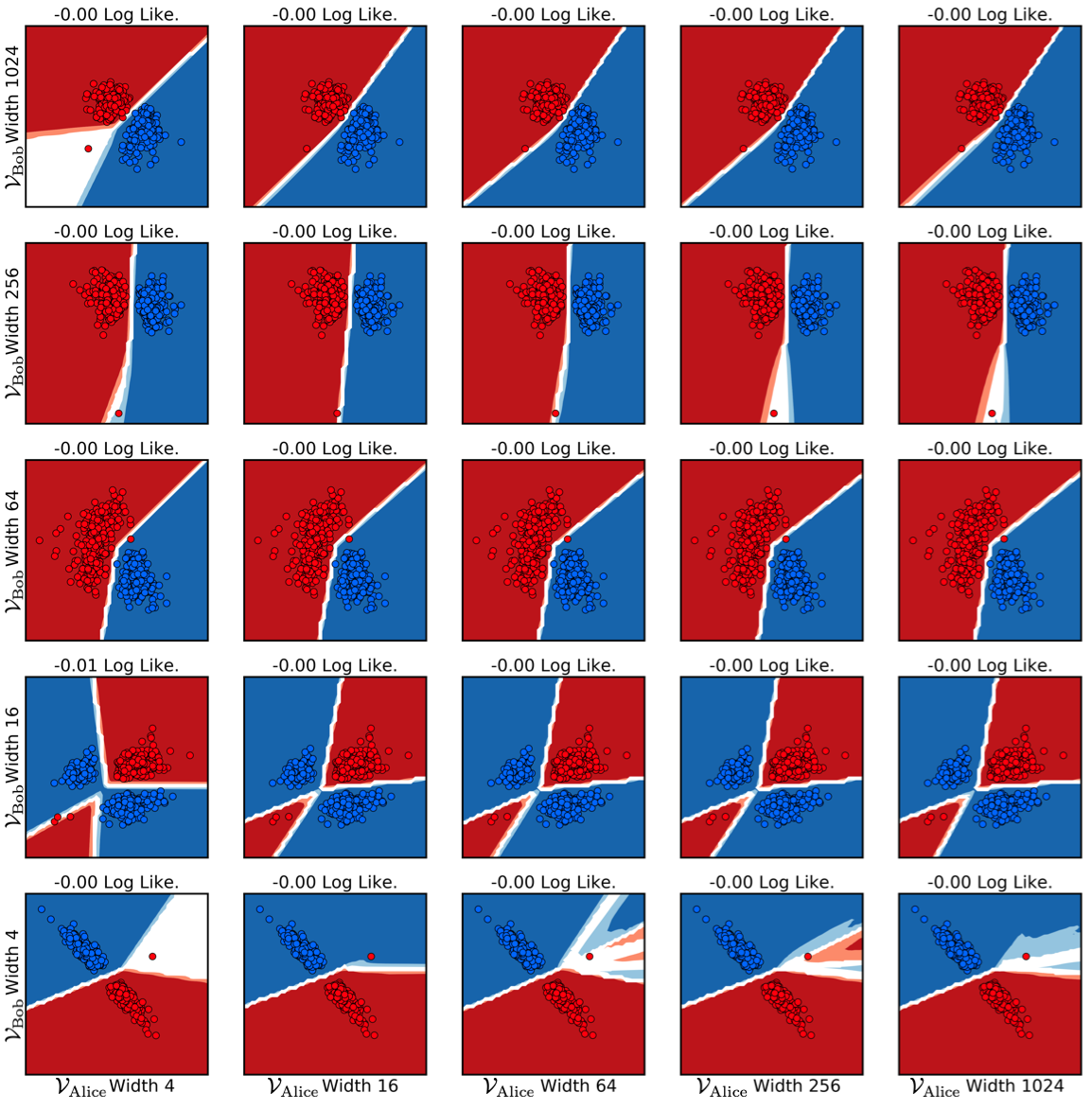} 
  \caption{2D Visualization}
  \label{fig:2d_qsufXL}
 \end{subfigure}
  \begin{subfigure}{0.47\textwidth}
\includegraphics[width=\textwidth]{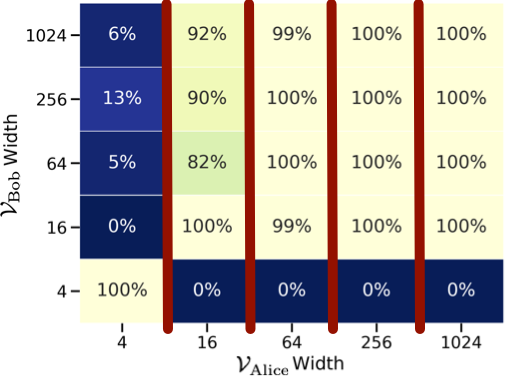} 
  \caption{Scaling Up}
  \label{fig:scaled_qsufXL}
 \end{subfigure}
\caption{Optimality of $\V$-sufficiency in larger functional families.\david{Log like in (a) isn't printed properly}
Plots are the same as in \cref{fig:qsuf} but for widths $[4,16,64,256,2014]$ instead of $[1,2,4,8,16]$. Notably, the representations learned by Bob become nearly linearly decodable for the largest $\V_{Bob}$.
}
\label{fig:qsufXL}
\end{figure}

\Cref{fig:qsufXL} shows the same results as \cref{fig:qsuf} for much larger widths of $[4,16,64,256,1024]$ instead of $[1,2,4,8,16]$.
\Cref{fig:scaled_qsufXL} shows that $\V_{Bob}=\V_{Alice}$ is still optimal but  the difference for using $\V_{Alice}$ larger than $\V_{Bob}$ is much less pronounced and mostly disappears at the largest widths.
For example, the difference in performance when $\V_{Bob}$ has width 4 and $\V_{Alice}$ has width 16 is much larger than the difference in performance when $\V_{Bob}$ has width 256 and $\V_{Alice}$ has width 1024.

This seems to imply that the larger the functional families, the more similar they become. However, \Cref{fig:qsufXL}a suggests a simpler explanation. 
Notice that that when $\V_{Bob}$ is very large, the representation that is learned is more linearly decodable.
Recall that from \cref{fig:sweep} larger functional families are indeed more powerful, however it seems that once networks are wide enough, SGD favors the learning of classifiers that are very simple, and thus the functional family does not need to be matched.
This notion echoes 
the fact that SGD is known to learn simple classifiers first \cite{KalimerisKNEYBZ19}.

We emphasize that this behavior does not contradict our theoretical prediction that the optimal setting for Bob is to set   $\V_{Bob}=\V_{Alice}$.
Instead, it simply illustrates that if both Alice and Bob use large neural networks with standard initialization and train with SGD, the difference in the full expressivity of the models is not accessed. It is still provably better to use the same family, and we never observe performance to be less than optimal when the families are chosen to match.

\subsection{Effect of $\beta$ on Different $*$-Minimality Terms}
\label{sec:appx_Qs}

\begin{figure}
\centering
\begin{subfigure}{0.49\textwidth}
\includegraphics[width=\textwidth]{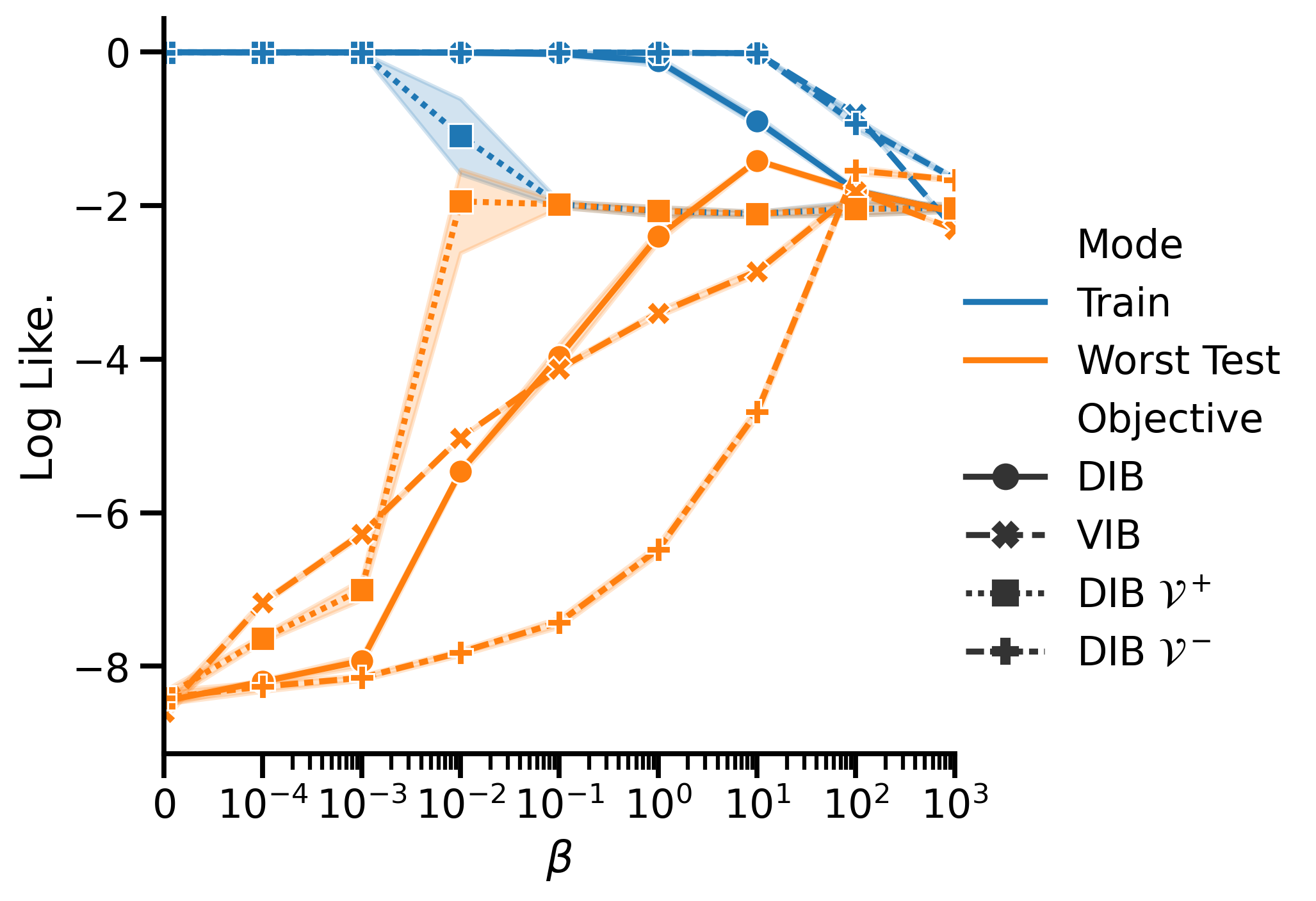} 
  \caption{Generalization Gap}
  \label{fig:qminimality_gap_loglike_all}
 \end{subfigure}
  \newline
  \begin{subfigure}{0.49\textwidth}
\includegraphics[width=\textwidth]{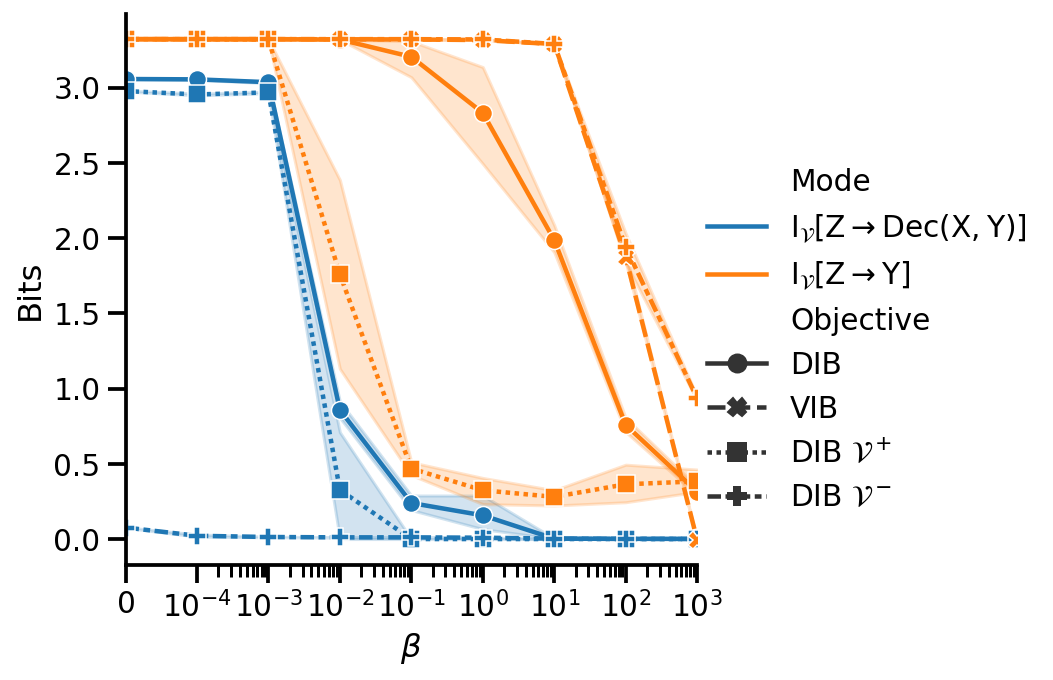} 
  \caption{Terms in DIB}
  \label{fig:qminimality_Vbits_all}
 \end{subfigure}
 \begin{subfigure}{0.45\textwidth}
\includegraphics[width=\textwidth]{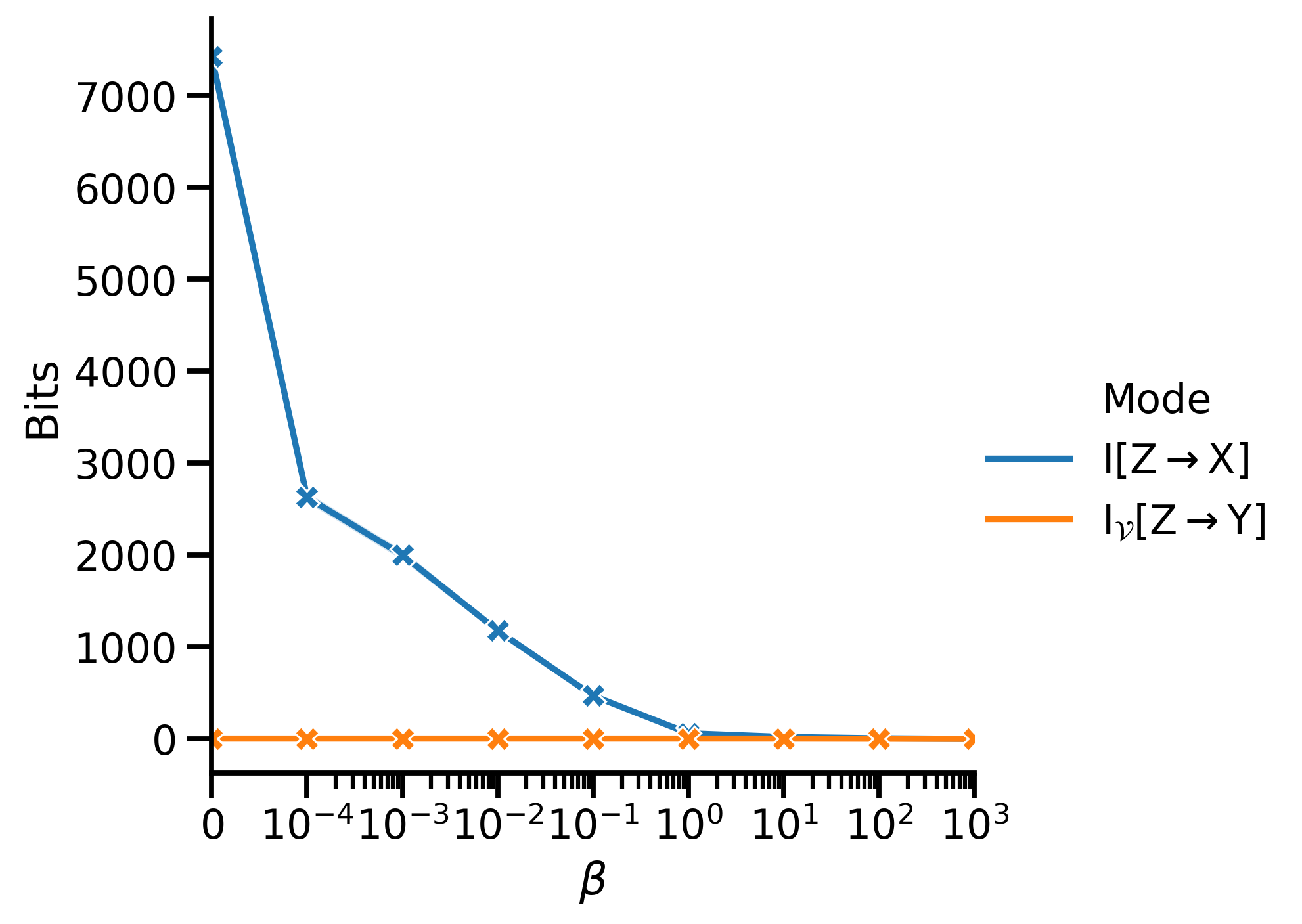} 
  \caption{Terms in VIB}
  \label{fig:qminimality_Vbits_vib}
 \end{subfigure}
\caption{
Effect of $\beta$ on the worst case generalization gap, as well as the terms estimated by DIB (using different $\V$s for the minimality term) and VIB. 
(a) Train log likelihood and worst case test log likelihood of the different predictors; 
(b) Estimated $\DIF{Y}{Z}$ and $\DIFxzCy{}$, $\DIFpxzCy{}$, and $\DIFmxzCy{}$;
(c) Estimated $\DIF{Y}{Z}$ and (a variational estimate of) $\MI{Z}{X}$ for VIB \cite{alemi2016deep}.
Similarly to \cref{fig:qminimality} all experiments ran on cifar10 and the results show the average over 5 runs as well as $95\%$ bootstrap confidence interval.
}
\label{fig:qminimality_all}
\end{figure}

In \cref{fig:qminimality_gap_loglike} and \cref{fig:qminimality_Vbits}, we showed the effect of varying the $\beta$ of Bob's DIB objective on Alice's worst-case performance. We also plotted the estimated value of the individual $\V$-sufficiency and $\V$-minimality terms.
\Cref{fig:qminimality_all} shows the same plot but for different $*$-minimality terms in Bobs objective, in particular: $\V$-minimality (single-layer MLP with 128 hidden units), $\V^+$-minimality (single-layer MLP with 8192 hidden units), $\V^-$-minimality (single-layer MLP with 2 hidden units), and
variational minimality (using VIB's bound \cite{alemi2016deep}).
For each objective, the  best (over $\beta$) test performance is transcribed in \cref{table:worstcase}.

\Cref{fig:qminimality_gap_loglike_all} shows that $\V^+$-minimality gives rise to an objective that may be more robust to the choice of $\beta$ but also exhibits higher variance (in line with the estimation bounds from \cite{xu2020theory}).
\Cref{fig:qminimality_Vbits_all} shows that the representation can be $\V^-$-minimal (that is, $\DIFmxzCy{}$ is close to zero) but have little effect on the $\V$-sufficiency term (large $\DIF{Y}{Z}$).
\yann{is that last sentence of any use ?}\david{This is mainly saying VIB is harder to compare with V-minimality? I took out..}

\subsection{$\V$-Minimality as a Regularizer}
\label{sec:appx_quant_2player}

In the main text, we have seen that minimizing $\DIFxzCy{}$ is theoretically optimal (\cref{theo:opt_qmin}) and empirically outperforms other regularizers in our two stage setting (\cref{table:worstcase}).
It is thus natural to ask whether $\DIFxzCy{}$ can also perform well as a regularizer in a standard neural network setting.

\begin{table}[ht!]
\centering
\caption{Evaluation of regularizers for permutation invariant classification (test accuracy).}
\label{table:perminv}
\begin{tabular}{@{}lrrrrrrrr@{}}
\toprule
& No Reg.  %
& Stoch. Rep.    %
& Dropout  %
& Wt. Dec.   %
& VIB    
& DIB %
\\ \midrule
MNIST
& $98.29 { \scriptstyle \,\pm\, .05 }$   %
& $98.33 { \scriptstyle \,\pm\, .04 }$   %
& $98.68 { \scriptstyle \,\pm\, .04 }$  %
& $98.49 { \scriptstyle \,\pm\, .04 }$  %
& $98.63 { \scriptstyle \,\pm\, .04 }$  %
& $\textbf{98.69} { \scriptstyle \,\pm\, .03 }$  %
\\
CIFAR10MNIST
& $46.49 { \scriptstyle \,\pm\, .07 }$   %
& $47.23 { \scriptstyle \,\pm\, .13 }$   %
& $\textbf{48.86} { \scriptstyle \,\pm\, .17 }$  %
& $44.86 { \scriptstyle \,\pm\, .06 }$  %
& $46.38 { \scriptstyle \,\pm\, .01 }$  %
& $48.07 { \scriptstyle \,\pm\, .10 }$  %
\\ \bottomrule
\end{tabular}
\end{table}

We investigate this question in the same setting as \citet{alemi2016deep}, where the neural network is an MLP, thus treating the pixels as permutation invariant.
We use the same hyperparameters as \cite{alemi2016deep}: 1e-4 learning rate with exponential decay of factor 0.984, Adam optimizer, 200 epochs, trained on the train and validation set, batch size 100, 256 dimensions for $\rv Z$, $\mathcal{X}-1024-1024-\mathcal{Z}$ MLP encoder, logistic regression classifier $\V$.
The only known difference being that we do not use exponential moving average (we did not test with it). 
We jointly train Bob and Alice (``1 Player, Avg ERM'' in \cref{sec:hyperparam_qsmin}).

We evaluate the model on MNIST (as done in \citet{alemi2016deep}) as well as on our CIFAR10+MNIST dataset.
\Cref{table:perminv} compares the test accuracy of DIB with the same regularizers as in \cref{table:worstcase}, the only difference being that all the regularizers are applied both on Bob and Alice (as it is now a single network trained jointly).
For dropout the rate is droping rate 50\%, for weight decay it uses a factor of 1e-4, VIB uses $\beta=1e-3$, DIB uses $\beta=0.1$.
We see that DIB performs best along with Dropout on MNIST, and performs second best after dropout on CIFAR10MNIST.

Although DIB performs well, it does not stand out as much as in the other settings we investigated in the main paper.
We suggest a few potential explanations:
\begin{inlinelist}
\item As is standard practice, we evaluate on accuracy, while our theory only speaks to log likelihood performance (although recall that \cref{table:correlation} does show a strong correlation with accuracy generalization gap);
\item With standard training methods (large learning rate and avg ERM), neural networks generalize relatively well without the need for regularizers, as seen by the strong performance of the stochastic representation baseline in \cref{table:perminv}.
\item Our representations work well for a downstream ERM, they only regularize the model by the representation. In the single player game setting, other methods (such as dropout) regularize both the representation and the downstream classifier, which is not discussed in our theory.
\end{inlinelist}

\subsection{Additional Correlation Experiments}
\label{sec:appx_correlation}
In this section, we expand on the results in~\cref{sec:beyondworst} in the main paper and
show how our approach compares to our 
implementation of sharpness
(the best generalization measure from
~\citep{jiang2019fantastic}) in a setting
with heterogeneous model and dataset choices.

We hypothesize that $\V$-minimality should
be a fairly model- and dataset- agnostic 
measure of generalization. 
Indeed, $\DIFxzCy{}$ has the advantage of being a measure in $[0,\log |\mathcal{Y}|]$, which is 0 when $\V$-minimal (\cref{proposition:Vminsuff}) and seems to be monotonically decreasing with the generalization capacity of a model (\cref{table:correlation}).

In order to study this hypothesis we sweep over different hyper-parameters across two datasets (CIFAR-10 and SVHN) and 
two models (ResNet18 and a $\mathcal{X}$-2048-2048-2048-$\mathcal{Z}$ MLP) each followed by a $\mathcal{Z}$-128-128-$\mathcal{Y}$ MLP, and 5 seeds.
The difference with the experiments in the main paper is that we do not run all possible combination of hyperparemters (computationally prohibitive as we already have 2*2*5=20 models for each seed,data,architecture) but rather sweep over one hyperparameter at the time and compute then average rank correlation.
Here are the hyperparameters we sweep over:
\begin{inlinelist}
\item learning rates (1e-3,1e-4,1e-5,1e-6);
\item weight decay (1e-6,1e-5,1e-4,1e-3,1e-2,0.1);
\item dropout (0.,0.1,0.2,0.3,0.4,0.5,0.6);
\item $\rv Z$ dimensionality (8,32,128,512,2048).
\end{inlinelist}
We additionally train a set of models using VIB with different $\beta$ values (100,10,1,1e-1,1e-2,1e-3).
This gives a total of $5*2*2*(4 + 6 + 7 + 5 + 6)=560$ models, from which we only keep models that reach a training loss of 0.01.
Similar to the main experiment in the paper, we compute
the correlation between the probes and
the observed generalization gap, with the difference being
that the correlation is now computed across experiments
with both the datasets as well as models while varying only one hyperparameter at a time.

\begin{table}[ht!] 
\footnotesize
\caption{Evaluation of our probe and sharp mag. in settings with different datasets and architectures }
\label{table:correlation_ours}
\begin{center}
\begin{tabular}{@{}llccccc@{}}
\toprule
&        & 
    W. Dec & 
    $|\mathcal{Z}|$     & 
     VIB     & 
    Lr   &
    Dropout   \\ 
    \midrule
    \multirow{2}{*}{$\tau_{log like.}$}   & 
    $\V$                & 
    $\textbf{0.67} $  & 
    $\textbf{0.50} $  & 
    $\textbf{0.45} $  & 
    $\textbf{0.62} $  &
    $ \textbf{0.09} $   \\
    & Sharp Mag.     & 
    $0.05$           & 
    $-0.27$           & 
    $-0.16 $           & 
    $0.09$           & 
    $-0.36$ \\
    \multirow{2}{*}{$\tau_{acc.}$}                           & 
    $\V$                                         & 
    $ \textbf{0.63} $ & 
    $ \textbf{0.52} $ & 
    $ \textbf{0.49}$ & 
    $ \textbf{0.56} $   &       
    $ \textbf{0.06} $          \\
    & Sharp Mag.     & 
    $ -0.03$          & 
    $ -0.09$          & 
    $ -0.03$ & 
    $ 0.29 $ & 
    $ -0.19 $ \\ 
    \bottomrule
\end{tabular}
\end{center}
\end{table}

\Cref{table:correlation_ours} shows the performance of $\DIFxzCy{}$ probe compared to sharpness (the best performing baseline in \cref{table:correlation}).
We observe that our approach is significantly better
correlated than sharpness,
both in terms of generalization in terms of accuracy
as well as log-likelihood. 
This seems to support that $\V$-minimality
gracefully handles different datasets and model architectures,
providing reliable estimates of generalization across the spectrum.
In contrast, we see that the sharpness magnitude cannot be used to predict well generalization when sweeping over datasets and architectures.
This suggests that the intuitive idea behind $\DIFxzCy{}$ (considering how easy it is to decode the training examples from a representation using the correct functional family) is useful and robust to predict generalization.


\end{document}